\documentclass[nohyperref]{article}

\usepackage{microtype}
\usepackage{graphicx}
\usepackage{multirow}
\usepackage{threeparttable}

\usepackage{hyperref}

\usepackage{wrapfig}
\usepackage{epsfig}
\usepackage{epstopdf}
\usepackage{hyperref}       
\usepackage{url}            
\usepackage{booktabs}       
\usepackage{amsfonts}       
\usepackage{nicefrac}       
\usepackage{microtype}      
\usepackage{xcolor}
\usepackage{amsopn}
\usepackage{amsmath}      
\usepackage{natbib}
\usepackage{algorithm}
\usepackage{algorithmic}
\usepackage{bm}
\usepackage{amsmath}
\usepackage{graphics}
\usepackage{epsfig}
\usepackage{multirow}
\usepackage{epstopdf}
\usepackage{graphicx}
\usepackage{wrapfig}
\usepackage{amssymb}
\usepackage{pifont}
\usepackage{caption}
\usepackage{subcaption}
\graphicspath{{figures/}}

\usepackage[accepted]{icml2022}

\usepackage{amsmath}
\usepackage{amssymb}
\usepackage{mathtools}
\usepackage{theorem}

\usepackage[capitalize,noabbrev]{cleveref}

\def \A {\mathcal{A}}
\def \T {\mathcal{T}}
\def \D {\mathcal{D}}
\def \I {\mathcal{I}}

\def \u {\mathbf{u}}
\def \v {\mathbf{v}}
\def \yy {\mathbf{y}}

\def \ome {\boldsymbol{\omega}}
\def \lamm {\boldsymbol{\lambda}}

\def \num {\mathtt{num}}
\def \net {\mathtt{net}}

\usepackage{bm}

\def\H{\mathbf{H}}
\def\S{\mathcal{S}}
\def\R{\mathbb{R}}
\def\Q{\mathbf{Q}}
\def\n{\mathbf{n}}
\def\b{\mathbf{b}}
\def\W{\mathbf{W}}

\def\G{\mathbf{G}}
\def\z{\mathbf{z}}

\newtheorem{lemma}{Lemma}[section]
\newtheorem{proposition}{Proposition}[section]
\newtheorem{assumption}{Assumption}[section]
\newtheorem{thm}{Theorem}[section]
\newtheorem{remark}{Remark}

\newenvironment{proof}{{\noindent\it Proof:}\quad}{
	$\hfill \square$ \newline \par}

\usepackage[textsize=tiny]{todonotes}

\icmltitlerunning{Optimization-Derived Learning with Essential Convergence Analysis of Training and Hyper-training}

\begin{document}
	
	\twocolumn[
	\icmltitle{Optimization-Derived Learning with Essential Convergence Analysis \\ of Training and Hyper-training}
	
	
	
	\icmlsetsymbol{equal}{*}
	
	\begin{icmlauthorlist}
		\icmlauthor{Risheng Liu}{DUT,liaoning,pengcheng}
		\icmlauthor{Xuan Liu}{DUT,liaoning}
		\icmlauthor{Shangzhi Zeng}{uvic}
		\icmlauthor{Jin Zhang}{SUSTech,center}
		\icmlauthor{Yixuan Zhang}{SUSTech}
	\end{icmlauthorlist}
	
	\icmlaffiliation{DUT}{DUT-RU International School of Information Science and Engineering, Dalian University of Technology, Dalian, Liaoning, China.}
	\icmlaffiliation{liaoning}{Key Laboratory for Ubiquitous Network and Service Software of Liaoning Province, Dalian, Liaoning, China.}
	\icmlaffiliation{pengcheng}{Peng Cheng Laboratory, Shenzhen, Guangdong, China.}
	\icmlaffiliation{uvic}{Department of Mathematics and Statistics, University
		of Victoria, Victoria, British Columbia, Canada.}
	\icmlaffiliation{SUSTech}{Department of Mathematics, SUSTech International Center for Mathematics, Southern University of Science and Technology, Shenzhen, Guangdong, China.}
	\icmlaffiliation{center}{National Center for Applied Mathematics Shenzhen, Shenzhen, Guangdong, China} 
	
	\icmlcorrespondingauthor{Jin Zhang}{zhangj9@sustech.edu.cn}
	
	\icmlkeywords{Machine Learning, ICML}
	
	\vskip 0.3in
	]
	
	
	
	\printAffiliationsAndNotice{}  
	
	\begin{abstract}

		Recently, Optimization-Derived Learning (ODL) has attracted attention from learning and vision areas, which designs learning models from the perspective of optimization. However, previous ODL approaches regard the training and hyper-training procedures as two separated stages, meaning that the hyper-training variables have to be fixed during the training process, and thus it is also impossible to simultaneously obtain the convergence of training and hyper-training variables. In this work, we design a Generalized Krasnoselskii-Mann (GKM) scheme based on fixed-point iterations as our fundamental ODL module, which unifies existing ODL methods as special cases. Under the GKM scheme, a Bilevel Meta Optimization (BMO) algorithmic framework is constructed to solve the optimal training and hyper-training variables together. We rigorously prove the essential joint convergence of the fixed-point iteration for training and the process of optimizing hyper-parameters for hyper-training, both on the approximation quality, and on the stationary analysis. Experiments demonstrate the efficiency of BMO with competitive performance on sparse coding and real-world applications such as image deconvolution and rain streak removal.

	\end{abstract}

	\section{Introduction}
	\label{sec:introduction}

	There have been a number of methods to handle learning and vision tasks,
	and conventional ones utilize either classic optimization or machine learning schemes directly.
	Using earlier optimizers to solve manually designed objective functions arises two problems:
	the task-related objective functions might be hard to solve,
	and not able to accurately model actual tasks.
	In recent years, a series of approaches called Optimization-Derived Learning (ODL) have emerged, 
	combining the ideas of optimization and learning, and leverage optimization techniques to establish learning methods~\cite{chen2021learning,monga2021algorithm,shlezinger2020model,he2021automl,hutter2019automated,yang2022transformers}.
	The fundamental idea of ODL is to incorporate trainable learning modules into an optimization process and then learn the corresponding parameterized models from collected training data. 
	Therefore, ODL aims to not only possess the convergence guarantee of optimization methods,
	but also achieve satisfactory practical performance
	with the help of neural networks. 
	Overall speaking, ODL has two goals in divergent directions, 
	i.e., 
	to train an algorithmic scheme to minimize the given objective faster 
	or to minimize the reconstruction error between the established model and the actual task.
	These two goals correspond to two processes in ODL, called training and hyper-training.
	For training, we aim to solve the optimization model (minimize the objective and find the optimal training variables),
	while for hyper-training, we aim to find the optimal hyper-parameters (hyper-training variables) to characterize the task. 

%
%
	

	\subsection{Related Works} \label{sec:related works}

	Over the past years, ODL approaches have been established based on various numerical optimization schemes and parameterization strategies (e.g., numerical hyper-parameters and network architectures)~\cite{feurer2019hyperparameter,thornton2013auto,schuler2015learning,chen2016trainable}. 	
	They have been widely applied to all kinds of learning and vision tasks, based on classic optimization schemes like Proximal Gradient or Iterative Shrinkage-Thresholding Algorithm (PG or ISTA)~\cite{daubechies2004iterative} and Alternating Direction Method of Multiplier (ADMM)~\cite{boyd2011distributed}.
	For example, 
	Learned ISTA (LISTA)~\cite{gregor2010learning,chen2018theoretical}, Differentiable Linearized ADMM (DLADMM)~\cite{xie2019differentiable}, and DUBLID~\cite{li2020efficient} can be applied to image deconvolution;
	ISTA-Net~\cite{2017ISTA} can be applied to CS reconstruction;
	ADMM-Net~\cite{yang2017admm} can be applied to CS-MRI;
	PADNet~\cite{liu2019deep} can be applied to image haze removal;
	FIMA~\cite{liu2019convergence} can be applied to image restoration;
	and Plug-and-Play ADMM~\cite{ryu2019plug,venkatakrishnan2013plug,chan2016plug} can be applied to image super resolution.

	According to divergent starting points and corresponding types of employed schemes,
	we can roughly divide existing ODL into two main categories,
	called ODL based on Unrolling with Numerical Hyper-parameters (UNH) 
	and ODL Embedded with Network Architectures (ENA), respectively. 
	UNH starts from a traditional optimization process for training, and aims to unroll the iteration with learnable hyper-parameters.
	Thus, it utilizes parameterized numerical algorithms to optimize the determinate objective function.
	In the early stages, many classic optimization schemes have been designed based on theories and experiences, 
	such as Gradient Descent (GD), 
	ISTA, 
	Augmented Lagrangian Method (ALM), 
	ADMM, 
	and Linearized ADMM (LADMM)~\cite{lin2011linearized}.
	Parameters within these schemes can be regarded as hyper-parameters, and then learned via unrolling.
	There are also some works which introduce learnable modules into the optimization process to solve the determinate objective function faster,
	such as 
	LBS~\cite{liu2018toward},
	GCM~\cite{liu2018learning},
	and TLF~\cite{liu2020investigating}.
%
%
%
%
%
	As for ENA, it starts from an optimization objective function, but it aims to design its network architecture to solve specific tasks.
	It replaces part of the original 
	training iterations with (or just directly incorporates) trainable architectures to 
	approach the ground truth more efficiently,
	so the relationship between the final and original model is unclear.	
	In LISTA and DLADMM, they learn matrices as linear layers of networks.
	There are also some methods employing networks in the classical sense,
	such as ISTA-Net and Plug-and-Play ADMM.
	Very lately, pre-trained CNN-based modules such as DPSR and DPIR~\cite{nguyen2017plug,zhang2019deep,zhang2020plug,yuan2020plug,ahmad2020plug,song2020new} are implemented to achieve image restoration.
    However, UNH cannot reduce the gap between the artificially designed determinate objective function and the actual task;
    while for ENA, the embedded complex network modules make the iterative trajectory and convergence property difficult to analyze.
%
%
%
%
%
%
	In addition, UNH and ENA have a common and inevitable flaw.
	As aforementioned, ODL aims to minimize the objective function for training and to minimize the reconstruction error for hyper-training.
	However, existing ODL methods can only achieve one of these two goals once,
	which means the training and hyper-training procedures have to be separated into two independent stages.  
	In the first stage, hyper-parameters as hyper-training variables are determined by hyper-training,
	while in the second stage, they are fixed and substituted into the training iterations to find the optimal training variables.
	Therefore, the optimal hyper-training and training variables cannot be obtained simultaneously.
	This feature makes existing ODL methods inflexible,
	and since they ignore the intrinsic relationship between hyper-training and training variables,
	the obtained solution may not be the true one~\cite{liu2020generic},
	especially when the optimal training variables are not unique.
	

	%

	From the viewpoint of theory, 
	although the empirical efficiency of ODL has been witnessed in applications, research on solid convergence analysis is still in its infancy. 
	This gap makes the broader usability of ODL questionable. 
	Recently, some works have tried to provide convergence analysis on ODL methods through classic optimization tools.  
	In particular,
	~\cite{chan2016plug,teodoro2018convergent,sun2019online} analyze the non-expansiveness of incorporated trainable architectures under the boundedness assumption of the embedded network. 
	\cite{ryu2019plug} requires that the network residual admits a Lipschitz constant strictly smaller than one. 
	However, these methods can only guarantee the convergence towards some fixed points of the approximated model, instead of the solution to the original task. 
	A naive strategy to ensure the convergence is to learn as few parameters (hyper-training variables) as possible, such as to learn nothing but the step size of ISTA~\cite{ablin2019learning}. 
	Besides, when learning a neural network, we need additional artificially designed corrections. 
	For example,~\cite{liu2019convergence,moeller2019controlling,heaton2020safeguarded} 
	use neural networks and optimization algorithms to generate temporary updates and manually design rules to select true updates. 
	The restrictions seriously limit these methods.
	In addition, most previous ODL approaches are designed specially based on a specific  optimization scheme, making their convergence theory hard to be extended to other methods.

	\subsection{Our Contributions}
	
	As mentioned above, existing ODL methods, containing UNH and ENA, can only handle learning and vision tasks by optimizing hyper-training variables and training variables separately,
	raising a series of problems.
	Besides, they only focus on some specific problems with special structures. 
	In dealing with these issues, in this work,
	we construct the Generalized Krasnoselskii-Mann (GKM) scheme as a new and general ODL formulation from the perspective of fixed-point iterations.
	This implicitly defined scheme is more flexible than traditional optimization models,
	and includes more methods than existing ODL. 
	After that we establish the Bilevel Meta Optimization (BMO) algorithmic framework
	to simultaneously solve the training and hyper-training tasks,
	inspired by the leader-follower game~\cite{von2010market,liu2021investigating}.
	Then the process of training under the fixed-point iterations is to solve the lower-level training variables,
	while the process of hyper-training is to find the optimal upper-level hyper-training variables.
	In this way, we can incorporate training and hyper-training together and obtain the true optimal solution even if the fixed points are not unique.
	A series of theoretical properties are strictly proved to guarantee the essential joint convergence of training  and hyper-training variables, both on the approximation quality and on the stationary analysis. 
	We also conduct experiments on various learning and vision tasks to verify the effectiveness of BMO. 
	Our contributions are summarized as follows.

	
	\begin{itemize}
%
%

		\item We establish a new ODL formulation from the perspective of fixed-point iterations under the GKM scheme,
		which introduces learnable parameters as hyper-training variables.
		Serving as a general form of various ODL methods, the implicitly defined scheme  not only contains existing ENA and UNH methodologies,
		but also produces combined models.
		
		\item Based on the GKM scheme, the BMO algorithm provides a leader-follower mechanism to incorporate the process of training and hyper-training. 
		Unlike existing ODL methods separating these two sub-tasks,
		our method can optimize training and hyper-training variables simultaneously,
		making it possible to investigate their intrinsic relationship to obtain the true solution.
		 
		
		\item To our best knowledge, this is the first work that provides strict essential convergence analysis of both training and hyper-training variables,
		containing the analysis on approximation quality and stationary convergence.
		This is what existing ODL methods cannot achieve
		since they 
		regard training and hyper-training as independent procedures.
		
	\end{itemize}

	\section{The Proposed Algorithmic Framework} \label{sec:algorithm}
	
	In this section, we first present a general ODL platform by generalizing the classical fixed-point scheme~\cite{edelstein1978nonexpansive}, 
	and existing ODL methods (UNH and ENA) can be regarded as special cases of our scheme. 
	Then by considering the processes of training and hyper-training from meta optimization~\cite{liu2021investigating}, 
	we put forward our Bilevel Meta Optimization (BMO) algorithm. 

	\subsection{Generalized Krasnoselskii-Mann Scheme for ODL}
	\label{sec:GKM}
	
	Here we propose the Generalized Krasnoselskii-Mann (GKM) learning scheme as a general form of ODL methods.
	To begin with, we consider the following optimization model as the training process:
	\begin{equation}\label{eq:min f}
		\min_{\u \in U} f(\u), \ s.t. \ \A(\u)=\yy,
	\end{equation}
	where $\u \in U$ is the training variable, $f(\u)$ is the objective function related to the task, and $\A(\u)=\yy$ is a necessary linear constraint ($\A$ is a linear operator).
	Denote $\D(\cdot)$ as the corresponding non-expansive algorithmic operator.
	Then to solve Eq.~\eqref{eq:min f} is to iterate for the fixed point of $\D(\cdot)$.
	Thus the model in Eq.~\eqref{eq:min f} can be transformed into
	\begin{equation}\label{eq:u in fix}
		\u \in \mathtt{Fix} (\D(\cdot)),
	\end{equation}
	where $\mathtt{Fix}(\cdot)$ represents the set of fixed points.
	This serves as our fundamental training scheme,
	and is actually more general than Eq.~\eqref{eq:min f}, 
	because it not only contains traditional optimization models,
	but also represents other implicitly defined ODL models and implicit networks~\cite{fung2022jfb},
	which are designed based on optimization but further added with learnable modules.

	This process can be implemented via the following classical Krasnoselskii-Mann (KM) updating scheme~\cite{reich2000convergence,borwein1992krasnoselski}, whose $k$-th iteration step is
	$
	\label{eq:km}
	\T(\u^k) = \u^k + \alpha (\D(\u^k) - \u^k),
	$
	where $\alpha\in(0,1)$, and $\D(\cdot)$ is the operator for some basic numerical methods such as GD, PG, and ALM. 
	It can be observed that here $\T$ is an $\alpha$-averaged non-expansive operator.

	In this work, we further generalize the classical KM schemes to the following Generalized KM (GKM) learning scheme: 
	\begin{equation}\label{eq:gkm}
		\T(\u^k,\ome) = \u^k + \alpha (\D(\u^k,\ome) - \u^k),
	\end{equation}
	where 
	$
	\D(\cdot, \ome) \in 
	\left\{  \D_\num(\cdot, \ome) \circ \D_\net(\cdot, \ome) \right\}.
	$ 
	Here 
	$\circ$ represents compositions of operators,
	$\D_\num$ denotes numerical operators in traditional optimization schemes,
	and $\D_\net$ denotes iterative handcrafted network architectures.
	The variability of hyper-parameters $\ome$ as hyper-training variables makes the scheme more flexible.
	These hyper-training variables correspond to different parameters for various ODL methods.
	For example, in classic optimization schemes, $\ome$ denotes step size;
	in LISTA or DLADMM, $\ome$ comes from differentiable proximal operators, thresholds, support selection or penalties;
	in other methods with network modules, such as LBS, GCM, TLF, ISTA-Net, Plug-and-Play ADMM, and pre-trained CNN, $\ome$ denotes network parameters.	
	Note that in order to guarantee the non-expansiveness of $\D$, we utilize some normalization techniques on these parameters, 
	such as spectral normalization~\cite{miyato2018spectral}.
	By specifying $\D$ as different types of $\D_\num$ and $\D_\net$,
	we can contain diverse ODL methods (e.g., UNH, ENA, and their combinations) within our GKM scheme. 
	Examples of specific forms of operator $\D$ can be found in Appendix~\ref{sec:appendix B about D}.

	\subsection{Bilevel Meta Optimization for Training and Hyper-training} 
	
	As mentioned in Section~\ref{sec:related works}, existing ODL methods separately consider training and hyper-training as two independent stages.
	Hence, they fail to investigate the intrinsic relationship between training variables $\u$ and hyper-training variables $\ome$. 
	In this work, we would like to utilize the perspective from meta optimization~\cite{neumuller2011parameter} to incorporate these two coupled processes of training and hyper-training. 
	Thus, the optimal training and hyper-training variables can be obtained simultaneously.
	We formulate this meta optimization task within the leader-follower game framework. 
	In the leader-follower (or Stackelberg) game, the leader commits to a strategy, while the follower observes the leader's commitment and decides how to play after that. 

	Specifically in our task, by recognizing hyper-training variables $\ome$ and training variables $\u$ as the leader and follower respectively, 
	we have that with $\ome$, the iteration module $\T$ is determined, from which $\u$ finds the fixed point of $\T$. 
	Therefore, the leader $\ome$ optimizes its objective via the best response of the follower $\u$~\cite{liu2021investigating}.
	In order to clearly describe this hierarchical relationship, we introduce the following bilevel formulation to model ODL:
	\begin{equation}\label{bilevel_fix}
		\min\limits_{\u\in U,\ome\in\Omega}\ell(\u;\ome), 
		\text{ s.t. } \u\in\mathtt{Fix}(\T(\cdot,\ome)),
	\end{equation}
	where $\ell$ denotes the objective function for measuring the performance of the training process, which usually is set to be the loss function. 
	Thus, the upper level corresponds to the hyper-training process.
	Hereafter, we call this formulation as Bilevel Meta Optimization (BMO).
	Note that this is more general than a traditional bilevel optimization problem, 
	since the lower level for training is the solution mapping of a broader fixed-point iteration
	as mentioned for Eq.~\eqref{eq:u in fix} in Section~\ref{sec:GKM}.

	BMO can overcome several issues in existing ODL methods.
	When updating $\ome$ in the upper level for hyper-training, 
	instead of only using the information from original data, 
	BMO allows us to utilize the task-related priors by updating $\ome$ and $\u$ simultaneously.
	On the other hand, in theory, BMO makes us able to analyze the essential joint convergence of $\ome$ and $\u$ under their nested relationship, rather than only consider the convergence of $\u$.
	Thus we are able to obtain the true optimal solution of the problem,
	instead of just the fixed points of the lower iteration,
	especially when the fixed points are not unique.
	We will give the detailed convergence analysis in Section~\ref{sec:theoritical}.

	Now we establish the solution strategy to solve $\u$ and $\ome$ simultaneously under BMO. 
	Inspired by~\cite{liu2020generic,BDA}, in order to make efficient use of the hierarchical information contained in Eq.~\eqref{bilevel_fix}, 
	we update variables from divergent perspectives,
	aggregating the information from training and hyper-training.

	First, we give the optimization direction $\v_{l}$ from the perspective of training process in the lower level of Eq.~\eqref{bilevel_fix}. 
	We use $\T\left(\cdot,\ome\right)$ to update $\u^k$ at the $k$-th step,
	That is, to solve $\u\in\mathtt{Fix}(\T(\cdot,\ome))$,
	we define $\v_{l}^{k} =\T \left( \u^{k-1},\ome \right)$.

	Subsequently, we further give the descent direction based on the hyper-training process in the upper level of Eq.~\eqref{bilevel_fix}.  
	To make the updating direction contain the information of hyper-training variables $\ome$, 
	a simple idea is to directly use the gradient of loss function $\ell$ with respect to $\u$. 
	However, the gradient descent of $\ell$ may destroy the non-expansive property. 
	Hence, we add an additional positive-definite correction matrix $\H_{\ome}$ parameterized by $\ome$, 
	and set the step size as a decreasing sequence (i.e., $s_k\rightarrow0$)
	to ensure the correctness of the descent direction, i.e.,
\begin{equation}
	\v_{u}^{k} =\u^{k-1}-s_k \H_{\ome}^{-1} \frac{\partial }{\partial \u}\ell(\u^{k-1},\ome).
\end{equation}

	Next, we aggregate the two iterative directions and introduce a projection operator to generate the final updating direction:
\begin{equation}
\u^k =\mathtt{Proj}_{U, \H_{\ome}}\left(\mu \v_{u}^{k}+(1-\mu) \v_{l}^{k}\right).\end{equation}
	Here $\mathtt{Proj}_{U,\H_{\ome}}(\cdot)$ is the projection operator associated to $\H_{\ome}$ 
	and is defined as $\mathtt{Proj}_{U,\H_{\ome}}(\u) = \mathrm{argmin}_{\bar{\u} \in U} \|\bar{\u} - \u\|_{\H_{\ome}}$.
	Note that the projection $\mathtt{Proj}_{U,\H_{\ome}}$ is only for ensuring the boundedness of $\u^k$ in the theoretical analysis. 
	In practice, $U$ is usually chosen as a sufficiently large bounded set (e.g., $\R^n$), 
	and thus the projection operator can be ignored. 
	Finally, after $K$ updates of the variable $\u$, 
	we update the hyper-training variables $\ome$ by gradient descent. 
	Note that  $\frac{\partial }{\partial \u}\ell(\u^{K},\ome)$ can be obtained by automatic differential efficiently~\cite{franceschi2017forward}.
	Finally, we summarize the full BMO solution strategy in Algorithm~\ref{alg:bmo}.

	\begin{algorithm}[H]
		\caption{The Solution Strategy of BMO}\label{alg:bmo}
		\begin{algorithmic}[1]
			\REQUIRE Step sizes $\{s_k\}$, $\gamma$ and parameter $\mu$
			\STATE Initialize $\ome^0$ .
			\FOR {$t=1\rightarrow T$}
			\STATE Initialize  $\u^0$.
			\FOR {$k=1\rightarrow K$}
			
			\STATE $\v_{l}^{k} =\T\left(\u^{k-1},\ome^{t-1}\right)$.
			\STATE $\v_{u}^{k} =\u^{k-1} - s_k\H^{-1}_{\ome} \frac{\partial }{\partial \u}\ell(\u^{k-1},\ome^{t-1}) $.
			\STATE $\u^k =\mathtt{Proj}_{U, \H_{\ome}}\left(\mu \v_{u}^{k}+(1-\mu) \v_{l}^{k}\right)$.
			\ENDFOR
			\STATE $\ome^{t}=\ome^{t-1}-\gamma\frac{\partial}{\partial \ome}\ell(\u^K,\ome^{t-1})$.
			\ENDFOR
		\end{algorithmic}
	\end{algorithm}

	\section{Joint Convergence Analysis} \label{sec:theoritical} 
	
	In this section,
	we discuss the essential convergence analysis of our proposed BMO algorithm (Algorithm~\ref{alg:bmo}) for the GKM scheme,
	towards the optimal solution and stationary points of optimization problem in Eq.~\eqref{bilevel_fix} with respect to both $\u$ and $\ome$. 
	Note that this joint convergence of training and hyper-training also provides a unified theoretical guarantee for existing ODL methods containing UNH and ENA.
	Complete theoretical analysis is stated in Appendix~\ref{sec:appendix proof}.

	By introducing an auxiliary function, the problem in Eq.~\eqref{bilevel_fix} can be equivalently rewritten as the following 
	\begin{equation}\label{eq:phi_def}
		\min_{\ome \in \Omega} \ \varphi(\ome),\quad  \text{where} \quad \varphi(\ome) := \inf_{\u \in \mathtt{Fix}(\T(\cdot,\ome)) \cap U } \ \ell(\u,\ome).
	\end{equation}
	The sequence $\{\ome^t\}$ generated by BMO (Algorithm~\ref{alg:bmo}) actually solves the following approximation problem of Eq.~\eqref{bilevel_fix}
	\begin{equation}\label{eq:phiK_def}
		\min_{\ome \in \Omega} \ \varphi_K(\ome) := \ell(\u^K(\ome),\ome),
	\end{equation}
	where $\u^K(\ome)$ is derived by solving the problem $\inf_{\u \in \mathtt{Fix}(\T(\cdot,\ome)) \cap U } \ \ell(\u,\ome)$ and can be given by 
	\begin{equation}\label{bilevel_alg}
		\left\{
		\begin{aligned}
			\v^k_l(\ome) & =  \T(\u^{k-1}(\ome),\ome), \\
			\v^k_u(\ome) & = \u^{k-1}(\ome) - s_k \H_{\ome}^{-1}\frac{\partial }{\partial \u}\ell(\u^{k-1},\ome), \\
			\u^k(\ome) & = \mathtt{Proj}_{U,\H_{\ome}} \big( \mu \v^k_u(\ome) + (1-\mu) \v^k_l(\ome) \big),
		\end{aligned}\right.
	\end{equation}
	where $k = 1,\ldots, K$.

	\subsection{Approximation Quality and Convergence}
	In this part, we will show that Eq.~\eqref{eq:phiK_def} is actually an appropriate approximation to Eq.~\eqref{bilevel_fix} in the sense that any limit point $(\bar{\u},\bar{\ome})$ of the sequence $\left\{\left( \u^K(\ome^K),\ome^K\right)\right\}$ with $\ome^K \in \mathrm{argmin}_{\ome \in\Omega}\varphi_{K}(\ome)$ is a solution to the bilevel problem in Eq.~\eqref{bilevel_fix}.
	Thus we can obtain the optimal solution of Eq.~\eqref{bilevel_fix} by solving Eq.~\eqref{eq:phiK_def}.
	We make the following standing assumption throughout this part.

	\begin{assumption}\label{ass:assum_F}
		$\Omega$ is a compact set and $U$ is a convex compact set. $\mathtt{Fix}(\T(\cdot,\ome))$ is nonempty for any $\ome \in \Omega$. $\ell(\u,\ome)$ is continuous on $\R^n \times \Omega$. For any $\ome \in \Omega$, $\ell(\cdot,\ome) : \R^n \rightarrow \R$ is $L_\ell$-smooth, convex and bounded below by $M_0$.
	\end{assumption}
	Please notice that $\ell$ is usually defined to be the MSE loss, and thus Assumption~\ref{ass:assum_F} is quite standard for ODL~\cite{ryu2019plug,zhang2020plug}. 
	We first present some necessary  preliminaries.
	For any two matrices $\H_1, \H_2 \in \R^{n \times n}$, we consider the following partial ordering relation:
\begin{equation}
	\H_1 \succeq \H_2 \quad \Leftrightarrow \quad  \langle \u,\H_1\u \rangle \ge \langle \u,\H_2\u \rangle, \quad \forall \u \in \R^n.
\end{equation}
	If $\H \succ 0$, $\langle \u_1, \H \u_2 \rangle$ for $\u_1,\u_2 \in \R^n$ defines an inner product on $\R^n$. 
	Denote the induced norm with $\| \cdot \|_\H$, 
	i.e., $\| \u \|_\H := \sqrt{\langle \u,\H \u \rangle}$ for any $\u \in \R^n$.
	Denote the graph of operator $\D(\cdot,\ome)$ to be 
	\[
			\mathrm{gph} \,\D(\cdot,\ome) := \{(\u,\v) \in \R^n \times \R^n ~|~ \v = \D(\u,\ome)\}.
	\]
	We assume $\D(\cdot,\ome)$ satisfies the following assumption throughout this part.
	\begin{assumption} \label{ass:assum_T}
		There exist $\H_{ub} \succeq \H_{lb} \succ 0$, 
		such that for each $\ome \in \Omega$, there exists $\H_{ub} \succeq \H_{\ome} \succeq \H_{lb}$ such that
		\begin{itemize}
			\item[(1)] $\D(\cdot,\ome)$ is non-expansive with respect to $\| \cdot \|_{\H_{\ome}}$, i.e., for all $(\u_1,\u_2) \in \R^n \times \R^n$,
			\begin{equation*}
				\|\D(\u_1,\ome) - \D(\u_2,\ome) \|_{\H_{\omega}} \le \| \u_1 - \u_2\|_{\H_{\ome}}.
			\end{equation*}
			\item[(2)] $\D(\cdot,\ome)$ is closed, i.e.,
			$\mathrm{gph} \,\D(\cdot,\ome)$
			is closed.
		\end{itemize}
	\end{assumption}

	Under Assumption~\ref{ass:assum_T}, we obtain the non-expansive property of $\T(\cdot,\ome)$ defined in Eq.~\eqref{eq:gkm} from~\cite{Heinz-MonotoneOperator-2011}[Proposition~4.25] immediately.
	Then 
	we can show that the sequence $\{\u^k(\ome)\}$ generated by Eq.~\eqref{bilevel_alg} not only converges to the solution set of $\inf_{\u \in \mathtt{Fix}(\T(\cdot,\ome)) \cap U } \ \ell(\u,\ome)$ 
	but also admits a uniform convergence towards the fixed-point set $\mathtt{Fix}(\T(\cdot,\ome)$ 
	with respect to $\| \u^k(\ome) - \T(\u^k(\ome),\ome) \|_{\H_{lb}}^2$ for $\ome \in \Omega$.

	\begin{thm}
		Let $\{\u^k(\ome)\}$ be the sequence generated by Eq.~\eqref{bilevel_alg} with $\mu \in (0,1)$ and $s_k = \frac{s}{k+1}$, $s \in (0, \frac{\lambda_{\min}(\H_{lb})}{L_{\ell}} )$, where $\lambda_{\min}(\H_{lb})$ denotes the smallest eigenvalue of matrix $\H_{lb}$. Then, we have for any $\ome \in \Omega$,
		\begin{equation*}
			\begin{array}{c}
				\lim\limits_{k \rightarrow \infty}\mathrm{dist}(\u^k(\ome),\mathtt{Fix}(\T(\cdot,\ome)) = 0,
			\end{array}
		\end{equation*}
		\begin{equation*}
			\begin{array}{c}
				\lim\limits_{k \rightarrow \infty}\ell(\u^k(\ome),\ome) =  \varphi(\ome).
			\end{array}
		\end{equation*}
		Furthermore, there exits $C > 0$ such that for any $\ome \in \Omega$,
		\[
		\| \u^k(\ome) - \T(\u^k(\ome),\ome) \|_{\H_{lb}}^2 \le C\sqrt{\frac{1+\ln(1+k)}{k^{\frac{1}{4}}}}.
		\]
	\end{thm}

	Thanks to the uniform convergence property of the sequence $\{\u^k(\ome)\}$, 
	inspired by the arguments used in~\cite{BDA},
	we can establish the convergence on both $\u$ and $\ome$ of BMO (Algorithm~\ref{alg:bmo}) towards the solution of optimization problem in Eq.~\eqref{bilevel_fix} in the following theorem.

	\begin{thm}
		Let $\{\u^k(\ome)\}$ be the sequence generated by Eq.~\eqref{bilevel_alg} with $\mu \in (0,1)$ and $s_k = \frac{s}{k+1}$, $s \in (0, \frac{\lambda_{\min}(\H_{lb})}{L_{\ell}} )$. 
		Then, let $\ome^K \in \mathrm{argmin}_{\ome \in\Omega}\varphi_{K}(\ome)$, and we have
		\begin{itemize}
			\item[(1)] any limit point $(\bar{\u},\bar{\ome})$ of the sequence $\{(\u^K(\ome^K),\ome^K) \}$ is a solution to the problem in Eq.~\eqref{bilevel_fix}, i.e., $\bar{\ome}\in\mathrm{argmin}_{\ome\in\Omega}\varphi(\ome)$ and $\bar{\u} = \T(\bar{\u},\bar{\ome}) $.
			\item[(2)] $\inf_{\ome \in \Omega}\varphi_K(\ome) \rightarrow \inf_{\ome \in \Omega} \varphi(\ome)$ as $K \rightarrow \infty$.
		\end{itemize}
	\end{thm}

	\subsection{Stationary Analysis}
	
	Here we provide the convergence analysis of our algorithm with respect to stationary points,
	i.e., any limit point $\bar{\ome}$ of the sequence $\{\ome^K\}$ satisfies $\nabla \varphi(\bar{\ome}) = 0$, where $\varphi(\ome)$ is defined in Eq.~\eqref{eq:phi_def}.
	We consider the following assumptions. 
	For $\D(\cdot,\ome)$ we request a stronger assumption than Assumption~\ref{ass:assum_T} on contractive property.
	
	\begin{assumption} \label{ass:assum_stationary ell}
		$\Omega$ is compact and $U = \R^n$. 
		$\mathtt{Fix}(\T(\cdot,\ome))$ is nonempty
		for any $\ome \in \Omega$. 
		$\ell(\u,\ome)$ is twice continuously differentiable on $\R^n \times \Omega$. 
		For any $\ome \in \Omega$, $\ell(\cdot,\ome) : \R^n \rightarrow \R$ is $L_\ell$-smooth, convex and bounded below by $M_0$.
	\end{assumption}

	\begin{assumption} \label{ass:assum_stationary D}
	    $\D(\cdot,\ome)$ is contractive w.r.t. $\| \cdot \|_{\H_{\ome}}$.
	\end{assumption}

	Denote $\hat{\S}(\ome):= \mathrm{argmin}_{\u \in \mathtt{Fix}(\T(\cdot, \ome)) \cap U } \ell (\u,\ome)$,
	and we have the following stationary result.
	
	\begin{thm}
		Suppose Assumptions~\ref{ass:assum_T},~\ref{ass:assum_stationary ell} and~\ref{ass:assum_stationary D} are satisfied, 
		$\frac{\partial}{\partial \u} \T(\u,\ome)$ and $\frac{\partial}{\partial \ome} \T(\u,\ome)$ are Lipschitz continuous with respect to $\u$,
		and $\hat{\S}(\ome)$ is nonempty for all $\ome \in \Omega$. 
		Let $\{\u^k(\ome)\}$ be the sequence generated by Eq.~\eqref{bilevel_alg} with $\mu \in (0,1)$ and $s_k = \frac{s}{k+1}$, $s \in (0, \frac{\lambda_{\min}(\H_{lb})}{L_{\ell}} )$.
		Let $\ome^K$ be an $\varepsilon_K$-stationary point of $\varphi_{K}(\ome)$, i.e., 
		$
		\| \nabla \varphi_K(\ome^K) \| = \varepsilon_K.
		$
		Then if $\varepsilon_K \rightarrow 0$, we have that any limit point $\bar{\ome}$ of the sequence $\{\ome^K\}$ is a stationary point of $\varphi$, i.e., 
		$
			\nabla \varphi(\bar{\ome}) = 0. 
		$
	\end{thm}

	\section{Applications for Learning and Vision} \label{sec:application}
	
	Here we demonstrate some examples on how to address various real-world learning and vision applications under our BMO framework by jointly optimizing the training and hyper-training variables. 
	In sparse coding and image deconvolution,
	we model the task to be sparsification of coefficients as training, so the training variable $\u$ is model parameters; 
	in rain streak removal, the task is to decompose the background and rain streak layer as a generalized training process, so $\u$ is the model output (a clear image).
	In all tasks, we embed learnable iterations to yield better results (closer to ground truth
	or target) than directly minimizing the original objective function.
	More detailed information of the operator $\D$ can be found in Appendix~\ref{sec:appendix B about D}, containing the proof that $\D$ in our applications satisfies the assumptions.

	\textbf{Sparse Coding.}
	Sparse coding has become a popular technique to extract features from raw data~\cite{zhang2015survey}. 
	The difficulty within is to recover the sparse vector from the noisy data and ill-posed transform matrix. 
	To be specific, given the input dataset $\b\in\R^m$ and transform matrix $\Q\in\R^{m\times n}$, 
	our goal is to find the model parameters representation $\u_1\in\R^n$ and noise $\u_2\in\R^n$ as training variables such that they satisfy the following model $\Q\u_1+\u_2=\b$, 
	and the representation $\u_1$ and noise $\u_2$ are expected to be sparse enough,
	i.e., $\kappa\Vert\u_1\Vert_1+\Vert\u_2\Vert_1$ is minimized. 
	Hence, the optimization problem in Eq.~\eqref{eq:min f} can be formulated as 
	\begin{equation}\label{eq:sc}
		\min_{\u}\kappa\Vert\u_1\Vert_1+\Vert\u_2\Vert_1, \text{ s.t. } \Q\u_1+\u_2=\b.
	\end{equation}
Classic numerical optimization methods usually use ADMM with gradient descent and projection operations to solve Eq.~\eqref{eq:sc}. 
 UNH attempts to learn descent directions that solve the problem faster and better than ADMM; ENA starts from a traditional numerical algorithm for an optimization objective function and replaces the projection operator with a functionally similar network, thus solving the specific task (rather than solve the original objective) and maintaining interpretability.  
 Notice that in comparison to some methods which intend to learn effective $\Q^\intercal$ for each update of variables, BMO use the same learnable module each time to update variables, which decouples our parameter quantity from the number of updates.
As for the descent operation, we consider 
$\D_{\mathtt{DLADMM}}$ as $\D_{\net}$. 

	\textbf{Image Deconvolution}
	For practical applications, we first consider image deconvolution (image deblurring)~\cite{richardson1972bayesian,andrews1977digital},
	a typical low-level vision task as a branch of image restoration, 
	whose purpose is to recover the clean image from the blurred one.
	The input image can be expressed as $\b = \Q \ast\u + \n$, 
	where $\Q ,\u$, and $\n$ respectively represent the blur kernel, latent clean image, and additional noise,
	and $\ast$ denotes the two-dimensional convolution operator. 
	Here we apply regularization methods based on Maximum A Posteriori (MAP) estimation. 
	Then the problem can be expressed as $\min_{\u\in U}\Vert\Q \ast \u-\b\Vert^2_2+\Phi(\u)$, 
	where $\Phi(\u)$ is the prior function of the image. 
	Since the image after wavelet transform is usually sparse, 
	we consider $\Phi(\u)=\kappa\Vert \W\u\Vert_1$, 
	where $\W$ is the wavelet transform matrix. 
	The objective function is \begin{equation}
\min_{\u}\Vert\Q \ast \u-\b\Vert^2_2+\kappa\Vert\W\u\Vert_1.\end{equation}
	Hence, similar to sparse coding, based on the wavelet transform model, we learn a wavelet coefficients as the training variables for image deconvolution. 
	For handling this problem,
	we composite $\D_{\mathtt{PG}}$ as $\D_\num$ with handcrafted network architectures DRUNet in DPIR~\cite{zhang2020plug} as $\D_{\net}$.

	\textbf{Rain Streak Removal.}
	With signal-dependent or signal-independent noise, images captured under rainy conditions often suffer from weak visibility~\cite{wang2019survey,li2019single}. 
	The rain streak removal task aims to decompose an input rainy image $\b$ into a rain-free background $\u_b$ and a rain streak layer $\u_r$, i.e., $\b=\u_b + \u_r$, 
	and hence enhances its visibility. 
	This problem is ill-posed since the dimension of unknowns $\u_b$ and $\u_r$ to be recovered is twice as many as that of the input $\b$. 
	The problem can be reformulated as the following optimization problem
	\begin{equation}\label{derain_}
		\min\limits_{\u_b,\u_r}\frac{1}{2}\Vert\u_b+\u_r-\b\Vert_2^2+\psi_b(\u_b)+\psi_r(\u_r),
	\end{equation}
	where $\psi_b(\u_b)$ denotes the priors on the background layer 
	and $\psi_r(\u_r)$ represents the priors on rain streak layer. 
	Here we set $\psi_b(\u_b)=\kappa_b\Vert\u_b\Vert_1$,$\psi_r(\u_r)=\kappa_r\Vert\nabla\u_r\Vert_1$. 
	By introducing auxiliary variables $\v_b$ and $\v_r$, the problem in Eq.~\eqref{eq:min f} is specified as 
	\begin{equation}\label{derain}
		\begin{aligned}
			\min\limits_{\u_b,\u_r,\v_b,\v_r} & \frac{1}{2}\Vert\u_b+\u_r-\b\Vert_2^2+\kappa_b\Vert\v_b\Vert_1+\kappa_r\Vert\v_r\Vert_1, \\
			\text{ s.t. }\quad  & \v_b=\u_b,\v_r=\nabla\u_r,
		\end{aligned}
	\end{equation}
	where $\nabla=\left[ \nabla_h;\nabla_v\right] $ denotes the gradient in horizontal and vertical directions. 
	The training variables in this task are output images $\u_b$ and $\u_r$ based on the simple summation model $\b=\u_b + \u_r$. 
	UNH usually uses ALM to solve Eq.~\eqref{derain}. ENA starts from ALM and uses the network to approximate the solution of the sub-problems.
	Here we consider $\D_{\mathtt{ALM}}$ as $\D_\num$ and the designed iterable network architectures RCDNet~\cite{wang2020model} as $\D_{\net}$.

	\section{Experimental Results} \label{sec:experiment}
	
	In this section, we illustrate the performance of BMO on sparse coding, image deconvolution and rain streak removal tasks. 
	More detailed parameter setting and network architectures can be found in Appendix~\ref{sec:appendix C experiments}.

	\subsection{Sparse Coding}

	We first investigate the performance in sparse coding, and compare BMO with LADMM and DLADMM, respectively as an example of UNH and ENA.
	Note that here we do not compare with pure networks, because we focus on convergence behaviors which pure networks cannot guarantee.
	Following the setting in~\cite{chen2018theoretical}, 
	we experiment on the classic Set14 dataset,
	in which salt-and-pepper noise is added to $10 \%$ pixels of each image. 
	Furthermore, the rectangle of each image is divided into non-overlapping patches of size $16 \times 16$. We use the patchdictionary method~\cite{xu2014fast} to learn a $256 \times 512$ dictionary $\Q$. 

	\begin{figure*}[!htbp]
		\centering
		\begin{subfigure}[t]{0.19\textwidth}
			{\includegraphics[height=3.2cm,width=3.2cm]{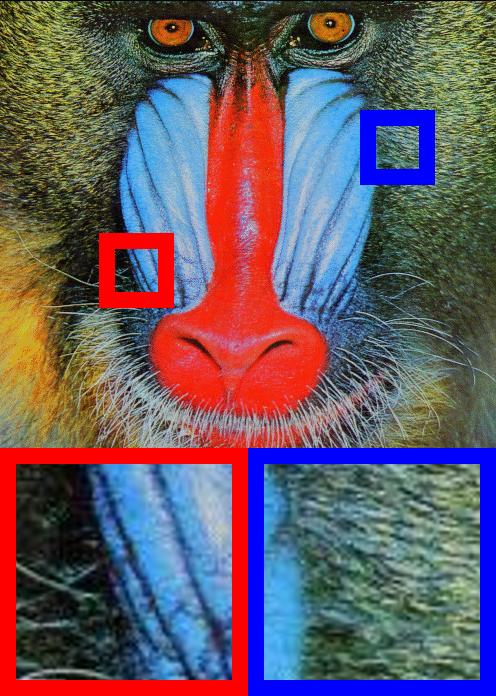}}
			\subcaption{Ground Truth}
		\end{subfigure}
		\begin{subfigure}[t]{0.19\textwidth}
			{\includegraphics[height=3.2cm,width=3.2cm]{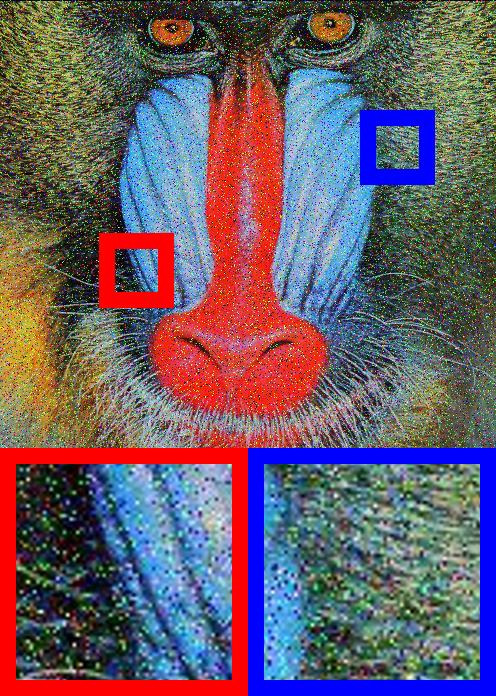}}
			\subcaption{Noisy Image}
		\end{subfigure}
		\begin{subfigure}[t]{0.19\textwidth}
			{\includegraphics[height=3.2cm,width=3.2cm]
				{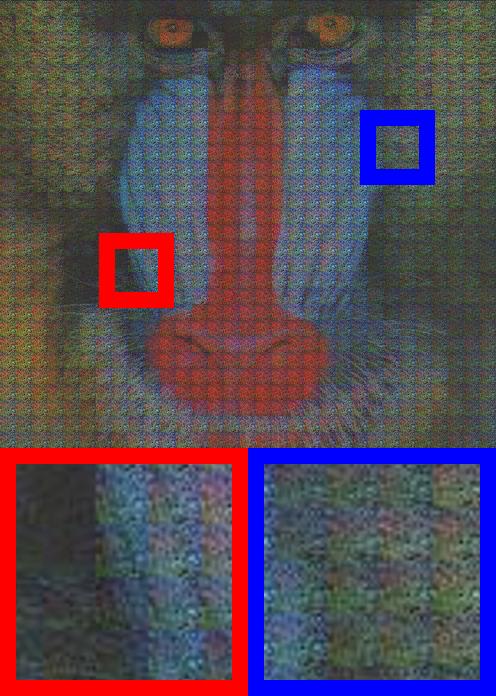}}
			\subcaption{LADMM}
		\end{subfigure}
		\begin{subfigure}[t]{0.19\textwidth}
			{\includegraphics[height=3.2cm,width=3.2cm]
				{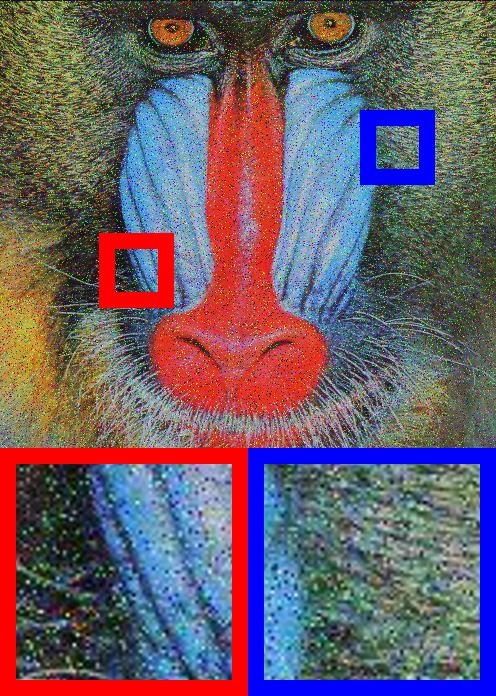}}
			\subcaption{DLADMM}
		\end{subfigure}
		\begin{subfigure}[t]{0.19\textwidth}
			{\includegraphics[height=3.2cm,width=3.2cm]
				{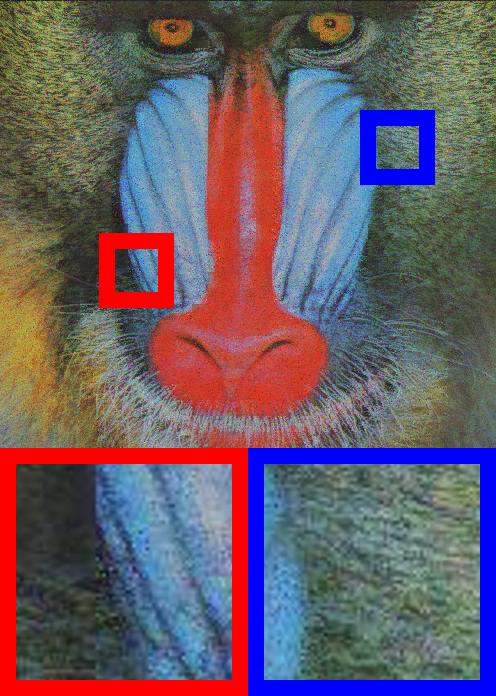}}
			\subcaption{BMO}
		\end{subfigure}
		\caption{Denoising results of the baboon image. 
			The larger red and blue boxes are the enlarged images of corresponding smaller boxes.
		}
		\label{fig:baboon}
	\end{figure*}

	\begin{figure}[ht]
		\centering
		\includegraphics[height=3cm,width=4cm]{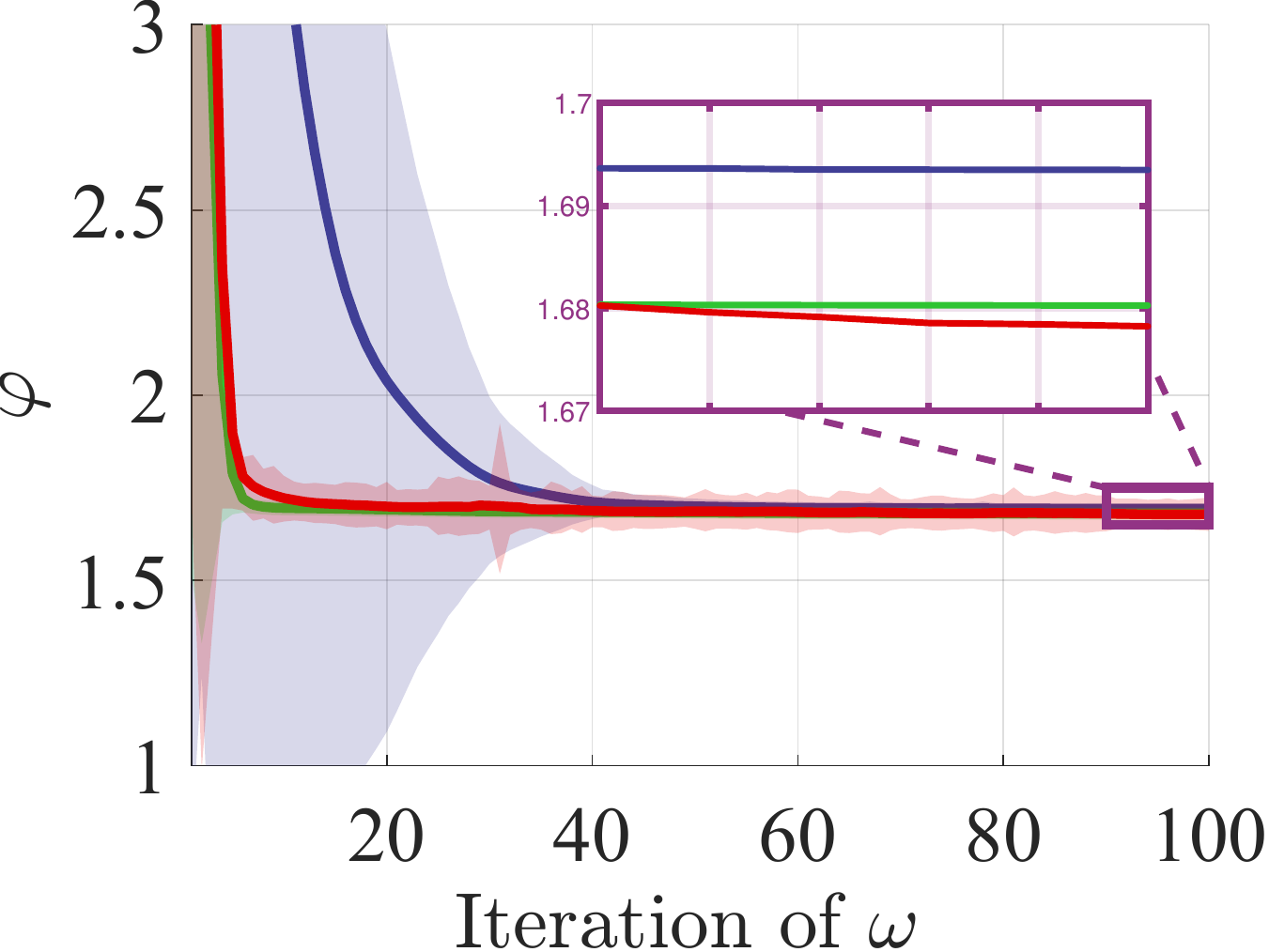}
		\includegraphics[height=3cm,width=4cm]{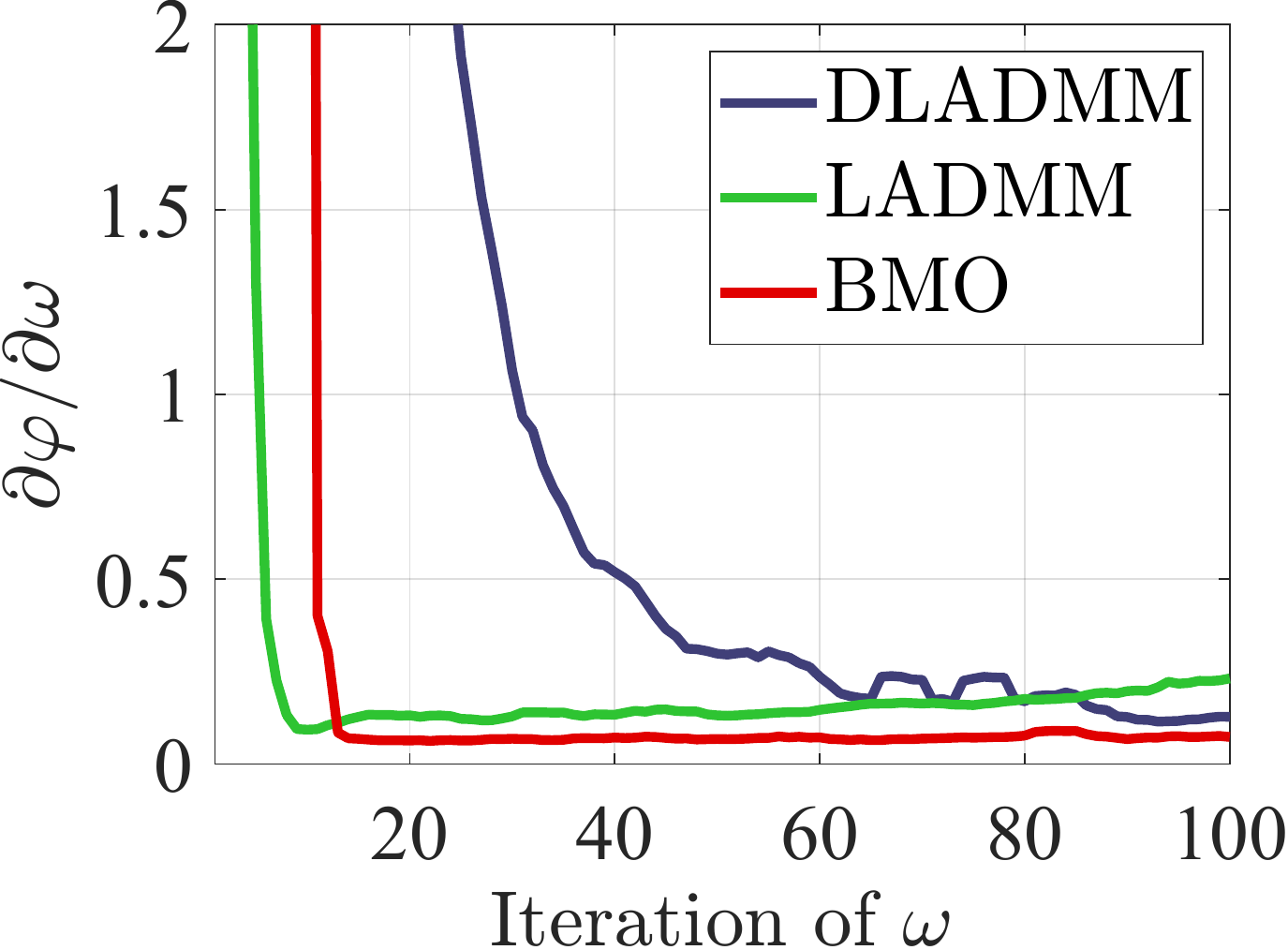}
		\caption{The convergence curves of $\varphi$ and $\frac{\partial\varphi}{\partial\ome}$ by
			DLADMM, LADMM and our BMO. 
			LADMM only learns the step size in ADMM, and DLADMM learns a matrix in ADMM. 
			It can be seen that our method achieves the optimal convergence of loss function with a stationary gradient curve.
		}
		\label{fig:convergence behavior}
	\end{figure}

	Table~\ref{tab:sparse code table} shows the PSNR and SSIM results.
	It can be seen that the performance of our BMO on both PSNR and SSIM is superior than LADMM and DLADMM.
	In Figure~\ref{fig:baboon}, we compare the visual quality of denoised images on the Baboon image, 
	and the quality of image recovered by BMO is visibly higher than that of LADMM and DLADMM. 
	This is because LADMM can only learn few hyper-parameters (such as the step size) for not destroying convergence, 
	and the structure of DLADMM is restricted by the number of layers, leading to a distance from the real fixed-point model. 
	However, thanks to the hybrid strategy to incorporate training and hyper-training, %
	BMO allows more hyper-training variables to improve the performance.
		
	\begin{table} 
		\centering
		\small
		\caption{PSNR and SSIM results for sparse coding on Set14.}
		\label{tab:sparse code table}
		\begin{tabular}{cccc} 
			\hline
			\multicolumn{1}{c}{Methods } & \multicolumn{1}{c}{Layers } & PSNR  & SSIM  \\
			\hline
			\multirow{2}{*}{DLADMM} & 5     & 15.59$\pm$0.81  & 0.52$\pm$0.13  \\
			\cline{2-4}          & 25    & 15.64$\pm$0.87  & 0.52$\pm$0.13  \\
			\hline
			\multirow{2}{*}{LADMM } & 5     & 10.47$\pm$2.36  & 0.41$\pm$0.14  \\
			\cline{2-4}          & 25    & 11.31$\pm$2.29  & 0.41$\pm$0.15  \\
			\hline
			\multirow{2}{*}{BMO} & 5     & \textbf{18.82$\pm$1.59}  & \textbf{0.63$\pm$0.16}  \\
			\cline{2-4}          & 25    & \textbf{18.98$\pm$2.53}  & \textbf{0.65$\pm$0.15}  \\
			\hline
		\end{tabular}%
	\end{table}	

	
	In Figure~\ref{fig:convergence behavior}, we further analyze the convergence behavior of hyper-training variables $\ome$ in 
	$\varphi_K(\ome)$ defined in Eq.~\eqref{eq:phiK_def}. 
	Although the final convergence of $\varphi$ by the three methods is close, 
	it can be seen from the gradient curve 
	that BMO outperforms in convergence speed and stability. 
	DLADMM can ensure the convergence of the upper loss function $\varphi$, 
	but the convergence speed is slow. 
	LADMM performs poorly in the convergence 
	of the gradient of upper loss function $\varphi$. 

	\begin{figure}[!tbp]
		\centering
		\begin{subfigure}[t]{0.48\textwidth}{
				\includegraphics[height=3cm,width=4cm]{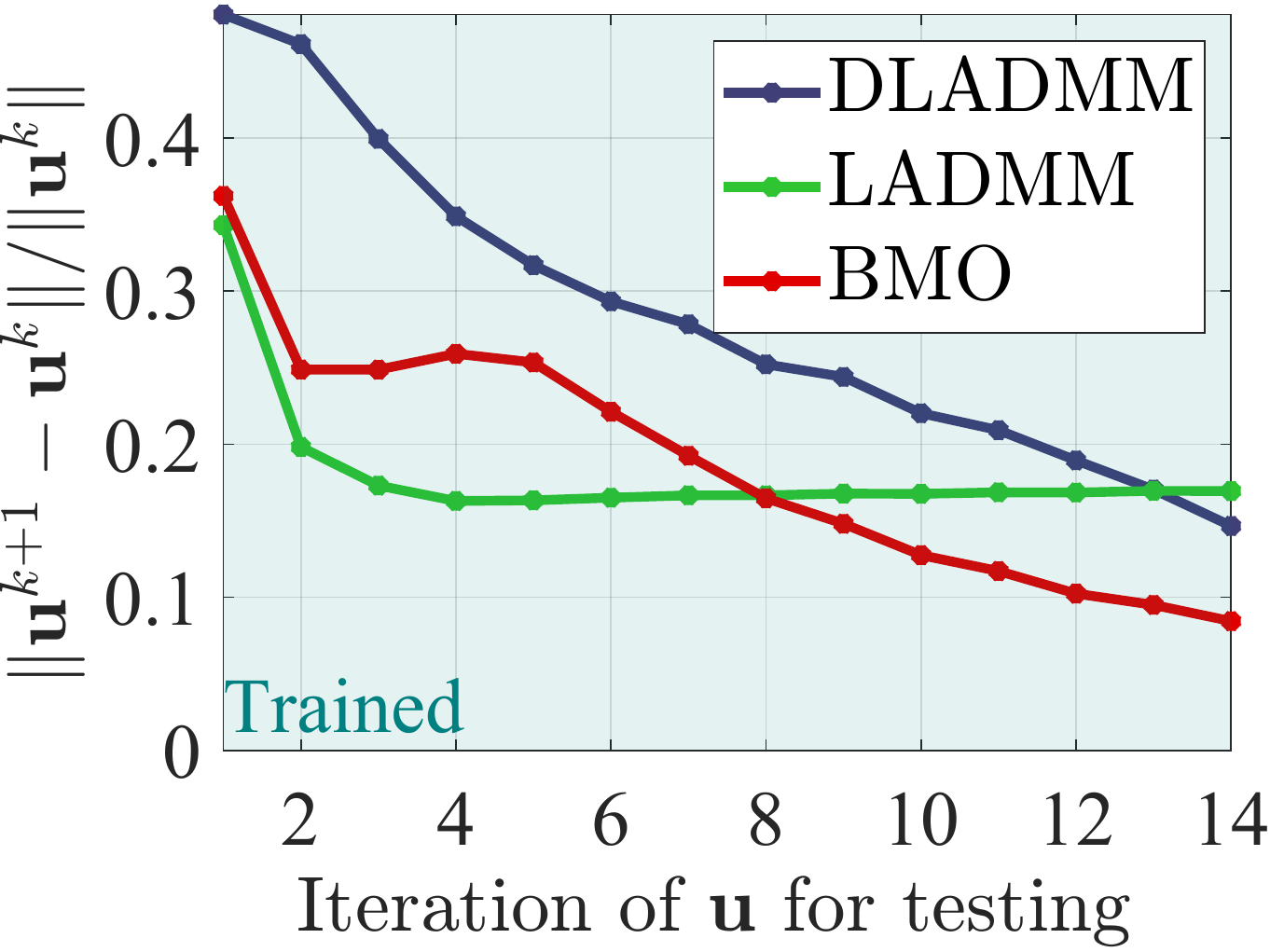}
				\includegraphics[height=3cm,width=4cm]{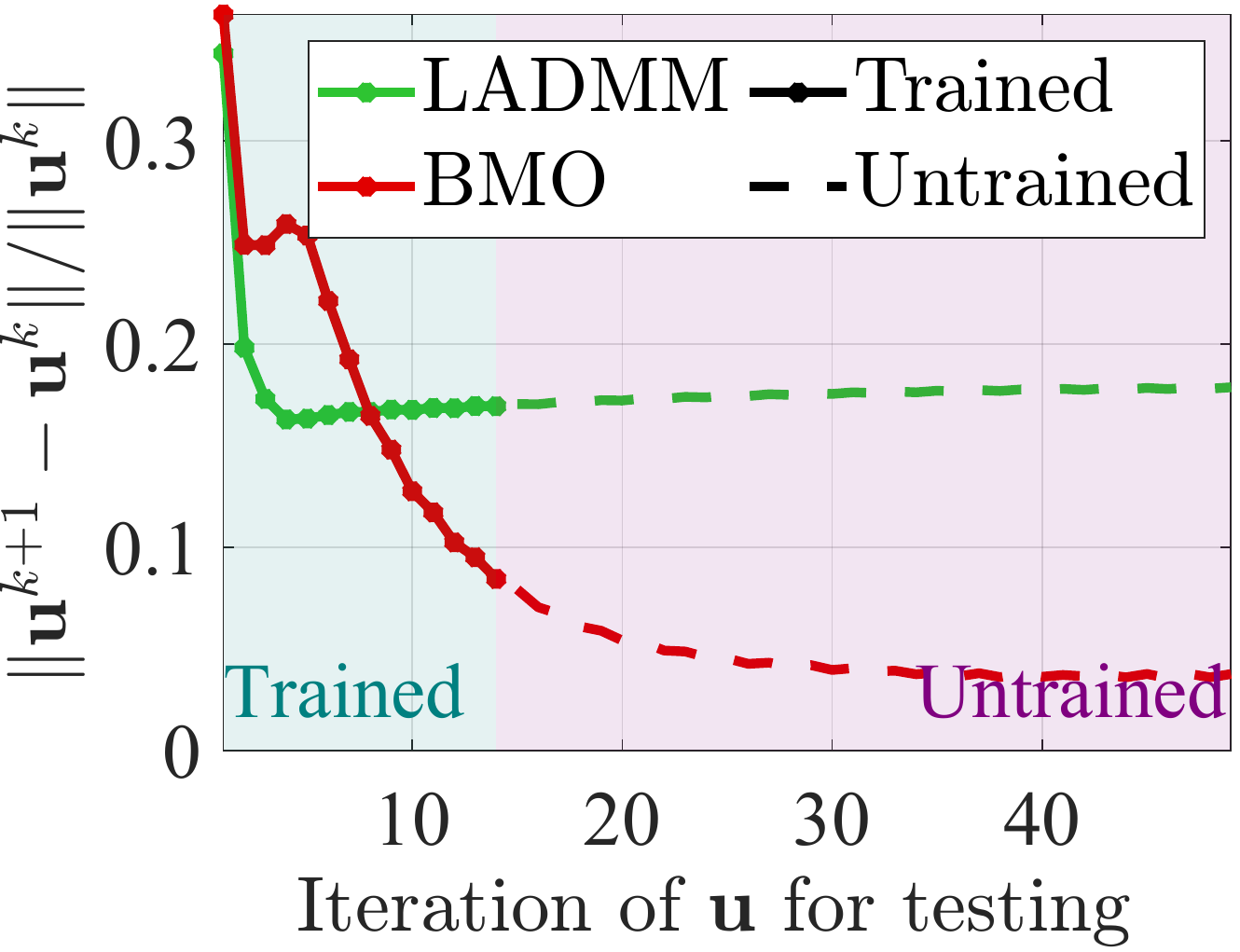}}
			\subcaption{Iterations for training = 15} 
			\label{sub:15}
		\end{subfigure}

		\begin{subfigure}[t]{0.48\textwidth}{
				\includegraphics[height=3cm,width=4cm]{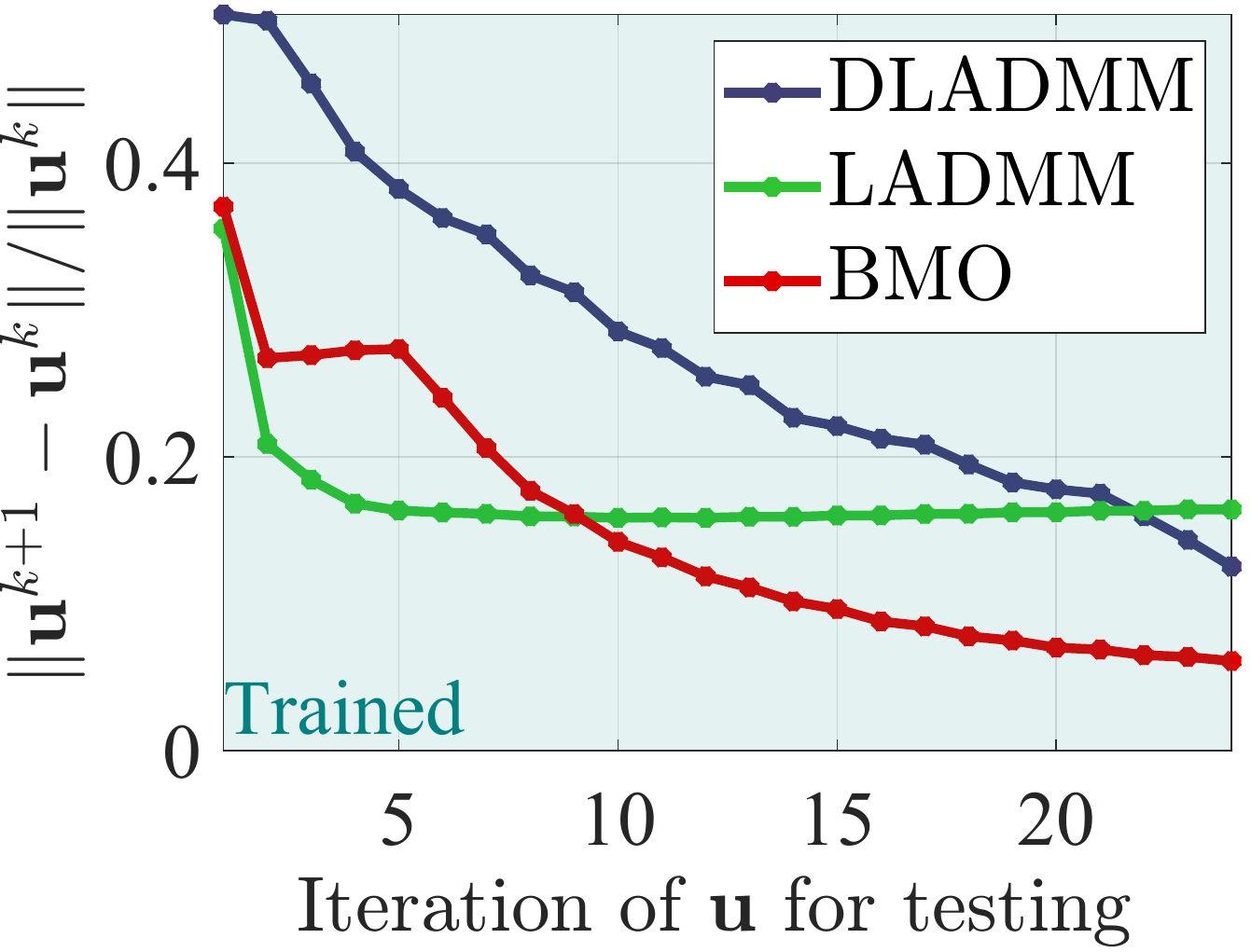}
				\includegraphics[height=3cm,width=4cm]{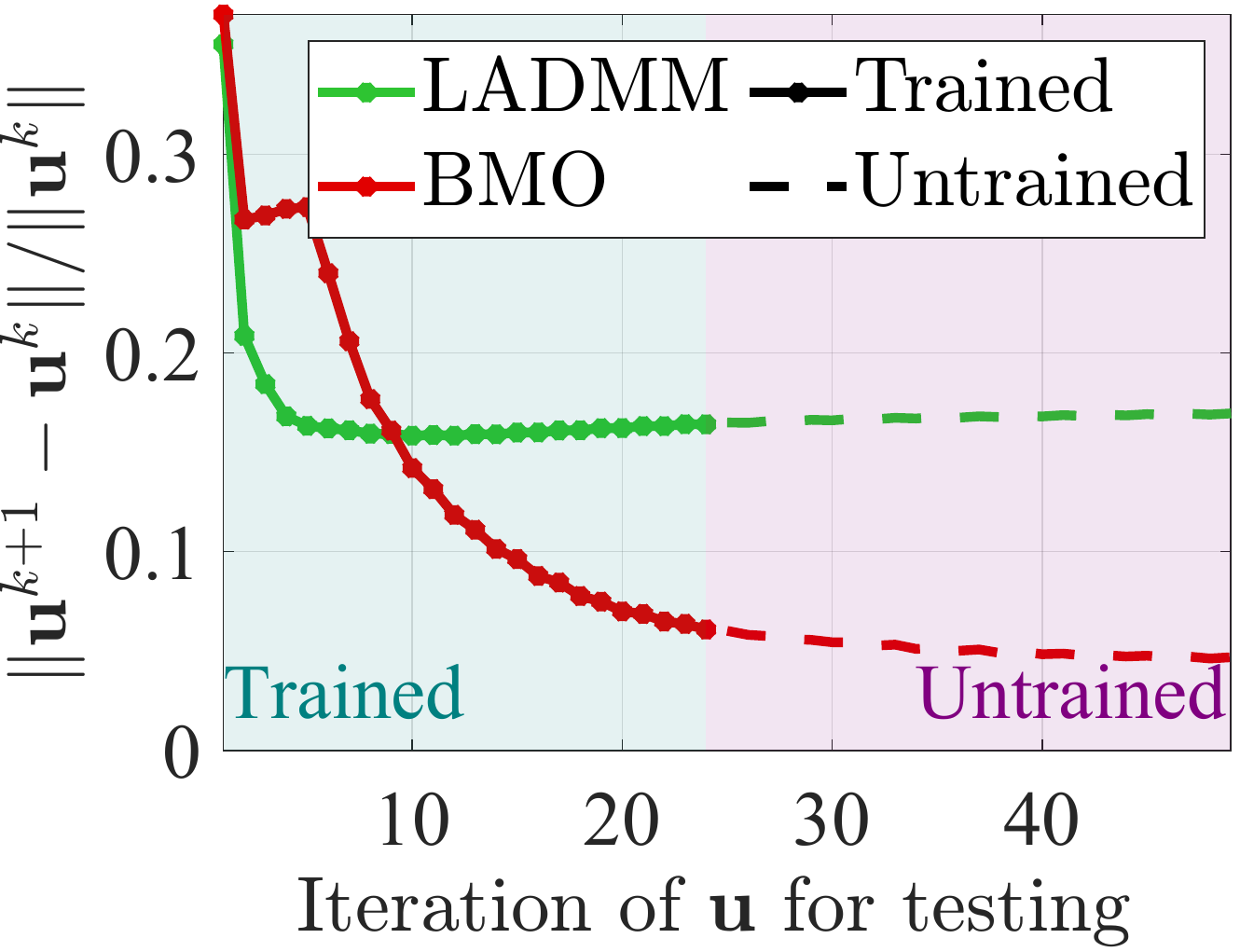}}
			\subcaption{Iterations for training = 25}
			\label{sub:5}
		\end{subfigure}
		
		\caption{Convergence curves of $\Vert\u^{k+1}-\u^{k}\Vert/\Vert\u^{k}\Vert$ with respect to $k$, the number of iterations of $\u$ for testing, after (\subref{sub:15}) 15 and (\subref{sub:5}) 25 iterations for training. 
			Solid lines on the right column represent the iterations for testing are less than those for training (trained iterations), while dotted lines represent the iterations for testing are more than those for training (untrained iterations). 
		}
		\label{fig:inner loop}
	\end{figure}

	Then, we verify the convergence of training variables $\u$ in the lower fixed-point iteration 
	with iterations of $\u$ for testing
	in Figure~\ref{fig:inner loop}.
	Note that for DLADMM, the number of iterations for training have to be more than those for testing, so in the right column we only show the curves of LADMM and BMO. 
	It can be observed that BMO performs better 
	in convergence stability and convergence speed with the increasing of iterations for testing. 
	LADMM convergences fast at first, 
	but it cannot further improve the convergence performance due to too few hyper-training variables. 
	DLADMM has slow convergence speed 
	because its number of hyper-parameters and network structure are restricted.

	In the right column of Figure~\ref{fig:inner loop}, 
	we show the convergence curve when the number of iterations for testing is more than those for training to further verify the stability and non-expansiveness of the learned lower iterative module. 
	Still, as can be seen, BMO is superior to LADMM,
	and the mapping learned by BMO on testing data can indeed continue to converge in the iterations beyond training steps, 
	implying that we have effectively learned a non-expansive mapping with convergence. 
	Another interesting finding is that even when the convergence curves oscillate a little in the trained iterations,
	our method can still learn a stable non-expansive mapping and converge successfully in the untrained iterations. 
	More investigations about the impact on the number of iterations for training are given in Appendix~\ref{sec:appendix C1 sparse coding}.


	Furthermore, we verify the influence of non-expansive property of $\D$ on the convergence of $\ome$ in Figure~\ref{SN}. 
	It can be seen that the non-expansive property reduces the gradient of $\varphi$ by an order of magnitude, and provides a better convergence of the lower iteration.
	These verify the important influence of the non-expansive property on the convergence. 

	\begin{figure}[!tbp]
		\centering
		\includegraphics[height=3cm,width=4cm]{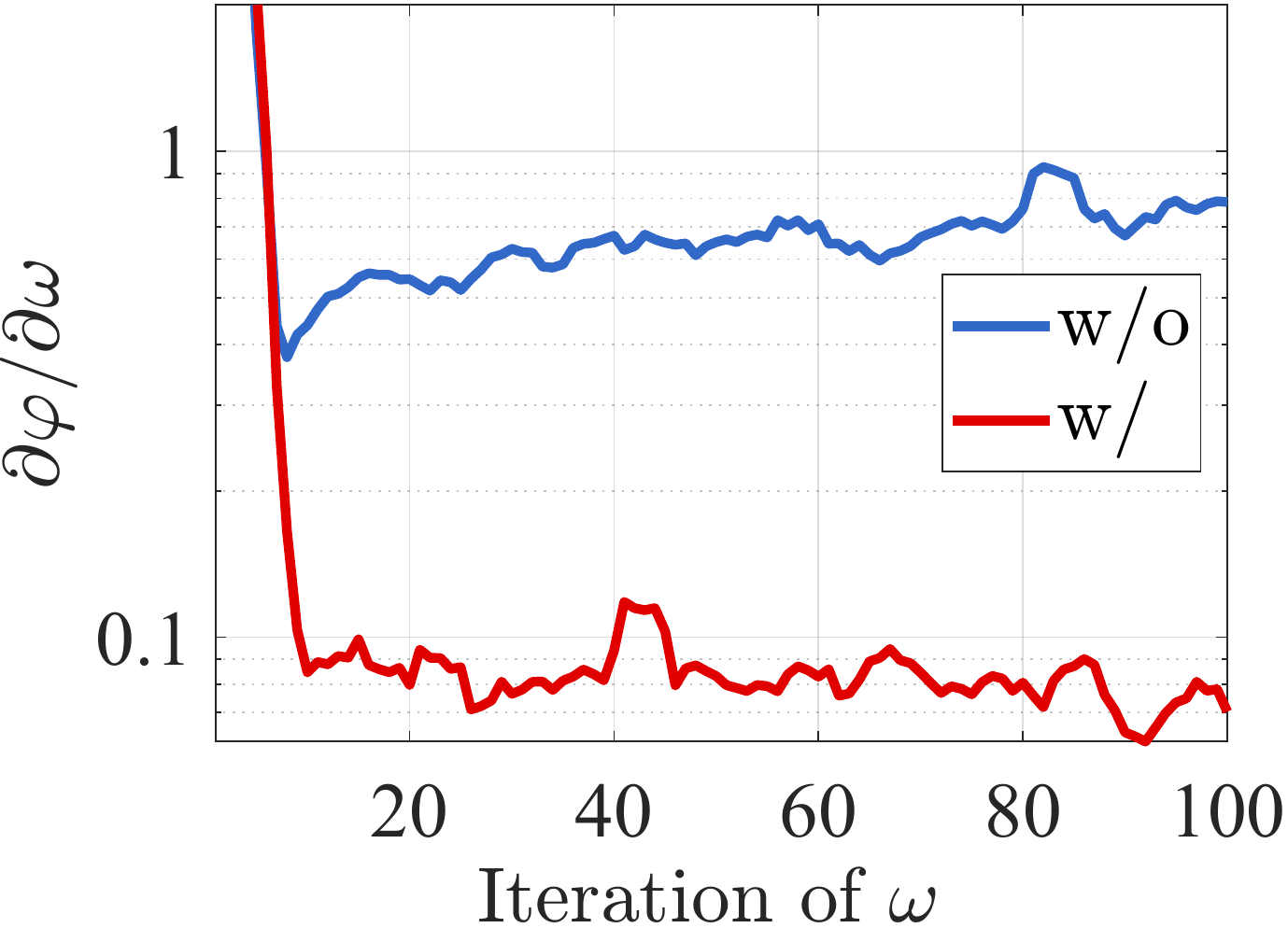}
		\includegraphics[height=3cm,width=4cm]{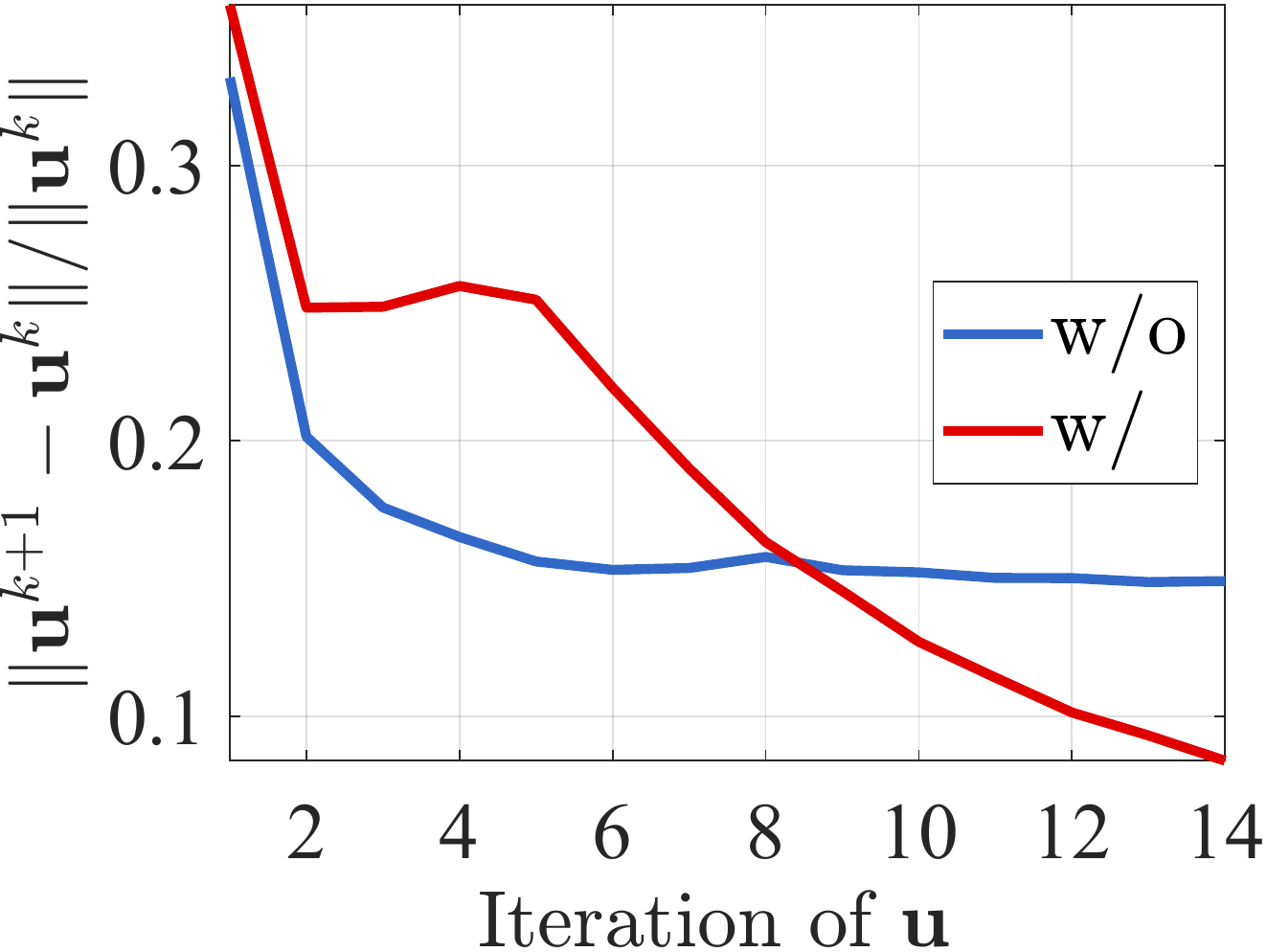}
		\caption{Convergence curves of $\frac{\partial\varphi}{\partial\ome}$ and $\Vert\u^{k+1}-\u^{k}\Vert/\Vert\u^{k}\Vert$ with or without non-expansive mapping. }
		\label{SN}
	\end{figure}

	\begin{figure*}[!htbp] 
		\captionsetup[subfigure]{justification=raggedright,singlelinecheck=false}
		\centering
		\begin{subfigure}[t]{0.13\textwidth}
			\includegraphics[height=1.2\linewidth,width=1.0\linewidth]{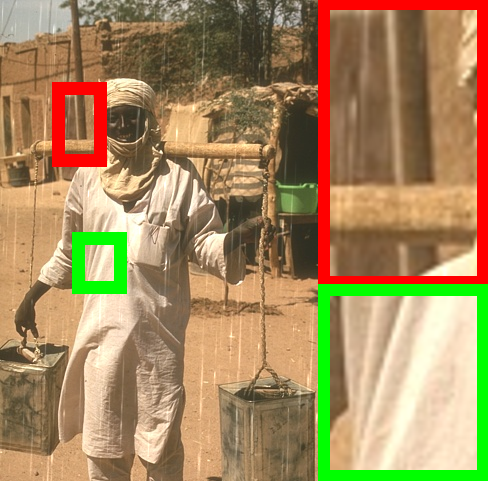}
		\end{subfigure}
		\begin{subfigure}[t]{0.13\textwidth}
			\includegraphics[height=1.2\linewidth,width=1.0\linewidth]{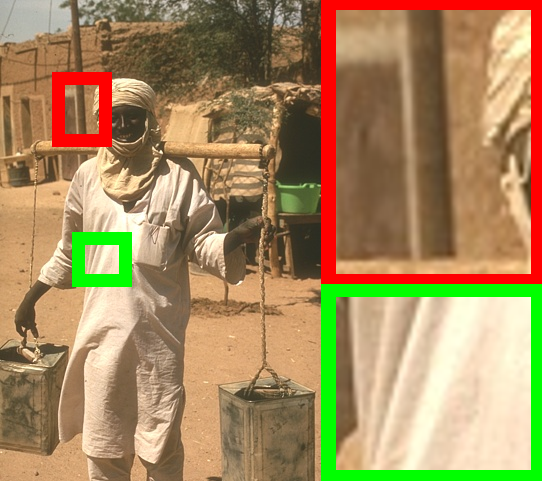}
			\subcaption*{26.59/0.930}
		\end{subfigure}
		\begin{subfigure}[t]{0.13\textwidth}
			\includegraphics[height=1.2\linewidth,width=1.0\linewidth]{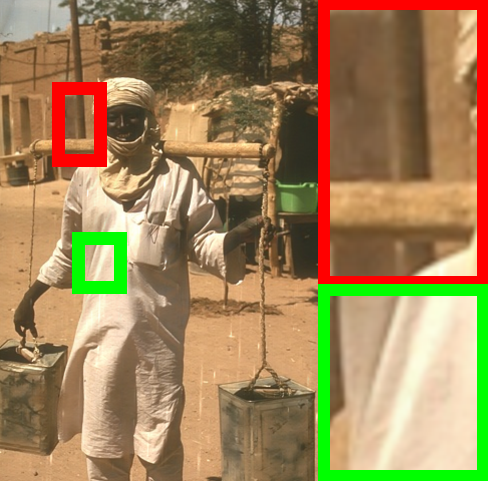}
			\subcaption*{20.49/0.594}
		\end{subfigure}
		\begin{subfigure}[t]{0.13\textwidth}
			\includegraphics[height=1.2\linewidth,width=1.0\linewidth]{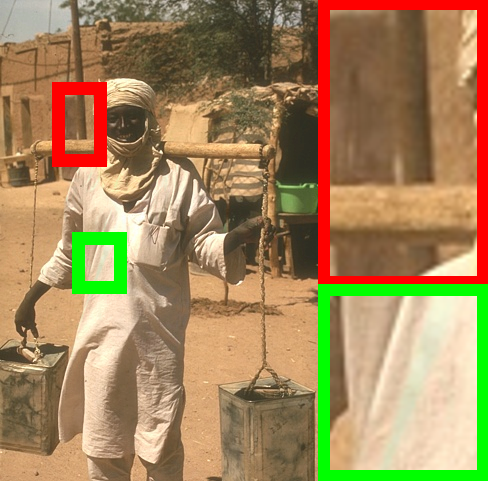}
			\subcaption*{31.13/0.965}
		\end{subfigure}
		\begin{subfigure}[t]{0.13\textwidth}
			\includegraphics[height=1.2\linewidth,width=1.0\linewidth]{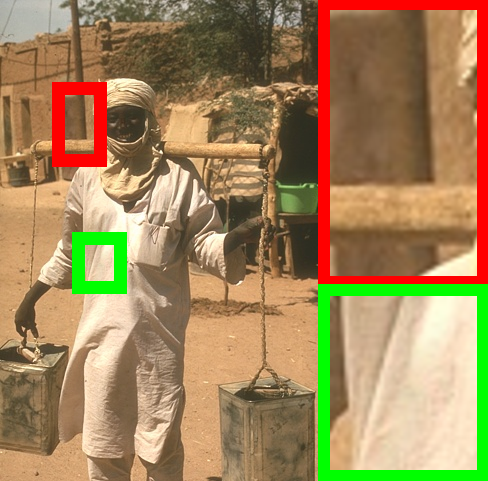}
			\subcaption*{30.89/0.967}
		\end{subfigure}
		\begin{subfigure}[t]{0.13\textwidth}
			\includegraphics[height=1.2\linewidth,width=1.0\linewidth]{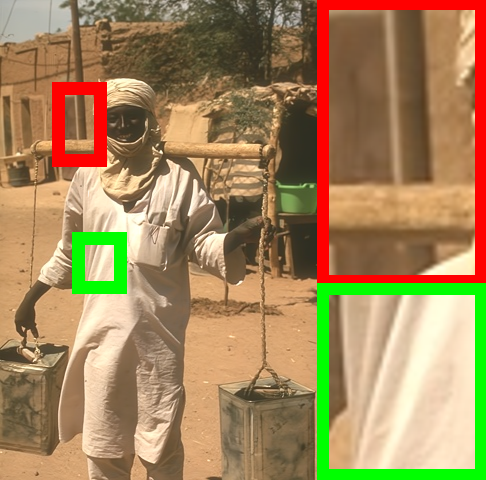}
			\subcaption*{31.19/0.968}
		\end{subfigure}
		\begin{subfigure}[t]{0.13\textwidth}
			\includegraphics[height=1.2\linewidth,width=1.0\linewidth]{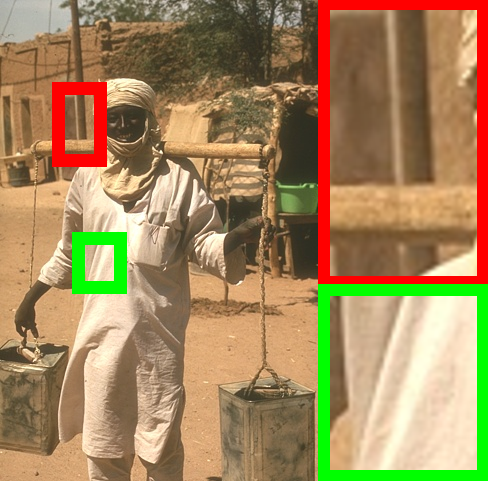}
			\subcaption*{{\textbf{33.94/0.980}}}
		\end{subfigure}
		
		\centering
		\begin{subfigure}[t]{0.13\textwidth}
			\includegraphics[height=1.2\linewidth,width=1.0\linewidth]{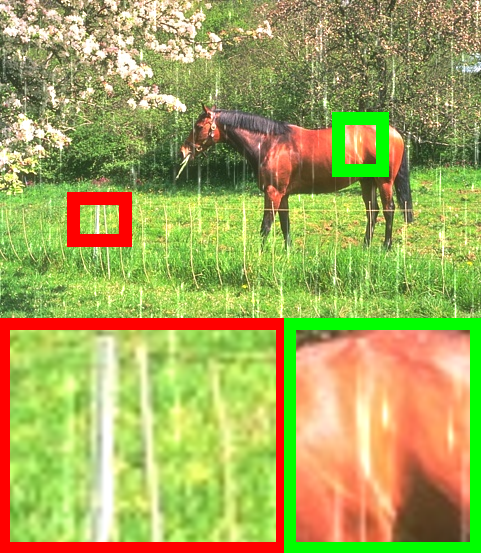}
			\subcaption*{\protect\\ Input}
		\end{subfigure}
		\begin{subfigure}[t]{0.13\textwidth}
			\centering
			\includegraphics[height=1.2\linewidth,width=1.0\linewidth]{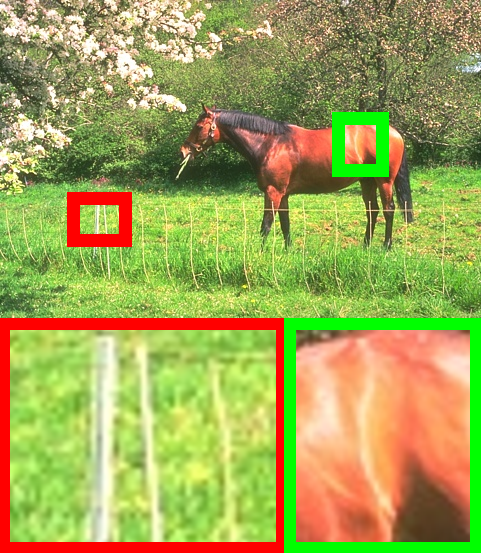}
			\subcaption*{33.76/0.945 \protect\\ Ground Truth}
		\end{subfigure}
		\begin{subfigure}[t]{0.13\textwidth}
			\includegraphics[height=1.2\linewidth,width=1.0\linewidth]{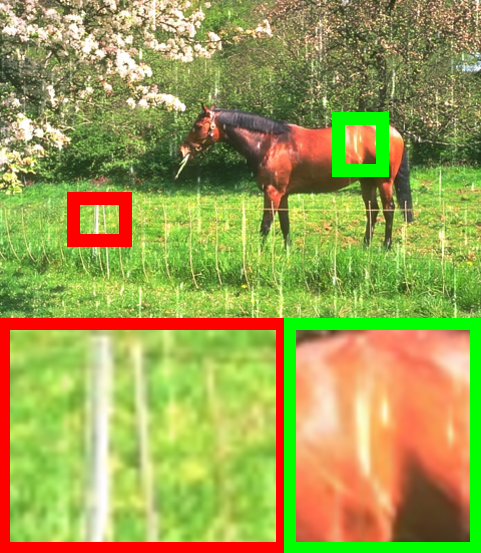}
			\subcaption*{26.04/0.811 \protect\\ DDN \protect\\ \cite{fu2017removing}}
		\end{subfigure}
		\begin{subfigure}[t]{0.13\textwidth}
			\includegraphics[height=1.2\linewidth,width=1.0\linewidth]{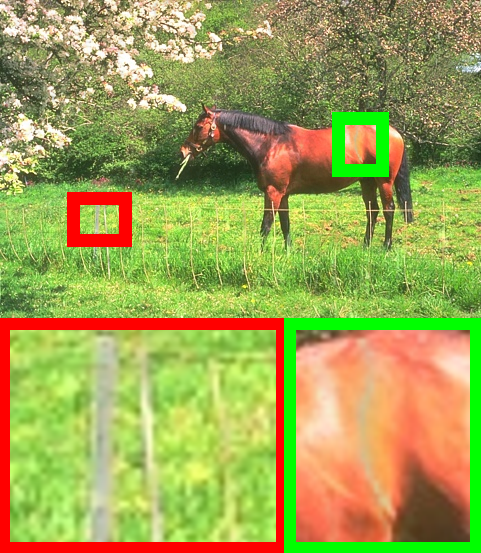}
			\subcaption*{37.39/0.978 \protect\\ JORDER\cite{derain}}
		\end{subfigure}
		\begin{subfigure}[t]{0.13\textwidth}
			\includegraphics[height=1.2\linewidth,width=1.0\linewidth]{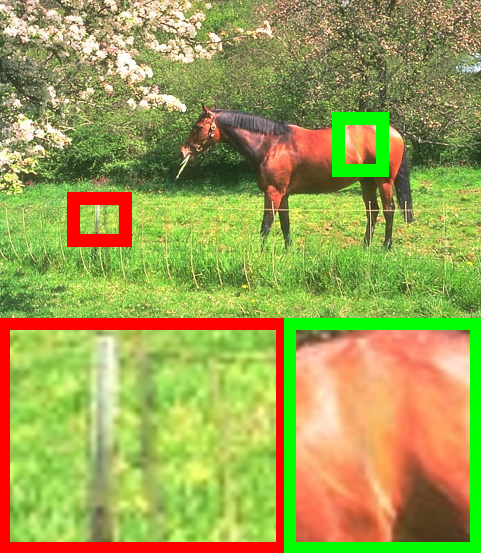}
			\subcaption*{36.72/0.980 \protect\\ PReNet\cite{ren2019progressive}}
		\end{subfigure}
		\begin{subfigure}[t]{0.13\textwidth}
			\includegraphics[height=1.2\linewidth,width=1.0\linewidth]{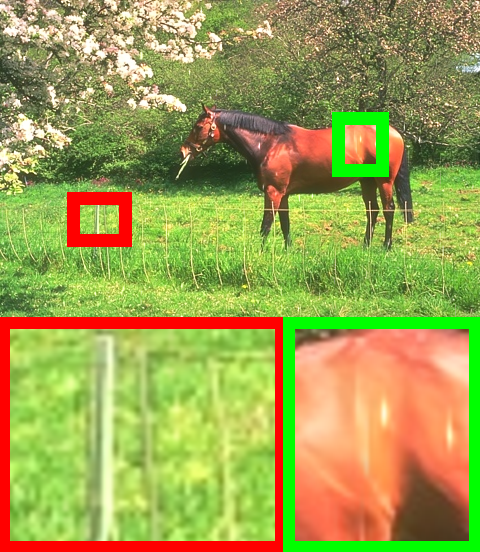}
			\subcaption*{{39.89/0.982}
			\protect\\ MPRNet\cite{zamir2021multi}}
		\end{subfigure}
		\begin{subfigure}[t]{0.13\textwidth}
			\includegraphics[height=1.2\linewidth,width=1.0\linewidth]{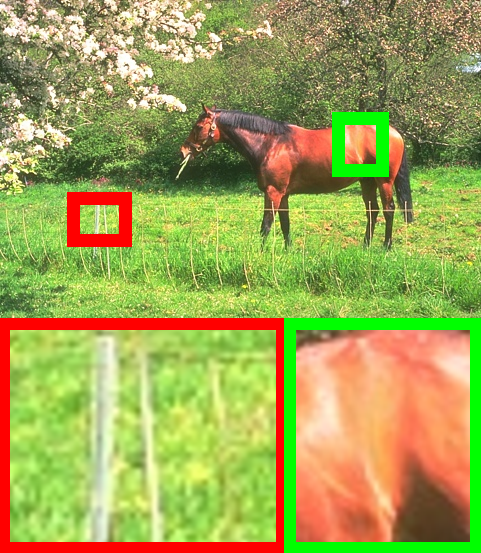}
			\subcaption*{{\textbf{42.49/0.991}} \protect\\ BMO}
		\end{subfigure}
		\caption{Performance of the rain streak removal task on two samples from Rain100L 
		The larger red and green boxes are the enlarged images of corresponding smaller boxes, and the numbers respectively correspond to PSNR and SSIM.}
		\label{fig Deraining results}
	\end{figure*}

	\subsection{Image Deconvolution}
	
	For the practical application on image deconvolution,  similar to~\cite{zhang2020plug}, 
	we use a large dataset containing 400 images from Berkeley Segmentation Dataset, 
	4744 images from Waterloo Exploration Database, 900 images from DIV2K dataset, 
	and 2750 images from Flick2K dataset.
	We test the performance of BMO on three classical testing images in Table~\ref{tab:deblurring table}. 
	and compare our method with representative methods including numerically designed method EPLL~\cite{zoran2011learning}, 
	learning-based method FDN~\cite{kruse2017learning}, 
	and plug-and-play methods including IRCNN, IRCNN+, and DPIR ~\cite{zhang2017learning,zhang2020plug}. 
	With embedded handcrafted network architectures $\D_{\net}$ and numerical schemes $\D_{\mathtt{PG}}$ to guarantee competitive performance, 
	BMO performs best in the last five columns and achieves top two in the first column 
	compared with state-of-the-art methods under different levels of noise on three testing images.
	Note that here we choose DRUNet in DPIR~\cite{zhang2020plug} as $\D_{\net}$,
	and the overall preferable results of BMO than directly using DPIR demonstrate the effect of using the composition of $\D_{\num}$ and $\D_{\net}$.

%

	\begin{table}
		\caption{PSNR (dB) results compared with state-of-the-art methods on Set3c~\cite{levin2009understanding} for image deconvolution. } 
		\renewcommand\arraystretch{1.1}
		\setlength\tabcolsep{2pt}
		\begin{center}
			\begin{small}
				\begin{tabular}{c|ccc|ccc}
					\hline
					Noise level &\multicolumn{3}{|c|}{$\sigma=1\%$}&\multicolumn{3}{|c}{$\sigma=3\%$}\\
					Image&Butterfly&Leaves&Starfish&Butterfly&Leaves&Starfish\\
					\hline
					EPLL&20.55&19.22&24.84&18.64&17.54&22.47\\
					FDN&27.40&26.51&27.48&24.27&23.53&24.71\\
					IRCNN&32.74&33.22&33.53&28.53&28.45&28.42\\
					IRCNN+&32.48&33.59&32.18&28.40&28.14&28.20\\
					DPIR&{\textbf{34.18}}&{{35.12}}&{{33.91}}&{{29.45}}&{{30.27}}&{{29.46}}\\
					BMO&{{33.67}}&{\textbf{35.39}}&{\textbf{33.98}}&{\textbf{29.46}}&{\textbf{30.69}}&{\textbf{29.64}}\\
					
					\hline
				\end{tabular}
			\end{small}
		\end{center}
		\label{tab:deblurring table}
	\end{table}

	\subsection{Rain Streak Removal} 
	
	As another example of applications, we carry out the rain streak removal experiment on synthesized rain datasets, including Rain100L and Rain100H. 
	Rain100L contains 200 rainy/clean image pairs  for training and another 100 pairs for testing,
	and Rain100H has been updated to include 1800 images for training and 200 images for testing. 
	We report the quantitative results of BMO in Table~\ref{tab:derain_table}
	with a series of state-of-the-art methods,
	containing DSC~\cite{fu2017clearing}, GMM~\cite{li2016rain}, JCAS~\cite{gu2017joint}, Clear~\cite{fu2017clearing}, DDN~\cite{fu2017removing}, RESCAN~\cite{li2018recurrent},  PReNet~\cite{ren2019progressive}, SPANet~\cite{wang2019spatial}, JORDER\_E~\cite{derain}, SIRR~\cite{wei2019semi}, MPRNet~\cite{zamir2021multi}, and RCDNet~\cite{wang2020model}. 
	It can be seen that BMO achieves higher PSNR and SSIM on both benchmark datasets.
	Note that BMO has a competitive performance compared with RCDNet, but BMO also provides better theoretical property.

	\begin{table}
		\small
		\centering
		\caption{Averaged PSNR and SSIM results among various methods for the single image rain removal task on two widely used synthesized datasets: Rain100L and Rain100H~\cite{derain}. }
		
		\begin{tabular}{c|cc|cc}
			\hline   { \multirow{1}{*}{Datasets}  } &   \multicolumn{2}{c|}{ Rain 100L } & \multicolumn{2}{c}{  { Rain 100H }} \\
			{Metrics} &   { PSNR} &   {SSIM } &   { PSNR} &   {SSIM } \\
			\hline   
			{ Input } & 26.90 & 0.838 & 13.56 & 0.370 \\
			{ DSC } & 27.34 & 0.849 & 13.77 & 0.319 \\
			{ GMM } & 29.05 & 0.871 & 15.23 & 0.449 \\
			{ JCAS } & 28.54 & 0.852 & 14.62 & 0.451 \\
			{ Clear } & 30.24 & 0.934 & 15.33 & 0.742 \\
			{ DDN } & 32.38 & 0.925 & 22.85 & 0.725 \\
			{ RESCAN } & 38.52 & 0.981 & 29.62 & 0.872 \\
			{ PReNet } & 37.45 & 0.979 & 30.11 & 0.905 \\
			{ SPANet } & 35.33 & 0.969 & 25.11 & 0.833 \\
			\text{ JORDER\_E } & 38.59 & 0.983 & 30.50 & 0.896 \\
			{ SIRR } & 32.37 & 0.925 & 22.47 & 0.716 \\
			{ MPRNet } & 36.40 & 0.965 & 30.41 & 0.890 \\
			{ RCDNet } & 40.00 & 0.986 & \textbf{31.28} & 0.903 \\
			{ BMO } & \textbf{40.07} & \textbf{0.986} &30.96 & \textbf{0.905} \\
			\hline

		\end{tabular}
		\label{tab:derain_table}
		
	\end{table}

				In Figure~\ref{fig Deraining results},
	we visually report the performance of rain streak removal task on two images from Rain100L~\cite{derain} compared with DDN, JORDER, PReNet and MPRNet.
	From the first row, one can observe that our BMO preserves the outline of the door in the background when removing the rain streak, 
	compared with DDN, JORDER, and PReNet, 
	which remain less details of the outline. 
	In the second row, 
	other learning-based methods tend to blur some image textures or leave distinct rain marks, 
	while BMO remains the original color and outline, and gains better performance on PSNR and SSIM.

	\section{Conclusions} 

	This paper first introduces the GKM scheme to unify various existing ODL approaches,
	and then proposes our BMO algorithmic framework to jointly solve the training and hyper-training task.
	We prove the essential convergence of training and hyper-training variables, from the perspective of both the approximation quality, and the stationary analysis.
	Experiments demonstrate our efficiency on sparse coding and  real-world applications on image processing.

	%
	%
	
	\section*{Acknowledgements}
	
	This work is partially supported by 	
	the National Natural Science Foundation of China (Nos. 61922019, 61733002, 62027826, and 11971220), 
	the National Key R \& D Program of China (No. 2020YFB1313503), 
	the major key project of PCL (No. PCL2021A12),
	the Fundamental Research Funds for the Central Universities,
	Shenzhen Science and Technology Program (No. RCYX20200714114700072), 
	and Guangdong Basic and Applied Basic Research Foundation (No. 2022B1515020082).

	\newpage 
	\bibliography{ODLdraft}
	\bibliographystyle{icml2022}

	\newpage
	\appendix
	\onecolumn
	\section*{Appendices}

	\section{Detailed Proofs} \label{sec:appendix proof}

	In this section,
	we discuss the convergence analysis of our proposed BMO algorithm (Algorithm~\ref{alg:bmo}) for the GKM scheme,
	towards the optimal solution and stationary points of optimization problem in Eq.~\eqref{bilevel_fix} with respect to both $\u$ and $\ome$. 
	Note that this joint convergence of training and hyper-training also provides a unified theoretical guarantee for existing ODL methods containing UNH and ENA.

	By introducing an auxiliary function, the problem in Eq.~\eqref{bilevel_fix} can be equivalently rewritten as the following 
	\begin{equation}\label{eq:appe phi_def}
		\min_{\ome \in \Omega} \ \varphi(\ome),\quad  \text{where} \quad \varphi(\ome) := \inf_{\u \in \mathtt{Fix}(\T(\cdot,\ome)) \cap U } \ \ell(\u,\ome).
	\end{equation}
	The sequence $\{\ome^t\}$ generated by BMO (Algorithm~\ref{alg:bmo}) actually solves the following approximation problem of Eq.~\eqref{bilevel_fix}
	\begin{equation}\label{phiK_def}
		\min_{\ome \in \Omega} \ \varphi_K(\ome) := \ell(\u^K(\ome),\ome),
	\end{equation}
	where $\u^K(\ome)$ is derived by solving the simple bilevel problem $\inf_{\u \in \mathtt{Fix}(\T(\cdot,\ome)) \cap U } \ \ell(\u,\ome)$ and can be given by 
	\begin{equation}\label{simple_bilevel_alg}
		\left\{
		\begin{aligned}
			\v^k_l(\ome) & =  \T(\u^{k-1}(\ome),\ome), \\
			\v^k_u(\ome) & = \u^{k-1}(\ome) - s_k \H_{\ome}^{-1}\frac{\partial }{\partial \u}\ell(\u^{k-1},\ome), \\
			\u^k(\ome) & = \mathtt{Proj}_{U,\H_{\ome}} \big( \mu \v^k_u(\ome) + (1-\mu) \v^k_l(\ome) \big),
		\end{aligned}\right.
	\end{equation}
	where $k = 1,\ldots, K$.

	\subsection{Approximation Quality and Convergence}
	In this part, we will show that Eq.~\eqref{phiK_def} is actually an appropriate approximation to Eq.~\eqref{bilevel_fix} in the sense that any limit point $(\bar{\u},\bar{\ome})$ of the sequence $\left\{\left( \u^K(\ome^K),\ome^K\right)\right\}$ with $\ome^K \in \mathrm{argmin}_{\ome \in\Omega}\varphi_{K}(\ome)$ is a solution to the bilevel problem in Eq.~\eqref{bilevel_fix}.
	Thus we can obtain the optimal solution of Eq.~\eqref{bilevel_fix} by solving Eq.~\eqref{phiK_def}.
	We make the following standing assumption throughout this part.

	\begin{assumption}\label{assum_F}
		$\Omega$ is a compact set and $U$ is a convex compact set. $\mathtt{Fix}(\T(\cdot,\ome))$ is nonempty for any $\ome \in \Omega$. $\ell(\u,\ome)$ is continuous on $\R^n \times \Omega$. For any $\ome \in \Omega$, $\ell(\cdot,\ome) : \R^n \rightarrow \R$ is $L_\ell$-smooth, convex and bounded below by $M_0$.
	\end{assumption}
	Please notice that $\ell$ is usually defined to be the MSE loss, and thus Assumption~\ref{assum_F} is quite standard for ODL~\cite{ryu2019plug,zhang2020plug}. 
	We first present some necessary  preliminaries.
	For any two matrices $\H_1, \H_2 \in \R^{n \times n}$, we consider the following partial ordering relation:
	\[
	\H_1 \succeq \H_2 \quad \Leftrightarrow \quad  \langle \u,\H_1\u \rangle \ge \langle \u,\H_2\u \rangle, \quad \forall \u \in \R^n.
	\]
	If $\H \succ 0$, $\langle \u_1, \H \u_2 \rangle$ for $\u_1,\u_2 \in \R^n$ defines an inner product on $\R^n$. 
	Denote the induced norm with $\| \cdot \|_\H$, 
	i.e., $\| \u \|_\H := \sqrt{\langle \u,\H \u \rangle}$ for any $\u \in \R^n$.
	We assume $\D(\cdot,\ome)$ satisfies the following assumption throughout this part.
	\begin{assumption} \label{assum_T}
		There exist $\H_{ub} \succeq \H_{lb} \succ 0$, 
		such that for each $\ome \in \Omega$, there exists $\H_{ub} \succeq \H_{\ome} \succeq \H_{lb}$ such that
		\begin{itemize}
			\item[(1)] $\D(\cdot,\ome)$ is non-expansive with respect to $\| \cdot \|_{\H_{\ome}}$, i.e., for all $(\u_1,\u_2) \in \R^n \times \R^n$,
			\begin{equation*}
				\|\D(\u_1,\ome) - \D(\u_2,\ome) \|_{\H_{\omega}} \le \| \u_1 - \u_2\|_{\H_{\ome}}.
			\end{equation*}
			\item[(2)] $\D(\cdot,\ome)$ is closed, i.e.,
			\[
			\mathrm{gph} \,\D(\cdot,\ome) := \{(\u,\v) \in \R^n \times \R^n ~|~ \v = \D(\u,\ome)\}
			\]
			is closed.
		\end{itemize}
	\end{assumption}

	Under Assumption~\ref{assum_T}, we obtain the following non-expansive properties of $\T(\cdot,\ome)$ defined in Eq.~\eqref{eq:gkm} from~\cite{Heinz-MonotoneOperator-2011}[Proposition~4.25] immediately.
	\begin{lemma}\label{lemma2.1}
		Given $\alpha \in (0,1)$, $\ome \in \Omega$, let $\T(\cdot,\ome) := (1-\alpha)\I + \alpha\D(\cdot,\ome)$, where $\I$ denotes the identity operator, then $\T(\cdot,\ome)$ is closed and satisfies that for any $(\u_1,\u_2) \in \R^n \times \R^n,$
		\begin{equation}
			\begin{aligned}
				&\| \u_1 - \u_2\|^2_{\H_{\ome}} - \|\T(\u_1,\ome) - \T(\u_2,\ome)\|_{\H_{\ome}}^2 \\
				&\ge \frac{1-\alpha}{\alpha} \| (\u_1 - \T(\u_1,\ome))- (\u_2 - \T(\u_2,\ome))\|_{\H_{\ome}}^2
				\ge 0.
			\end{aligned}		
		\end{equation}
	\end{lemma}
	
	In this part, as the identity of $\ome$ is clear from the context, 
	for succinctness we will write $\ell(\u)$ instead of $\ell(\u,\ome)$,
	$\T(\u)$ instead of $\T(\u,\ome)$, 
	$\H$ instead of $\H_{\ome}$, 
	$\ell^*$ instead of $\varphi(\ome)$, 
	$\S$ instead of $\mathrm{Fix}(\T(\cdot, \ome))$, 
	and $\hat{\S}$ instead of $ \hat{\S}(\ome):= \mathrm{argmin}_{\u \in \mathtt{Fix}(\T(\cdot, \ome)) \cap U } \ell (\u,\ome)$. 
	Moreover, we will omit the notation $\ome$ and use the notations $\u^k$, $\v^{k}_{u}$ and $\v^{k}_l$ instead of $\u^k(\ome)$, $\v^{k}_{u}(\ome)$ and $\v^{k}_l(\ome)$, respectively.
	
	
	Before giving the proof of Theorem~\ref{thm:simple_bilevel_convergence}, we present some helpful lemmas and propositions. The proof is inspired by~\cite{BDA}.
	\begin{lemma} \label{simple_bilevel_lem}
		For any given $\ome \in \Omega$, let $\{\u^k\}$ be the sequence generated by Eq.~\eqref{simple_bilevel_alg} with $s_k = \frac{s}{k+1}$, $s \in (0, \frac{\lambda_{\min}(\H_{lb})}{L_{\ell}} )$ and $\mu \in (0,1)$, 
		where $\lambda_{\min}(\H_{lb})$ denotes the smallest eigenvalue of matrix $\H_{lb}$ in Assumption~\ref{assum_T}.
		Then for any $\u \in \S$, we have
		\begin{equation}\label{eq:alg_p_lem1_eq}
			\begin{aligned}
				\mu s \beta_k \ell(\u) \ge\, & \mu s \beta_k \ell(\v^{k+1}_{u}) -  \frac{1}{2} \|\u - \u^k\|_\H^2 
				+ \frac{1}{2}\|\u - \u^{k+1}\|_\H^2\\
				& + \frac{1}{2}\left\| \left((1-\mu) \v^{k+1}_l + \mu \v^{k+1}_{u} \right) - \u^{k+1} \right\|_\H^2 \\
				& + \frac{(1-\mu)\eta}{2} \|\u^k - \v^{k+1}_l\|_\H^2
				+ \frac{\mu}{2}(1 - s_k \frac{L_{\ell}}{\lambda_{\min}(\H)}) \|\u^k - \v^{k+1}_{u}\|_\H^2,
			\end{aligned}
		\end{equation}
		where $\beta_k = \frac{1}{k+1}$ and $\eta = \frac{1-\alpha}{\alpha} > 0$.
	\end{lemma}
	
	\begin{proof}
		We can obtain from the definition0 of $\v^k_u$ that
		\begin{equation}\label{alg_p_eqs_opt_con}
			0 = \beta_k\nabla \ell(\u^k) + \frac{1}{s}\H(\v^{k+1}_u - \u^k),
		\end{equation}
		where $\beta_k = \frac{1}{k+1}$, and thus for any $\u$,
		\begin{equation}
			\begin{array}{l}
				0 =  \beta_k\langle \nabla \ell(\u^k) , \u - \v^{k+1}_{u} \rangle + \frac{1}{s} \langle   \v^{k+1}_{u} - \u^k, \H(\u - \v^{k+1}_{u})\rangle. \label{eq:alg_pp_lem1_eq2}
			\end{array}
		\end{equation}
		Since $\ell$ is convex and $\nabla \ell$ is $L_{\ell}$-Lipschitz continuous, we have
		\begin{equation}\label{alg_p_eqs_lip_con}
			\begin{aligned}
				&\langle \nabla \ell(\u^k), \u - \v^{k+1}_{u} \rangle  \\
				= \,  &\langle \nabla \ell(\u^k), \u - \u^k\rangle +  \langle \nabla \ell(\u^k), \u^k - \v^{k+1}_{u} \rangle \\
				\le\, &\ell(\u) - \ell(\u^k) + \ell(\u^k) - \ell(\v^{k+1}_{u}) + \frac{L_{\ell}}{2}\|\u^k- \v^{k+1}_{u}\|^2  \\
				\le \,&\ell(\u) -  \ell(\v^{k+1}_{u}) + \frac{L_{\ell}}{2\lambda_{\min}(\H)}\|\u^k- \v^{k+1}_{u}\|_\H^2.
			\end{aligned}
		\end{equation}
		Combining with $\langle \v^{k+1}_{u} - \u^k , \H(\u - \v^{k+1}_{u})\rangle = \frac{1}{2} \left(\|\u - \u^k\|_\H^2 - \|\u- \v^{k+1}_{u}\|_\H^2 - \|\u^k - \v^{k+1}_{u}\|_\H^2 \right)$ and Eq.~\eqref{eq:alg_pp_lem1_eq2} yields that for any $\u$,
		\begin{equation}\label{eq:alg_pp_lem1_eq4}
			\begin{aligned}
				\beta_k \ell(\u) \ge \beta_k\ell(\v^{k+1}_{u}) - \frac{1}{2s}\|\u - \u^k\|_\H^2
				+ \frac{1}{2s} \|\u- \v^{k+1}_{u}\|_\H^2\\
				+ \frac{1}{2s}(1 - s \beta_k \frac{L_{\ell}}{\lambda_{\min}(\H)})\|\u^k - \v^{k+1}_{u}\|_\H^2.
			\end{aligned}
		\end{equation}
		Next, since $\v^{k+1}_l = \T(\u^k)$ and $\T$ satisfies the inequality in Lemma~\ref{lemma2.1}, we have for any $\u \in \S$,
		\begin{equation}\label{eq:alg_pp_lem1_eq3}
			\|\u^k - \u\|_\H^2 - \|\v^{k+1}_l - \u\|_\H^2 \ge \eta \|\u^k - \v^{k+1}_l\|_\H^2,
		\end{equation}	
		with $\eta = \frac{1-\alpha}{\alpha} > 0$. Multiplying Eq.~\eqref{eq:alg_pp_lem1_eq4} and Eq.~\eqref{eq:alg_pp_lem1_eq3} by $\mu s$ and $\frac{1-\mu}{2}$, respectively, and then summing them up yields that for any $\u \in \S$, 
		\begin{equation}\label{eq:alg_pp_lem1_eq5}
			\begin{aligned}
				\mu s \beta_k \ell(\u) \ge\, & \mu s \beta_k\ell(\v^{k+1}_{u}) -  \frac{1}{2} \|\u - \u^k\|_\H^2
				+ \frac{1}{2}\left( (1-\mu)  \left\|\u - \v^{k+1}_l\right\|_\H^2 + \mu \left\|\u- \v^{k+1}_{u}\right\|_\H^2 \right)\\
				&+ \frac{(1-\mu)\eta}{2} \|\u^k - \v^{k+1}_l\|_\H^2
				+ \frac{\mu}{2}(1 - s \beta_k \frac{L_{\ell}}{\lambda_{\min}(\H)}) \|\u^k - \v^{k+1}_{u}\|_\H^2 .
			\end{aligned}
		\end{equation}
		The convexity of $\|\cdot\|_\H^2$ implies that
		\begin{equation*}
			\begin{aligned}
				(1-\mu)  \|\u - \v^{k+1}_l\|_\H^2 & + \mu \|\u- \v^{k+1}_{u}\|_\H^2\\
				& \ge \| \u - \left((1-\mu) \v^{k+1}_l + \mu \v^{k+1}_{u} \right) \|_\H^2.
			\end{aligned}	
		\end{equation*}
		Next, as $\mathtt{Proj}_{U,\H}$ is firmly non-expansive with respect to $\|\cdot\|_\H$ (see, e.g.,\cite{Heinz-MonotoneOperator-2011}[Proposition 4.8]), for any $\u \in U$, we have
		\begin{equation}\label{eq:alg_pp_lem1_eq4.5}
			\begin{aligned}
				&\left\| \u - \left((1-\mu) \v^{k+1}_l + \mu \v^{k+1}_{u} \right) \right\|_\H^2 \\
				&\ge \|\u - \u^{k+1}\|_\H^2 + 
				\left\| \left((1-\mu) \v^{k+1}_l + \mu \v^{k+1}_{u} \right) - \u^{k+1} \right\|_\H^2.
			\end{aligned}
		\end{equation}
		Then, 
		we obtain from Eq.~\eqref{eq:alg_pp_lem1_eq5} that for any $\u \in \S$,
		\begin{equation}
			\begin{aligned}
				\mu s \beta_k \ell(\u) \ge\, & \mu s \beta_k \ell(\v^{k+1}_{u}) -  \frac{1}{2} \|\u - \u^k\|_\H^2 
				+ \frac{1}{2}\|\u - \u^{k+1}\|_\H^2\\
				& + \frac{1}{2}\left\| \left((1-\mu) \v^{k+1}_l + \mu \v^{k+1}_{u} \right) - \u^{k+1} \right\|_\H^2 \\
				& + \frac{(1-\mu)\eta}{2} \|\u^k - \v^{k+1}_l\|_\H^2\\
				& + \frac{\mu}{2}(1 - s \beta_k \frac{L_{\ell}}{\lambda_{\min}(\H)}) \|\u^k - \v^{k+1}_{u}\|_\H^2 .
			\end{aligned}
		\end{equation} 
		This completes the proof. 
	\end{proof}
	
	\begin{lemma}\label{lem_bounded}
		For any given $\ome \in \Omega$, let $\{\u^k\}$ be the sequence generated by Eq.~\eqref{simple_bilevel_alg} with $s_k = \frac{s}{k+1}$, $s \in (0, \frac{\lambda_{\min}(\H_{lb})}{L_{\ell}} )$ and $\mu\in (0,1)$.
		Then for any $\bar{\u} \in \S$, we have
		\begin{equation}
			\begin{aligned}
				\|\v^{k+1}_l - \bar{\u}\|_\H &\le \|\u^k - \bar{\u}\|_\H,
			\end{aligned}
		\end{equation}
		and operator $\I - s_k \H^{-1}\nabla \ell$ is non-expansive (i.e., $1$-Lipschitz continuous) with respect to $\|\cdot\|_\H$.
		Furthermore, when $U$ is compact, sequences $\{\u^k\}$, $\{\v^{k}_l \}$, $\{\v^{k}_u \}$ are all bounded. 
	\end{lemma}
	\begin{proof}
		Since $\v^{k+1}_l = \T(\u^k)$ and $\T$ satisfies the inequality in Lemma~\ref{lemma2.1}, we have
		\[
		\| \v^{k+1}_l - \bar{\u}\|_\H \le \|\u^k - \bar{\u}\|_\H.
		\]
		When $U$ is compact, the desired boundedness of $\{\u^k\}$ follows directly from the iteration scheme given in Eq.~\eqref{simple_bilevel_alg}. 
		Since for any $u_1, u_2$, 
		\begin{equation*}
			\begin{aligned}
				&\langle \H^{-1}\nabla \ell(u_1) - \H^{-1}\nabla \ell(u_2), u_1 - u_2\rangle_\H
				= \langle \nabla \ell(u_1) - \nabla \ell(u_2), u_1 - u_2\rangle \\
				&\ge \frac{1}{L_\ell}\|\nabla \ell(u_1) - \nabla \ell(u_2)\|^2 
				\ge \frac{\lambda_{\min}(\H)}{L_\ell}\| \H^{-1}\nabla \ell(u_1) -  \H^{-1}\nabla \ell(u_2)\|_\H^2, 
			\end{aligned}
		\end{equation*}
		where the first inequality follows from \cite{Heinz-MonotoneOperator-2011}[Corollary 18.16]. This implies that $\H^{-1}\nabla \ell$ is $\frac{\lambda_{\min}(\H)}{L_\ell}$-cocoercive (see \cite{Heinz-MonotoneOperator-2011}[Definition~4.4]) with respect to $\langle \cdot, \cdot \rangle_\H$ and $\|\cdot\|_\H$. Then, according to \cite{Heinz-MonotoneOperator-2011}[Proposition 4.33], we know that when $0 < s_k \le \frac{\lambda_{\min}(\H)}{L_{\ell}}$, operator $\I - s_k \H^{-1}\nabla \ell$ 
		is non-expansive with respect to $\|\cdot\|_\H$. Then, since $\v^{k+1}_{u} = \u^k - \alpha_k s H^{-1}\nabla \ell(\u^k)$, we have
		\[
		\|\v^{k+1}_{u} - (\bar{\u} - s_k\H^{-1}\nabla \ell(\bar{\u}))\|_\H \le  \|\u^k - \bar{\u}\|_\H,
		\]
		and $s_k \in (0,\frac{\lambda_{\min}(\H)}{L_{\ell}})$, we have that $\{\v^{k+1}_{u}\}$ is bounded.
	\end{proof}

	\begin{lemma}\label{lem2}\cite{BDA}[Lemma 2]
		Let $\{a_k\}$ and $\{b_k\}$ be sequences of non-negative real numbers. Assume that there exists $n_0 \in \mathbb{N}$ such that 
		\[
		a_{k+1} + b_k - a_k\le 0, \quad \forall k \ge n_0.
		\]
		Then $\lim_{k \rightarrow \infty} a_k$ exists and $\sum_{k=1}^{\infty}b_k < \infty$.
	\end{lemma}
	

	\begin{proposition}\label{prop:simple_bilevel_convergence}
		For any given $\ome \in \Omega$, let $\{\u^k\}$ be the sequence generated by Eq.~\eqref{simple_bilevel_alg} with $s_k = \frac{s}{k+1}$, $s \in (0, \frac{\lambda_{\min}(\H_{lb})}{L_{\ell}} )$ and $\mu\in (0,1)$. Suppose that $U$ is compact and ${\hat{\S}(\ome)}$ is nonempty , we have
		\begin{equation*}
			\begin{array}{c}
				\lim\limits_{k \rightarrow \infty}\mathrm{dist}(\u^k, {\hat{\S}(\ome)}) = 0,
			\end{array}
		\end{equation*}
		and then
		\begin{equation*}
			\begin{array}{c}
				\lim\limits_{k \rightarrow \infty}\ell(\u^k,\ome) =  \varphi(\ome).
			\end{array}
		\end{equation*}
	\end{proposition}
	\begin{proof}
		Let $\delta > 0$ be a constant satisfying $\delta < \frac{1}{2}\min\{(1-\mu)\eta, \mu(1 - sL_{\ell}/\lambda_{\min}(\H))\}$. We define a sequence of $\{\tau_n\}$ by
		\begin{equation*}
			\begin{aligned}
				\tau_n := \max\Big\{k \in \mathbb{N}\  |\ k\le n\ \text{and}\ 
				& \delta\|\u^{k-1} - \v^{k}_l\|_\H^2 + \delta\|\u^{k-1} - \v^{k}_{u}\|_\H^2 \\ 
				& + \frac{1}{4} \left\| \left((1-\mu) \v^{k}_l + \mu \v^{k}_{u} \right) - \u^k \right\|_\H^2 
				+ \mu s \beta_{k-1}\left(\ell(\v^{k}_{u}) - \ell^*\right) < 0 \Big\},
			\end{aligned}
		\end{equation*}
		where $\beta_{k} = \frac{1}{k+1}$ and $\eta = \frac{1-\alpha}{\alpha} > 0$.
		Inspired by \cite{cabot2005proximal}, we consider the following two cases: 
		
		(a) $\{\tau_n\}$ is finite, i.e., there exists $k_0 \in \mathbb{N}$ such that for all $k \ge k_0$,
		\begin{equation*}\small
			\begin{aligned}
				\delta\|\u^{k-1} - \v^{k}_l\|_\H^2 + \frac{1}{4} \left\| \left((1-\mu) \v^{k}_l + \mu \v^{k}_{u} \right) - \u^k \right\|_\H^2 &\\
				+ \delta\|\u^{k-1} - \v^{k}_{u}\|_\H^2
				+ \mu s \beta_{k-1}\left(\ell(\v^{k}_{u}) - \ell^*\right) & \ge 0,
			\end{aligned}
		\end{equation*}
		(b) $\{\tau_n\}$ is not finite, i.e., for all $ k_0 \in \mathbb{N}$, there exists $k \ge k_0$ such that 
		\begin{equation*}\small
			\begin{aligned}
				\delta\|\u^{k-1} - \v^{k}_l\|_\H^2 + \frac{1}{4} \left\| \left((1-\mu) \v^{k}_l + \mu \v^{k}_{u} \right) - \u^k \right\|_\H^2 \\
				+ \delta\|\u^{k-1} - \v^{k}_{u}\|_\H^2
				+ \mu s \beta_{k-1}\left(\ell(\v^{k}_{u}) - \ell^*\right) < 0.
			\end{aligned}
		\end{equation*}
		
		\noindent\textbf{Case (a):} We assume that $\{\tau_n\}$ is finite and there exists $k_0 \in \mathbb{N}$ such that 
		\begin{equation}\label{thm1_eq0}\small
			\begin{array}{l}
				\delta\|\u^{k-1} - \v^{k}_l\|_\H^2 + \frac{1}{4} \left\| \left((1-\mu) \v^{k}_l + \mu \v^{k}_{u} \right) - \u^k \right\|_\H^2 + \\ \delta\|\u^{k-1} - \v^{k}_{u}\|_\H^2
				+ \mu s \beta_{k-1}\left(\ell(\v^{k}_{u}) - \ell^*\right) \ge 0,
			\end{array}
		\end{equation}
		for all $k \ge k_0$. Let $\bar{\u}$ be any point in $\hat{\S}$, setting $\u$ in Eq.~\eqref{eq:alg_p_lem1_eq} of Lemma~\ref{simple_bilevel_lem} to be $\bar{\u}$, as  $\mu \in (0,1)$ and $\beta_k \le 1$, we have
		\begin{equation}\label{eq:thm1_eq1}\small
			\begin{aligned}
				&\frac{1}{2} \| \bar{\u} - \u^k\|_\H^2 \\
				\ge \ &\frac{1}{2}\| \bar{\u} - \u^{k+1}\|_\H^2 +   \left( \frac{(1-\mu)\eta}{2} - \delta  \right)  \|\u^k - \v^{k+1}_l\|_\H^2 \\
				& + \left( \frac{\mu(1 - sL_{\ell}/\lambda_{\min}(\H))}{2} - \delta \right) \|\u^k - \v^{k+1}_{u}\|_\H^2   \\
				&+ \frac{1}{4}\left\| \left((1-\mu) \v^{k+1}_l + \mu \v^{k+1}_{u} \right) - \u^{k+1} \right\|_\H^2 
				+ \delta\|\u^k - \v^{k+1}_l\|_\H^2 \\ 
				&+ \frac{1}{4} \left\|\left((1-\mu) \v^{k+1}_l + \mu \v^{k+1}_{u} \right) - \u^{k+1} \right\|_\H^2 + \delta\|\u^k - \v^{k+1}_{u}\|_\H^2\\
				&+ \mu s \beta_{k}\left(\ell(\v^{k+1}_{u}) - \ell^*\right).
			\end{aligned}
		\end{equation}
		For all $k \ge k_0$, combining $0 < \delta < \frac{1}{2}\min\{(1-\mu)\eta, \mu(1 - sL_{\ell}/\lambda_{\min}(\H))\}$ and Eq.~\eqref{eq:thm1_eq1}, it follows from Lemma~\ref{lem2} that
		\begin{equation*}
			\begin{array}{l}
				\sum\limits_{k=0}^{\infty}\|\u^k - \v^{k+1}_l\|_\H^2 < \infty,\\ 
				\sum\limits_{k=0}^{\infty}\|\u^k - \v^{k+1}_{u}\|_\H^2 < \infty,\\
				\sum\limits_{k=0}^{\infty}\left\|\left((1-\mu) \v^{k+1}_l + \mu \v^{k+1}_{u} \right) - \u^{k+1} \right\|_\H^2  < \infty,\\
				\sum\limits_{k=0}^{\infty} \beta_{k}\left(\ell(\v^{k+1}_{u}) - \ell^*\right) < \infty,\\ 
			\end{array}
		\end{equation*}
		and $\lim_{k \rightarrow \infty}\| \bar{\u} - \u^k\|_\H^2 $ exists. 
		
		Now, we show the existence of subsequence $\{\u^{j}\} \subseteq \{\u^k\}$ such that $\lim_{\ell \rightarrow \infty} \ell(\u^{j} ) \le \ell^* $. This obviously holds if for any $\hat{k} > 0$, there exists $k > \hat{k}$ such that $\ell(\u^k ) \le \ell^*$. Therefore, we only need to consider the case where there exists $\hat{k} > 0$ such that $ \ell(\u^k ) > \ell^*$ for all $k \ge \hat{k}$. If there does not exist subsequence $\{\u^{j}\} \subseteq \{\u^k\}$ such that $\lim_{j \rightarrow \infty} \ell(\u^{j} ) \le \ell^* $, there must exist $\epsilon > 0$ and $k_1 \ge \max\{\hat{k}, k_0\}$ such that $\ell(\u^k) - \ell^* \ge 2\epsilon$ for all $k \ge k_1$. As $U$ is compact, it follows from Lemma~\ref{lem_bounded} that sequences $\{\u^k\}$ and $\{\v^{k+1}_{u}\}$ are both bounded. Then since $\ell$ is continuous and $\lim_{k \rightarrow \infty}\|\u^k - \v^{k+1}_{u}\|_\H = 0$ with $\H \succ 0$, there exists $k_2 \ge k_1$ such that $|\ell(\u^k) - \ell(\v^{k+1}_{u})| < \epsilon$ for all $k \ge k_2$ and thus $\ell(\v^{k+1}_{u}) - \ell^* \ge \epsilon$ for all $k \ge k_2$.
		Then we have
		\begin{equation*}
			\epsilon \sum_{k = k_2}^{\infty} \beta_k \le \sum_{k = k_2}^{\infty}  \beta_k \left(\ell(\v^{k+1}_{u}) - \ell^*\right) < \infty,
		\end{equation*}
		where the last inequality follows from $ \sum_{k=0}^{\infty} \beta_k \left(\ell(\v^{k+1}_{u}) - \ell^*\right) < \infty$.
		This result contradicts to the definition of $\beta_{k}$ and the fact that $\sum_{k=0}^{\infty} \beta_k = \sum_{k=0}^{\infty} \frac{1}{k+1} = +\infty$. 
		As $\{\u^{j}\}$ and $\{\v^{k+1}_{l}\}$ are bounded, and $\lim_{j \rightarrow \infty} \|\u^{j} - \v^{j+1}_l\|_\H = 0$ with $\H \succ 0$, we can assume without loss of generality that $\lim_{j \rightarrow \infty} \v^{j+1}_l = \u^{j} = \tilde{\u} $ by taking a subsequence. By the continuity of $\ell$, we have $\ell(\tilde{\u}) = \lim_{j \rightarrow \infty} \ell( \u^{j} ) \le \ell^*$.
		Next, since $\v^{j+1}_l = \T(\u^{j})$, let $\ell \rightarrow \infty$, by the closedness of $\T$, we have
		\[
		\tilde{\u} = \T(\tilde{\u}),
		\]
		and thus $\tilde{\u} \in \S$. Combining with $\ell(\tilde{\u}) \le \ell^*$, we show that $\tilde{\u} \in \hat{\S}$. Then by taking $\bar{\u} = \tilde{\u}$ and since $\lim_{k \rightarrow \infty}\|\bar{\u}- \u^k\|_\H^2 $ exists, we have $\lim_{k \rightarrow \infty}\|\bar{\u}- \u^k\|_\H^2 = 0$ with $\H \succ 0$ and thus $\lim_{k \rightarrow \infty}\mathrm{dist}(\u^k, \hat{\S}) = 0$.
		
		\noindent\textbf{Case (b):} We assume that $\{\tau_n\}$ is not finite and for any $ k_0 \in \mathbb{N}$, there exists $k \ge k_0$ such that $ \delta\|\u^{k-1} - \v^{k}_l\|_\H^2 + \frac{1}{4} \left\| \left((1-\mu) \v^{k}_l + \mu \v^{k}_{u} \right) - \u^k \right\|_\H^2 + \delta\|\u^{k-1} - \v^{k}_{u}\|_\H^2 + \mu s \beta_{k-1}\left(\ell(\v^{k}_{u}) - \ell^*\right) < 0$ .
		It follows from the assumption that $\tau_n$ is well defined for $n$ large enough and $\lim_{n \rightarrow \infty} \tau_n = + \infty$. We assume without loss of generality that $\tau_n$ is well defined for all $n$.
		
		By setting $\u$ in Eq.~\eqref{eq:alg_p_lem1_eq} of Lemma~\ref{simple_bilevel_lem} to be $\mathtt{Proj}_{\hat{\S}}(\u^k)$, we have
		\begin{equation}\label{thm1_eq4}\small
			\begin{aligned}
				&\frac{1}{2} \mathrm{dist}_\H^2 (\u^k,\hat{\S}) \\
				\ge \ &\frac{1}{2} \mathrm{dist}_\H^2 ( \u^{k+1},\hat{\S}) \\
				&+  \left( \frac{(1-\mu)\eta}{2} - \delta  \right)  \|\u^k - \v^{k+1}_l\|_\H^2 \\
				&+ \left( \frac{\mu(1 - sL_{\ell}/\lambda_{\min}(\H))}{2} - \delta \right) \|\u^k - \v^{k+1}_{u}\|_\H^2   \\
				&+ \frac{1}{4}\left\| \left((1-\mu) \v^{k+1}_l + \mu \v^{k+1}_{u} \right) - \u^{k+1} \right\|_\H^2 + \delta\|\u^k - \v^{k+1}_l\|_\H^2 \\ 
				& + \frac{1}{4} \left\|\left((1-\mu) \v^{k+1}_l + \mu \v^{k+1}_{u} \right) - \u^{k+1} \right\|_\H^2+ \delta\|\u^k - \v^{k+1}_{u}\|_\H^2 \\
				&+ \mu s \beta_{k}\left(\ell(\v^{k+1}_{u}) - \ell^*\right),
			\end{aligned}
		\end{equation}
		where $\mathrm{dist}_\H^2 (\u,\hat{\S}) := \inf_{\u' \in \hat{\S}} \| \u - \u'\|_\H$.
		Suppose $\tau_n \le n-1$, and by the definition of $\tau_n$, we have 
		\begin{equation*} \small
			\begin{aligned}
				&\delta\|\u^k - \v^{k+1}_l\|_\H^2 \\
				&+ \delta\|\u^k - \v^{k+1}_{u}\|_\H^2 
				+ \frac{1}{4} \left\|\left((1-\mu) \v^{k+1}_l + \mu \v^{k+1}_{u} \right) - \u^{k+1} \right\|_\H^2  \\
				&+ \mu s \beta_{k}\left(\ell(\v^{k+1}_{u}) - \ell^*\right)  \ge 0,
			\end{aligned}
		\end{equation*}
		for all $\tau_n \le k \le n-1 $. Then 
		\begin{equation}\label{thm1_eq5}
			h_{k+1}-h_{k} \le 0, \quad \tau_n \le k \le n-1,
		\end{equation}
		where $h_k := \frac{1}{2} \mathrm{dist}_\H^2 (\u^k,\hat{\S}) $. 
		Adding these $n-\tau_n$ inequalities, we have
		\begin{equation}\label{thm1_eq6}
			h_{n} \le h_{\tau_n}.
		\end{equation}
		Eq.~\eqref{thm1_eq6} is also true when $\tau_n = n$ because $h_{\tau_n} = h_n$. Then, once we are able to show that $\lim_{n \rightarrow \infty}h_{\tau_n} = 0$, we can obtain from Eq.~\eqref{thm1_eq6} that $\lim_{n \rightarrow \infty}h_{n} = 0$. 
		
		By the definition of $\{\tau_n\}$, $\ell^* > \ell(\v^{k}_{u})$ for all $k \in \{\tau_n\}$. Since $U$ is compact, according to Lemma~\ref{lem_bounded}, both $\{\u^{\tau_n}\}$ and $\{\v_{u}^{\tau_n}\}$ are bounded, and hence $\{h_{\tau_n}\}$ is bounded. As $\ell$ is assumed to be bounded below by $M_0$, we have
		\[
		0 \le \ell^* - \ell(\v^{k}_{u}) \le \ell^* - M_0.
		\]
		According to the definition of $\tau_n$, we have for all $k \in \{\tau_n\}$,
		\begin{equation*}\small
			\begin{aligned}
				&\delta\|\u^{k-1} - \v^{k}_l\|_\H^2 +  \delta\|\u^{k-1} - \v^{k}_{u}\|_\H^2 +\frac{1}{4} \left\| \left((1-\mu) \v^{k}_l + \mu \v^{k}_{u} \right) - \u^k \right\|_\H^2 \\
				&< \mu s \beta_{k-1}\left(\ell^*  - \ell(\v^{k}_{u}) \right)\\
				&\le \mu s \beta_{k-1} \left(\ell^* - M_0\right).
			\end{aligned}
		\end{equation*}
		As  $\lim_{n \rightarrow \infty} \tau_n = + \infty$, $\beta_k = \frac{1}{k+1} \rightarrow 0$, we have
		\begin{equation*}
			\begin{array}{c}
				\lim_{n \rightarrow \infty}\|\u^{\tau_n-1} -\v_{l}^{\tau_n}\|_\H = 0, \\
				\lim_{n \rightarrow \infty}\|\u^{\tau_n-1} - \v_{u}^{\tau_n}\|_\H = 0, \\
				\lim_{n \rightarrow \infty}\| \left((1-\mu)\v^{\tau_n}_{l} + \mu \v^{\tau_n}_{u}\right) - \u^{\tau_n} \|_\H = 0.
			\end{array}
		\end{equation*}
		Let $\tilde{\u}$ be any limit point of $\{\u^{\tau_n}\}$, and $\{\u^{j}\}$ be the subsequence of $\{\u^{\tau_n}\}$ such that 
		\begin{equation*}
			\lim_{j \rightarrow \infty}\u^{j} = \tilde{\u}.
		\end{equation*}
		As $\lim_{n \rightarrow \infty}\|\u^{\tau_n-1} - \u^{\tau_n}\|_\H \le \lim_{n \rightarrow \infty}(\|\u^{\tau_n-1} - \left((1-\mu)\v^{\tau_n}_{l} + \mu \v^{\tau_n}_{u}\right)\|_\H + \| \left((1-\mu)\v^{\tau_n}_{l} + \mu \v^{\tau_n}_{u}\right) - \u^{\tau_n} \|_\H ) = 0$ and $\H \succ 0$, we have $\lim_{j \rightarrow \infty}\u^{j-1} = \tilde{\u} $.
		Next, since $\lim_{\ell \rightarrow \infty}\|\u^{j-1} - \v_{l}^{j}\|_\H = 0$ and $\H \succ 0$, it holds that $\lim_{j \rightarrow \infty}\v_{l}^{j} = \tilde{\u}$. Then, it follows from $\v^{j}_l = \T(\u^{j-1})$, and the closedness of $\T$, we have
		\[
		\tilde{\u} = \T(\tilde{\u}),
		\]
		and thus $\tilde{\u} \in \S$. As $\ell^* > \ell(\v^{k}_{u})$ for all $k \in \{\tau_n\}$ and hence $\ell^* > \ell(\v_{u}^j)$ for all $j$. Then it follows from the continuity of $\ell$ and $\lim_{n \rightarrow \infty}\|\v_u^{\tau_n} - \u^{\tau_n}\|_\H = 0$, $\H \succ 0$ that $\ell^* \ge \ell(\tilde{\u})$, which implies $\tilde{\u} \in \hat{\S}$ and $\lim_{j \rightarrow 0} h_{j} = 0$. Now, as we have shown above that $\tilde{\u} \in \hat{\S}$ for any limit point $\tilde{\u}$ of $\{\u^{\tau_n}\}$, we can obtain from the boundness of $\{\u^{\tau_n}\}$ and $\{h_{\tau_n}\}$ that $\lim_{n \rightarrow \infty} h_{\tau_n} = 0$. Thus $\lim_{n \rightarrow \infty}h_{n} = 0$, and $\lim_{k \rightarrow \infty}\mathrm{dist}(\u^k, \hat{\S}) = 0$.
	\end{proof}
	
	We denote $D = \sup\limits_{\u,\u'\in U}\|\u-\u'\|_{\H_{ub}}$ and $M_{\ell} := \sup\limits_{\u \in U, \ome \in \Omega} \| \frac{\partial}{\partial \u} \ell(\u, \ome) \| $. And it should be noticed that both $D$ and $M_{\ell}$ are finite when $U$ and $\Omega$ are compact.
	\begin{lemma}\label{lem_monontone}
		For any given $\ome \in \Omega$, let $\{\u^k\}$ be the sequence generated by Eq.~\eqref{simple_bilevel_alg} with $s_k = \frac{s}{k+1}$, $s \in (0, \frac{\lambda_{\min}(\H_{lb})}{L_{\ell}} )$ and $\mu\in (0,1)$, then we have
		\begin{equation*}
			\begin{aligned}
				\|\u^{k+1} - \u^k \|_\H^2 \le \|\u^k - \u^{k-1}\|_\H^2 + \frac{\mu}{(k+1)^2}\|\u^{k-1} - \v^{k}_{u} \|_\H^2 
				+ \frac{2\mu sDM_{\ell}}{\lambda_{\min}(\H)k(k+1)}.
			\end{aligned}
		\end{equation*}
	\end{lemma}
	\begin{proof}
		Since 
		\begin{equation*}
			\begin{array}{l}
				\u^{k+1} = \mathtt{Proj}_{U,\H}\left(  \mu \v^{k+1}_{u} + (1-\mu) \v^{k+1}_l \right),
			\end{array}
		\end{equation*}
		by denoting $\Delta^k_{\beta}:=\beta_{k} - \beta_{k-1}$, we have the following inequality
		\begin{equation*}
			\begin{aligned}
				&\|\u^{k+1} - \u^k \|_\H^2 \\
				\le\ &\mu\|\v^{k+1}_{u} - \v^{k}_{u}\|_\H^2 + (1-\mu)\|\v^{k+1}_{l} - \v^{k}_{l}\|_\H^2, \\
				\le\ &\mu \Big( \|(\I - s_k \H^{-1}\nabla \ell)(\u^k - \u^{k-1})\|_\H^2 
				+ \frac{ s^2}{k^2(k+1)^2}\|\H^{-1}\nabla \ell(\u^{k-1}) \|_\H^2 \\
				&+ \frac{2 s}{k(k+1)} \|(\I - s_k \H^{-1}\nabla \ell)(\u^k - \u^{k-1})\|_\H \|\H^{-1}\nabla \ell(\u^{k-1}) \|_\H \Big) \\
				&+ (1-\mu)\|\v^{k+1}_{l} - \v^{k}_{l}\|_\H^2 \\
				\le &\|\u^k - \u^{k-1}\|_\H^2 
				\ + \ \frac{2\mu s}{k(k+1)}\|\u^k - \u^{k-1}\|_\H\|\H^{-1}\nabla \ell(\u^{k-1}) \|_\H \\
				&+ \frac{\mu }{(k+1)^2}\|\u^{k-1} - \v^{k}_{u} \|_\H^2,
			\end{aligned}
		\end{equation*}
		where the first inequality follows from the non-expansiveness of $\mathtt{Proj}_{U,\H}$ with respect to $\|\cdot\|_\H$ and the convexity of $\|\cdot\|_\H^2$, the second inequality comes from the definition of $\v^{k}_{u}$ and the last inequality follows from the definitions of $\v^{k}_{u}$, $\v^{k}_{l}$ and the non-expansiveness of $\I - s_k \H^{-1}\nabla \ell$ and $\T$ with respect to $\|\cdot\|_\H$ from Lemma~\ref{lem_bounded} and~\ref{lemma2.1}.
		Then, since $\sup_{\u,\u'\in U}\|\u-\u'\|_\H \le D$, $\sup_{\u\in U} \|\nabla \ell(\u) \| \le M_{\ell}$, we have the following result
		\begin{equation*}
			\begin{array}{r}
				\|\u^{k+1} - \u^k \|_\H^2 \le\;  \|\u^k - \u^{k-1}\|_\H^2 + \frac{\mu}{(k+1)^2}\|\u^{k-1} - \v^{k}_{u} \|_\H^2 
				+ \frac{2\mu sDM_{\ell}}{\lambda_{\min}(\H)k(k+1)}.
			\end{array}
		\end{equation*}
	\end{proof}

	\begin{proposition}\label{simple_bilevel_complexity}
		For any given $\ome \in \Omega$, let $\{\u^k\}$ be the sequence generated by Eq.~\eqref{simple_bilevel_alg} with $s_k = \frac{s}{k+1}$, $s \in (0, \frac{\lambda_{\min}(\H_{lb})}{L_{\ell}} )$ and $\mu\in (0,1)$. Suppose ${\hat{\S}(\ome)}$ is nonempty, $U$ is compact, $\ell(\cdot,\ome)$ is bounded below by $M_0$, we have
		for $k \ge 2$,
		\begin{equation*}
			\begin{array}{l}
				\| \v^{k+1}_{l}  - \u^k \|_\H^2 \le (2C_2 + C_3)\frac{1+\ln (1+k)}{k^{\frac{1}{4}}},
			\end{array}
		\end{equation*}
		where $C_1:= \left( 3( D^2 + 2\mu s ( \ell^* - M_0 ) ) + 2\mu sDM_{\ell}/\lambda_{\min}(\H_{lb}) \right)/\min  \left\{ (1 - sL_{\ell}/\lambda_{\min}(\H_{lb})), \frac{1-\alpha}{\alpha}, 1 \right\}$, $C_2 := 12D\sqrt{C_1}$ and $C_3:= \frac{\alpha(D^2 + 2\mu s \left( \ell^* -  M_0 \right))}{(1-\alpha)(1-\mu)}$.
	\end{proposition}
	\begin{proof}
		Let $\bar{\u}$ be any point in $\hat{\S}$, and set $\u$ in Eq.~\eqref{eq:alg_p_lem1_eq} of Lemma~\ref{simple_bilevel_lem}  to be $\bar{\u}$. Then we have
		\begin{equation}\label{thm2_eq1}
			\begin{aligned}
				& \frac{1}{2} \|\bar{\u} - \u^k\|_\H^2  + \frac{\mu s}{k+1} (\ell^* - \ell(\v^{k+1}_{u})) \\ \ge\,
				& \frac{1}{2}\|\bar{\u} - \u^{k+1}\|_\H^2 + \frac{(1-\mu)(1-\alpha)}{2\alpha} \|\u^k - \v^{k+1}_l\|_\H^2 \\
				&+ \frac{\mu}{2}(1 - s_kL_{\ell}/\lambda_{\min}(\H)) \|\u^k - \v^{k+1}_{u}\|_\H^2 \\
				& + \frac{1}{2}\left\| \left((1-\mu) \v^{k+1}_l + \mu \v^{k+1}_{u} \right) - \u^{k+1} \right\|_\H^2.
			\end{aligned}
		\end{equation}
		Adding the Eq.~\eqref{thm2_eq1} from $k = 0$ to $k = n-1$ with $n \ge 1$ and since $s_k \le s$, we have
		\begin{equation}\label{thm2_eq2}
			\begin{aligned}
				&\frac{1}{2} \|\bar{\u} - \u^n\|_\H^2  + \frac{(1-\mu)(1-\alpha)}{2\alpha}\sum\limits_{k=0}^{n-1}\|\u^k - \v^{k+1}_l\|_\H^2 
				\\
				&+ \frac{\mu}{2}(1 - sL_{\ell}/\lambda_{\min}(\H)) \sum\limits_{k=0}^{n-1} \|\u^k - \v^{k+1}_{u}\|_\H^2\\
				&+ \frac{1}{2}\sum\limits_{k=0}^{n-1}\left\| \left((1-\mu) \v^{k+1}_l + \mu \v^{k+1}_{u} \right) - \u^{k+1} \right\|_\H^2 \\
				\le \ &\frac{1}{2} \|\bar{\u} - \u^{0}\|_\H^2  + \sum\limits_{k=0}^{n-1}\frac{\mu s}{k+1}\left( \ell^* -  \ell(\v^{k+1}_{u}) \right) \\
				\le \ &\frac{1}{2} \|\bar{\u} - \u^{0}\|_\H^2 + \mu s (1+\ln n) \left( \ell^* -  M_0 \right),
			\end{aligned}
		\end{equation}
		where the last inequality follows from the assumption that $\inf \ell \ge M_0$. 
		By Lemma~\ref{lem_monontone}, we have
		\begin{equation*}\label{thm2_eq4}
			\begin{aligned}
				\|\u^{k+1} - \u^k \|_\H^2 \le\; & \|\u^k - \u^{k-1}\|_\H^2 + \frac{\mu}{(k+1)^2}\|\u^{k-1} - \v^{k}_{u} \|_\H^2 \\
				&+ \frac{2\mu sDM_{\ell}}{\lambda_{\min}(\H)k(k+1)},
			\end{aligned}
		\end{equation*}
		and thus
		\begin{equation}\label{thm2_eq4.5}
			\begin{array}{l}
				n\|\u^{n} - \u^{n-1} \|_\H^2 \le\sum\limits_{k=0}^{n-1}\|\u^{k+1} - \u^k\|_\H^2\\
				+ \mu\sum\limits_{k=0}^{n-2}\|\u^k - \v^{k+1}_{u} \|_\H^2 + 2\mu sDM_{\ell}/\lambda_{\min}(\H).
			\end{array}
		\end{equation}
		Then it follows from Eq.~\eqref{thm2_eq2} and Eq.~\eqref{thm2_eq4.5} that
		\begin{equation*}
			\begin{aligned}
				&\min \left\{ (1 - sL_{\ell}/\lambda_{\min}(\H)), \frac{1-\alpha}{\alpha}, 1 \right\}\, n \|\u^{n} - \u^{n-1} \|_\H^2  \\
				\le \ &\min \left\{ (1 - sL_{\ell}/\lambda_{\min}(\H)), \frac{1-\alpha}{\alpha}, 1 \right\} \sum\limits_{k=0}^{n-1}\|\u^{k+1} - \u^k\|_\H^2 
				\\ &+ \mu(1 - sL_{\ell}/\lambda_{\min}(\H))\sum\limits_{k=0}^{n-1} \|\u^k - \v^{k+1}_{u}\|_\H^2\\
				&+ 2\mu sDM_{\ell}/\lambda_{\min}(\H) \\
				\le \ &\frac{2(1-\mu)(1-\alpha)}{\alpha}  \sum\limits_{k=0}^{n-1}\| \u^k - \v^{k+1}_{l} \|_\H^2 \\
				&+ 3\mu(1 -  sL_{\ell}/\lambda_{\min}(\H))\sum\limits_{k=0}^{n-1}\| \u^k - \v^{k+1}_{u} \|_\H^2 \\
				&+ 2\sum\limits_{k=0}^{n-1}\left\| (1-\mu)\v^{k+1}_{l} + \mu \v^{k+1}_{u} - \u^{k+1} \right\|_\H^2 \\
				& + 2\mu sDM_{\ell}/\lambda_{\min}(\H) \\
				\le &3 \left( \|\bar{\u} - \u^{0}\|_\H^2  + 2\mu s (1+\ln n)\left( \ell^* - M_0 \right) \right) \\
				&+ 2\mu sDM_{\ell}/\lambda_{\min}(\H),
			\end{aligned}
		\end{equation*}
		where the second inequality comes from $\u^k - \u^{k+1} = (1-\mu)(\u^k - \v^{k+1}_{l}) + \mu(\u^k - \v^{k+1}_{u}) + (1-\mu)\v^{k+1}_{l} + \mu \v^{k+1}_{u} - \u^{k+1}$ and the convexity of $\|\cdot\|_\H^2$. Combining with the fact that $\|\bar{\u} - \u^{0}\|_\H \le D$, we have 
		\begin{equation}\label{thm2_eq5}
			\|\u^{n} - \u^{n-1} \|_\H^2 \le \frac{C_1(1+\ln n)}{n},
		\end{equation}
		where $C_1:= ( 3( D^2 + 2\mu s ( \ell^* - M_0 ) ) + 2\mu sDM_{\ell}/\lambda_{\min}(\H_{lb}) )/\min  \left\{ (1 - sL_{\ell}/\lambda_{\min}(\H_{lb})), \frac{1-\alpha}{\alpha}, 1 \right\}$.
		Next, by Lemma~\ref{lem_bounded}, we have for all $k$,
		\begin{equation*}
			\begin{aligned}
				\|\v^{k+1}_{l}  - \u^k\|_\H \le \| \v^{k+1}_{l} - \bar{\u}\|_\H + \| \u^k- \bar{\u}\|_\H \\
				\le 2\|  \u^k- \bar{\u} \|_\H \le 2D.
			\end{aligned}		
		\end{equation*}
		Then, we have
		\begin{equation}\label{thm2_eq3}\small
			\begin{aligned}
				&\| \v^{k+1}_{l}  - \u^k \|_\H^2 \\
				\le \ &2\| \v^{k}_{l}  - \u^{k-1} \|_\H \| \v^{k+1}_{l}  - \u^k - \v^{k}_{l} + \u^{k-1} \|_\H \\
				&+ \| \v^{k}_{l}  - \u^{k-1} \|_\H^2 + \| \v^{k+1}_{l}  - \u^k - \v^{k}_{l} + \u^{k-1} \|_\H^2  \\
				\le \ & 4D\left( \| \v^{k+1}_{l}  - \v^{k}_{l} \|_\H + \| \u^k - \u^{k-1} \|_\H \right) \\
				&+ \| \v^{k}_{l}  - \u^{k-1} \|_\H^2 + 2\| \v^{k+1}_{l}  - \v^{k}_{l} \|_\H^2 + 2\| \u^k - \u^{k-1} \|_\H^2 \\
				\le \ &  \| \v^{k}_{l}  - \u^{k-1} \|_\H^2 + 12D\| \u^k - \u^{k-1} \|_\H,
			\end{aligned}
		\end{equation}
		where the last inequality comes from $\| \u^k - \u^{k-1}\|_\H \le D$, $\v^{k}_{l} \in \T(\u^{k-1})$ and the non-expansiveness of $\T$ with respect to $\| \cdot \|_\H$. This implies that for any $n > n_0 \ge 0$,
		\begin{equation*}
			\begin{array}{r}
				\| \v^{n+1}_{l}  - \u^n \|_\H^2 
				\le 12D \sum\limits_{k= n_0+1}^{n}\| \u^k - \u^{k-1} \|_\H \\
				+ \|\v_l^{n_0+1} - \u^{n_0}\|_\H^2.
			\end{array}
		\end{equation*}
		Thus, for any $m \ge 2$ and $n_0 = n-m+1$, the following holds
		\begin{equation}\label{thm2_eq3.5}
			\begin{array}{l}
				m \| \v^{n+1}_{l}  - \u^n \|_\H^2 \\
				\le 12D \sum\limits_{k= n_0+1}^{n}(k-n_0)\| \u^k - \u^{k-1} \|_\H	+ \sum\limits_{k= n_0}^{n}\| \v^{k+1}_{l} - \u^k \|_\H^2 \\
				\le \sum\limits_{k=n_0}^{n}\| \v^{k+1}_{l} - \u^k \|_\H^2 + 12D\sqrt{C_1}\frac{m(m-1)}{2} \frac{\sqrt{(1+\ln n_0)}}{\sqrt{n_0}},
			\end{array}
		\end{equation}
		where the last inequality follows from Eq.~\eqref{thm2_eq5} that $\| \u^k - \u^{k-1} \|_\H^2 \le \frac{C_1(1+\ln n_0)}{n_0}$ for all $k \ge n_0$, and it can be easily verified that the above inequality also holds when $m = 1$.
		According to Eq.~\eqref{thm2_eq2}, we have
		\begin{equation*}
			\begin{array}{l}
				\frac{(1-\mu)(1-\alpha)}{2\alpha}\sum\limits_{k=0}^{n-1}\|\u^k - \v^{k+1}_l\|_\H^2\\
				\le \frac{1}{2} \|\bar{\u} - \u^{0}\|_\H^2 + \mu s (1+\ln n) \left( \ell^* -  M_0 \right).
			\end{array}
		\end{equation*}
		Then, for any $n > 0$, let $m$ be the smallest integer such that $m \ge n^{\frac{1}{4}}$ and let $n_0 = n-m+1$, combining the above inequality with Eq.~\eqref{thm2_eq3.5}, we have
		\begin{equation*}
			\begin{array}{l}
				\frac{\|\bar{\u} - \u^{0}\|_\H^2 + 2\mu s \left(1+\ln (1+n) \right) \left( \ell^* -  M_0 \right)}{(1-\mu)(1-\alpha)/\alpha} \\
				\ge \sum\limits_{k= n_0}^{n}\| \u^k - \v^{k+1}_l\ \|_\H^2 \\
				\ge  m \| \v^{n+1}_{l}  - \u^n \|_\H^2 - C_2\frac{m(m-1)}{2} \frac{\sqrt{(1+\ln n_0)}}{\sqrt{n_0}} ,
			\end{array}
		\end{equation*}
		where $C_2 := 12D\sqrt{C_1}$.
		Next, as $n^{\frac{1}{4}}+1 \ge m \ge n^{\frac{1}{4}}$, and hence $n_0 \ge (m-1)^4 - m +1$. Then $16n_0 - m^2(m-1)^2 \ge (m-1)[(m-1)(3m-4)(5m-4)-1] > 0$ when $m \ge 2$. Thus, when $n \ge 2$, we have $m \ge 2$, $m(m-1)\le 4\sqrt{n_0}$ and thus $\frac{m(m-1)}{2} \frac{\sqrt{(1+\ln n_0)}}{\sqrt{n_0}} \le 2\sqrt{(1+\ln n_0)}$. Then, let $C_3:= \frac{D^2 + 2\mu s \left( \ell^* -  M_0 \right)}{(1-\mu)(1-\alpha)/\alpha}$, we have for any $n \ge 2$,
		\begin{equation*}
			\begin{aligned}
				\| \v^{n+1}_{l}  - \u^n \|_\H^2 &\le \frac{1}{m}\left( C_3\left(1+\ln(1+n) \right)  + 2C_2\sqrt{(1+\ln n_0)}\right) \\
				&\le (2C_2 + C_3)\frac{1+\ln(1+n)}{n^{\frac{1}{4}}},
			\end{aligned}
		\end{equation*}
		where the last inequality follows from $\sqrt{1+\ln n_0} \le 1+\ln(1+n)$ and $m \ge n^{\frac{1}{4}}$. 	
	\end{proof}
	

	Based on the discussion above, 
	we can show that the sequence $\{\u^k(\ome)\}$ generated by Eq.~\eqref{simple_bilevel_alg} not only converges to the solution set of $\inf_{\u \in \mathtt{Fix}(\T(\cdot,\ome)) \cap U } \ \ell(\u,\ome)$ 
	but also admits a uniform convergence towards the fixed-point set $\mathtt{Fix}(\T(\cdot,\ome)$ 
	with respect to $\| \u^k(\ome) - \T(\u^k(\ome),\ome) \|_{\H_{lb}}^2$ for $\ome \in \Omega$.


	\begin{thm}\label{thm:simple_bilevel_convergence}
		Let $\{\u^k(\ome)\}$ be the sequence generated by Eq.~\eqref{simple_bilevel_alg} with $\mu \in (0,1)$ and $s_k = \frac{s}{k+1}$, $s \in (0, \frac{\lambda_{\min}(\H_{lb})}{L_{\ell}} )$, where $\lambda_{\min}(\H_{lb})$ denotes the smallest eigenvalue of matrix $\H_{lb}$. Then, we have for any $\ome \in \Omega$,
		\begin{equation*}
			\begin{array}{c}
				\lim\limits_{k \rightarrow \infty}\mathrm{dist}(\u^k(\ome),\mathtt{Fix}(\T(\cdot,\ome)) = 0,
			\end{array}
		\end{equation*}
		and
		\begin{equation*}
			\begin{array}{c}
				\lim\limits_{k \rightarrow \infty}\ell(\u^k(\ome),\ome) =  \varphi(\ome).
			\end{array}
		\end{equation*}
		Furthermore, there exits $C > 0$ such that for any $\ome \in \Omega$,
		\[
		\| \u^k(\ome) - \T(\u^k(\ome),\ome) \|_{\H_{lb}}^2 \le C\sqrt{\frac{1+\ln(1+k)}{k^{\frac{1}{4}}}}.
		\]
	\end{thm}

	\begin{proof}
		The property for any $\ome \in \Omega$,
		\begin{equation*}
			\begin{aligned}
				\lim\limits_{k \rightarrow \infty}\mathrm{dist}(\u^k(\ome),\mathtt{Fix}(\T(\cdot,\ome)) = 0, 
				\\ \text{and} \quad 			\lim\limits_{k \rightarrow \infty}\ell(\u^k(\ome),\ome) =  \varphi(\ome).
			\end{aligned}
		\end{equation*}
		follows from Proposition~\ref{prop:simple_bilevel_convergence} immediately.
		Since $U$ and $\Omega$ are both compact, and $\ell(\u,\ome)$ is continuous on $U \times \Omega$, we have that $\ell(\u,\ome)$ is uniformly bounded above on $U \times \Omega$ and thus $	\varphi(\ome) = \inf_{u \in \mathrm{Fix}(T(\cdot, \ome)) \cap U }\ell(\u,\ome)$ is bounded on $\Omega$. And combining with the assumption that $\ell(\u,\omega)$ is bounded below by $M_0$ on $U \times \Omega$, we can obtain from the Proposition~\ref{simple_bilevel_complexity} that there exists $C > 0$ such that for any $\ome \in \Omega$, we have
		\[
		\| \u^k(\ome) - \T(\u^k(\ome),\ome) \|_\H^2 \le C\sqrt{\frac{1+\ln(1+k)}{k^{\frac{1}{4}}}}.
		\] 
	\end{proof}

	Thanks to the uniform convergence property of the sequence $\{\u^k(\ome)\}$, 
	inspired by the arguments used in~\cite{BDA},
	we can establish the convergence on both $\u$ and $\ome$ of BMO ( Algorithm~\ref{alg:bmo}) towards the solution of optimization problem in Eq.~\eqref{bilevel_fix} in the following proposition and theorem.

	\begin{proposition}\label{thm:general}
		Suppose $U$ and $\Omega$ are compact,
		\begin{itemize}
			\item[(a)] $\{\u^{K}(\ome)\} \subset U$ for any $\ome \in \Omega$, and for any $\epsilon>0$, there exists $k(\epsilon)>0$ such that whenever $k>k(\epsilon)$, 	
			\begin{equation*}
				\sup_{\ome \in \Omega}\left\| \u^k(\ome) - \T(\u^k(\ome), \ome) \right\| \le \epsilon.
			\end{equation*}
			\item[(2)]For each $\ome \in \Omega$,
			\begin{equation*}
				\lim\limits_{k \rightarrow \infty}\varphi_k(\ome) \rightarrow \varphi(\ome).\label{eq:dist-varphi}
			\end{equation*}
		\end{itemize}
		Let $\ome^K\in\mathrm{argmin}_{\ome \in\Omega}\varphi_{K}(\ome)$, then we have
		\begin{itemize}
			\item[(1)] any limit point $(\bar{\u},\bar{\ome})$ of the sequence $\{(\u^K(\ome^K), \ome^K) \}$ satisfies $\bar{\ome}\in\mathrm{argmin}_{\ome\in\Omega}\varphi(\ome)$ and $\bar{\u} = \T(\bar{\u},\bar{\ome}) $.
			\item[(2)] $\inf_{\ome \in \Omega}\varphi_K(\ome) \rightarrow \inf_{\ome \in \Omega} \varphi(\ome)$ as $K \rightarrow \infty$. 
		\end{itemize} 
	\end{proposition}
	\begin{proof}
		For any limit point $(\bar{\u},\bar{\ome})$ of the sequence $\{(\u^K(\ome^K), \ome^K) \}$, let $\{(\u^i(\ome^i), \ome^{i})\}$ be a subsequence of $\{(\u^K(\ome^K), \ome^K) \}$ such that $\u^i(\ome^i) \rightarrow \bar{\u} \in U$ and $\ome^{i} \rightarrow \bar{\ome} \in \Omega$.
		It follows from the assumption that for any $\epsilon > 0$, there exists $k(\epsilon) > 0$ such that for any $i > k(\epsilon)$, we have
		\begin{equation*}
			\| \u^i(\ome^i) - \T(\u^i(\ome^i),\ome^i) \| \le \epsilon.
		\end{equation*}
		By letting $i \rightarrow \infty$, and since $\T$ is closed on $U$, we have
		\begin{equation*}
			\| \bar{\u} - \T(\bar{\u}, \bar{\ome}) \| \le \epsilon.
		\end{equation*}
		As $\epsilon$ is arbitrarily chosen, we have $ \bar{\u} =  \T(\bar{\u}, \bar{\ome})$ and thus $\bar{\u} \in \mathtt{Fix}(T(\cdot, \bar{\ome}))$.
		
		Next, as $\ell$ is continuous at $(\bar{\u}, \bar{\ome})$, for any $\epsilon > 0$, there exists $k(\epsilon) > 0$ such that for any $i > k(\epsilon)$, it holds
		\begin{equation*}
			\ell(\bar{\u}, \bar{\ome}) \le \ell(\u^i(\ome^i),\ome^i) + \epsilon.
		\end{equation*}
		Then, we have, for any $i > k(\epsilon)$ and $\ome \in \Omega$,
		\begin{equation}\label{eq1}
			\begin{aligned}
				\varphi(\bar{\ome}) &= \inf_{ \u \in \mathtt{Fix}(T(\cdot, \bar{\ome})) \cap U  } \ell(\u, \bar{\ome}) \\
				&\le \ell(\bar{\u}, \bar{\ome}) \\
				&\le  \ell(\u^i(\ome^i),\ome^i) + \epsilon \\
				&= \varphi_i(\ome^i) + \epsilon \\
				&\le \varphi_i(\ome) + \epsilon.
			\end{aligned}
		\end{equation}
		Taking $i \rightarrow \infty$ and by the assumption, we have for any $\ome \in \Omega$,
		\begin{equation*}
			\begin{aligned}
				\varphi(\bar{\ome}) \le \lim_{i \rightarrow \infty}\varphi_i(\ome) + \epsilon = \varphi(\ome) + \epsilon.
			\end{aligned}
		\end{equation*}
		By letting $\epsilon \rightarrow 0$, we have
		\begin{equation*}
			\varphi(\bar{\ome}) \le \varphi(\ome), \quad \forall \ome \in \Omega,
		\end{equation*}
		which implies $\bar{\ome} \in \arg\min_{\ome \in \Omega} \varphi(\ome)$. 
		The second conclusion can be obtained by the same arguments used in the proof of \cite{BDA}[Theorem 1].
	\end{proof}

	From Proposition~\ref{thm:general}, we can derive the following theorem.
	
	\begin{thm}
		Let $\{\u^k(\ome)\}$ be the sequence generated by Eq.~\eqref{simple_bilevel_alg} with $\mu \in (0,1)$ and $s_k = \frac{s}{k+1}$, $s \in (0, \frac{\lambda_{\min}(\H_{lb})}{L_{\ell}} )$. 
		Then, let $\ome^K \in \mathrm{argmin}_{\ome \in\Omega}\varphi_{K}(\ome)$, and we have
		\begin{itemize}
			\item[(1)] any limit point $(\bar{\u},\bar{\ome})$ of the sequence $\{(\u^K(\ome^K),\ome^K) \}$ is a solution to the problem in Eq.~\eqref{bilevel_fix}, i.e., $\bar{\ome}\in\mathrm{argmin}_{\ome\in\Omega}\varphi(\ome)$ and $\bar{\u} = \T(\bar{\u},\bar{\ome}) $.
			\item[(2)] $\inf_{\ome \in \Omega}\varphi_K(\ome) \rightarrow \inf_{\ome \in \Omega} \varphi(\ome)$ as $K \rightarrow \infty$.
		\end{itemize}
	\end{thm}

	\begin{proof}
		As shown in Theorem~\ref{thm:simple_bilevel_convergence}, there exists $C > 0$ such that for any $\ome \in \Omega$,
		\[
		\| \u^k(\ome) - \T(\u^k(\ome),\ome) \|_\H^2 \le C\sqrt{\frac{1+\ln(1+k)}{k^{\frac{1}{4}}}}.
		\] 
		Since $\sqrt{\frac{1+\ln(1+k)}{k^{\frac{1}{4}}}} \rightarrow 0$ as $k \rightarrow \infty$ and $\{\u^k(\ome)\} \subset U$ from Eq.~\eqref{simple_bilevel_alg}, condition (a) in Proposition~\ref{thm:general} holds. Next, it follows from Theorem~\ref{thm:simple_bilevel_convergence} that $\varphi_k(\ome) = \ell(\u^k(\ome),\ome) \rightarrow \varphi(\ome)$ as $k \rightarrow \infty$ for any $\ome \in \Omega$ and thus condition (b) in Proposition~\ref{thm:general} is satisfied and the conclusion follows from Proposition~\ref{thm:general} immediately.
	\end{proof}

	\subsection{Stationary Analysis}
	
	Here we provide the convergence analysis of our algorithm with respect to stationary points,
	i.e., for any limit point $\bar{\ome}$ of the sequence $\{\ome^K\}$ satisfies $\nabla \varphi(\bar{\ome}) = 0$, where $\varphi(\ome)$ is defined in Eq.~\eqref{eq:appe phi_def}.
	
	We consider the special case where $U = \R^n$, and 
	$\mathtt{Fix}(\T(\cdot,\ome))$ has a unique fixed point, i.e. 
	the solution set $\S=\mathtt{Fix}(\T(\cdot,\ome))$ is a singleton, 
	and we denote the unique solution by $\u^*(\ome)$. 
	Our analysis is partly inspired by~\cite{BDA} and~\cite{grazzi2020iteration}.

	\begin{assumption}\label{assum_stationary ell}
		$\Omega$ is a compact set and $U = \R^n$. 
		$\mathtt{Fix}(\T(\cdot,\ome))$ is nonempty
		for any $\ome \in \Omega$. 
		$\ell(\u,\ome)$ is twice continuously differentiable on $\R^n \times \Omega$. 
		For any $\ome \in \Omega$, $\ell(\cdot,\ome) : \R^n \rightarrow \R$ is $L_\ell$-smooth, convex and bounded below by $M_0$.
	\end{assumption}

	For $\D(\cdot,\ome)$ we request a stronger assumption than Assumption~\ref{assum_T}
	that $\D(\cdot,\ome)$ is contractive with respect to $\| \cdot \|_{\H_{\ome}}$ throughout this part, to guarantee the uniqueness of the fixed point. 
	\begin{assumption} \label{assum_stationary D}
		There exist $\H_{ub} \succeq \H_{lb} \succ 0$, 
		such that for each $\ome \in \Omega$, there exists $\H_{ub} \succeq \H_{\ome} \succeq \H_{lb}$ such that
		\begin{itemize}
			\item[(1)] $\D(\cdot,\ome)$ is contractive with respect to $\| \cdot \|_{\H_{\ome}}$, i.e., there exists $\bar{\rho} \in (0,1)$, such that for all $(\u_1,\u_2) \in \R^n \times \R^n$,
			\begin{equation}
				\|\D(\u_1,\ome) - \D(\u_2,\ome) \|_{\H_{\ome}} \le \bar{\rho} \| \u_1 - \u_2\|_{\H_{\ome}}.
			\end{equation}
			
			\item[(2)] $\D(\cdot,\ome)$ is closed, i.e.,
			\[
			\mathrm{gph} \,\D(\cdot,\ome) := \{(\u,\v) \in \R^n \times \R^n ~|~ \v = \D(\u,\ome)\}
			\]
			is closed.
		\end{itemize}
	\end{assumption}

	In this part, 
	for succinctness we will write 
	$\H$ instead of $\H_{\ome}$, 
	and denote $\hat{\S}(\ome):= \mathrm{argmin}_{\u \in \mathtt{Fix}(\T(\cdot, \ome)) \cap U } \ell (\u,\ome)$. 
	We begin with the following lemma.

	\begin{lemma}\label{lem akbk}\cite{BDA}[Lemma 5]
		Let $\{a_k\}$ and $\{b_k\}$ be sequences of non-negative real numbers. 
		Assume that $b_k \rightarrow 0$ and there exist $\rho \in (0,1)$, and $n_0 \in \mathbb{N}$, such that $a_{k+1} \le \rho a_k + b_k, \ \forall k \ge n_0$.	
		Then $a_k \rightarrow 0$ as $k \rightarrow \infty$.
	\end{lemma}

	\begin{proposition}\label{prop1}
		Suppose Assumption~\ref{assum_stationary ell} and Assumption~\ref{assum_stationary D} are satisfied, 
		$\frac{\partial}{\partial \u} \T(\u,\ome)$ and $\frac{\partial}{\partial \ome} \T(\u,\ome)$ are Lipschitz continuous with respect to $\u$,
		and $\hat{\S}(\ome)$ is nonempty for all $\ome \in \Omega$. 
		Let $\{\u^k(\ome)\}$ be the sequence generated by Eq.~\eqref{simple_bilevel_alg} with $\mu \in (0,1)$ and $s_k = \frac{s}{k+1}$, $s \in (0, \frac{\lambda_{\min}(\H_{lb})}{L_{\ell}} )$.
		Then we have 
		\begin{equation*}
			\sup_{\ome \in \Omega} \left\| \nabla \varphi_k(\ome) - \nabla \varphi(\ome) \right\|_\H \rightarrow 0,\ \text{as}\ k \rightarrow \infty.
		\end{equation*}
		
	\end{proposition}
	\begin{proof}
		According to the update scheme of $\u^{k+1}$ given in Eq.~\eqref{simple_bilevel_alg},
		since $U = \R^n$, we have for all $\ome \in \Omega$,
		\begin{equation}\label{eq:uk+1}
			\u^{k+1}(\ome) =  \mu \v^{k+1}_u(\ome) + (1-\mu) \v^{k+1}_l(\ome)
			= \mu \left( \u^k(\ome) - s_{k+1} \H^{-1}\frac{\partial }{\partial \u}\ell(\u^k(\ome),\ome) \right)
			+ (1-\mu) \T(\u^k(\ome),\ome).
		\end{equation}
		As $\u^*(\ome)$ is the fixed point of $\T$, we have $\T(\u^*(\ome),\ome) = \u^*(\ome)$.
		Thus 
		\begin{equation*}
			\u^{k+1}(\ome) - \u^*(\ome) = \mu \left( \u^k(\ome) - \u^*(\ome) \right) 
			+ (1-\mu) \left( \T(\u^k(\ome),\ome) - \T(\u^*(\ome),\ome) \right) 
			- \mu s_{k+1} \H^{-1}\frac{\partial }{\partial \u}\ell(\u^k(\ome),\ome).
		\end{equation*}
		Then
		\begin{equation*}
			\begin{aligned}
				& \left\| \u^{k+1}(\ome) - \u^*(\ome) \right\|_\H \\
				& \le \mu \left\| \u^k(\ome) - \u^*(\ome) \right\|_\H
				+ (1-\mu) \left\| \T(\u^k(\ome),\ome) - \T(\u^*(\ome),\ome) \right\|_\H
				+ \mu s_{k+1} \left\| \H^{-1}\frac{\partial }{\partial \u}\ell(\u^k(\ome),\ome) \right\|_\H. 
			\end{aligned}	
		\end{equation*}
		From the contraction of $\D(\cdot,\ome)$, we have 
		$\T(\cdot,\ome)$ is contractive with respect to $\| \cdot \|_{\H}$, i.e., for all $(\u_1,\u_2) \in \R^n \times \R^n$,
		$\|\T(\u_1,\ome) - \T(\u_2,\ome) \|_{\H} \le \rho \| \u_1 - \u_2\|_{\H}$,
		where $\rho \in (0,1)$.
		The $L_\ell$-smoothness of $\ell$ yields that
		\begin{equation*}
			\begin{aligned}
				\left\| \H^{-1}\frac{\partial }{\partial \u}\ell(\u^k(\ome),\ome) \right\|_\H
				&\le \frac{1}{\sqrt{\lambda_{\min}(\H)}} \left\| \frac{\partial }{\partial \u}\ell(\u^k(\ome),\ome) \right\|
				\le \frac{1}{\sqrt{\lambda_{\min}(\H)}} \left( \left\| \frac{\partial }{\partial \u}\ell(\u^*(\ome),\ome) \right\| + L_\ell \left\| \u^k(\ome) - \u^*(\ome) \right\| \right), \\
			\end{aligned}
		\end{equation*}
		and $\| \frac{\partial }{\partial \u}\ell(\u^*(\ome),\ome) \|$ is bounded for all $\ome \in \Omega$.
		Denote $M_\ell^* :=\sup_{\ome \in \Omega} \| \frac{\partial }{\partial \u}\ell(\u^*(\ome),\ome) \|$,
		and thus
		\begin{equation} \label{eq:H}
			\begin{aligned}
				\left\| \H^{-1}\frac{\partial }{\partial \u}\ell(\u^k(\ome),\ome) \right\|_\H
				&\le \frac{1}{\sqrt{\lambda_{\min}(\H)}} \left( M_\ell^* + L_\ell \left\| \u^k(\ome) - \u^*(\ome) \right\|  \right).
			\end{aligned}
		\end{equation}
		Therefore, 
		\begin{equation*}
			\begin{aligned}
				& \left\| \u^{k+1}(\ome) - \u^*(\ome) \right\|_\H \\
				& \le \left( \mu + (1-\mu) \rho \right) \left\| \u^k(\ome) - \u^*(\ome) \right\|_\H 
				+ \frac{\mu s_{k+1}}{\sqrt{\lambda_{\min}(\H)}} \left( M_\ell^* + L_\ell \left\| \u^k(\ome) - \u^*(\ome) \right\|  \right) \\
				& \le \left( \mu + (1-\mu) \rho + \frac{\mu s_{k+1} L_\ell}{\lambda_{\min}(\H)} \right) \left\| \u^k(\ome) - \u^*(\ome) \right\|_\H + \frac{\mu s_{k+1}M_\ell^*}{\sqrt{\lambda_{\min}(\H)}} . 
			\end{aligned}	
		\end{equation*}
		As $s_{k+1} \rightarrow 0$ and $\rho \in (0,1)$, there exists $n_0 \in \mathbb{N}$, such that $\mu + (1-\mu) \rho + \frac{\mu s_{k+1} L_\ell}{\lambda_{\min}(\H)} \in (0,1), \forall k>n_0$.
		Then we obtain from Lemma~\ref{lem akbk} that 
		\begin{equation} \label{eq:sup uk+1 - u*}
			\sup_{\ome \in \Omega} \| \u^k(\ome) - \u^*(\ome) \|_\H \rightarrow 0,\ \text{as}\ k \rightarrow \infty.
		\end{equation}

		By taking derivative with respect to $\ome$ on both sides of Eq.~\eqref{eq:uk+1}, we have 
		\begin{equation}\label{eq:partial uk}
			\frac{\partial \u^{k+1}(\ome)}{\partial \ome}
			= \mu \left( \frac{\partial \u^k(\ome)}{\partial \ome} - s_{k+1} \H^{-1} \nabla_{\u\u} \ell(\u^k(\ome),\ome) \frac{\partial \u^k(\ome)}{\partial \ome} \right)
			+ (1-\mu) \left( 
			\frac{\partial \T(\u^k(\ome),\ome)}{\partial \u}  \frac{\partial \u^k(\ome)}{\partial \ome} 
			+ \frac{\partial \T(\u^k(\ome),\ome)}{\partial \ome} \right)  .
		\end{equation}
		From $\T(\u^*(\ome),\ome) = \u^*(\ome)$, we obtain 
		\begin{equation} \label{eq:partial u*}
			\frac{\partial \u^*(\ome)}{\partial \ome}
			= \frac{\partial \T(\u^*(\ome),\ome)}{\partial \u}\frac{\partial \u^*(\ome)}{\partial \ome} + \frac{\partial \T(\u^*(\ome),\ome)}{\partial \ome}.
		\end{equation}
		Then combining Eq.~\eqref{eq:partial uk} and Eq.~\eqref{eq:partial u*} derives 
		\begin{equation*}
			\begin{aligned}
				\frac{\partial \u^{k+1}(\ome)}{\partial \ome} - \frac{\partial \u^*(\ome)}{\partial \ome}
				= & \mu \left( \frac{\partial \u^k(\ome)}{\partial \ome} - \frac{\partial \u^*(\ome)}{\partial \ome} \right)
				+ (1-\mu) \frac{\partial \T(\u^k(\ome),\ome)}{\partial \u} \left( \frac{\partial \u^k(\ome)}{\partial \ome} - \frac{\partial \u^*(\ome)}{\partial \ome} \right) \\
				& + (1-\mu)\left( \frac{\partial \T(\u^k(\ome),\ome)}{\partial \u} - \frac{\partial \T(\u^*(\ome),\ome)}{\partial \u} \right) \frac{\partial \u^*(\ome)}{\partial \ome} \\
				& + (1-\mu) \left( \frac{\partial \T(\u^k(\ome),\ome)}{\partial \ome} - \frac{\partial \T(\u^*(\ome),\ome)}{\partial \ome} \right) \\
				& - \mu s_{k+1} \H^{-1} \nabla_{\u\u} \ell(\u^k(\ome),\ome) \left( \frac{\partial \u^k(\ome)}{\partial \ome} - \frac{\partial \u^*(\ome)}{\partial \ome} \right) \\
				& - \mu s_{k+1} \H^{-1} \nabla_{\u\u} \ell(\u^k(\ome),\ome) \frac{\partial \u^*(\ome)}{\partial \ome}.
			\end{aligned}		
		\end{equation*}
		
		Hence, 
		\begin{equation} \label{eq:partial uk+1 - partial u*}
			\begin{aligned}
				\sup_{\ome \in \Omega} \left\| \frac{\partial \u^{k+1}(\ome)}{\partial \ome} - \frac{\partial \u^*(\ome)}{\partial \ome} \right\|_\H
				\le & \left( \mu 
				+ (1-\mu) \sup_{\ome \in \Omega}\left\| \frac{\partial \T(\u^k(\ome),\ome)}{\partial \u} \right\|_\H 
				+ \mu s_{k+1} \sup_{\ome \in \Omega}\left\| \H^{-1} \nabla_{\u\u} \ell(\u^k(\ome),\ome) \right\|_\H \right) \\
				& \times \sup_{\ome \in \Omega}\left\| \frac{\partial \u^k(\ome)}{\partial \ome} - \frac{\partial \u^*(\ome)}{\partial \ome} \right\|_\H \\		
				& + (1-\mu) \sup_{\ome \in \Omega}\left\| \frac{\partial \T(\u^*(\ome),\ome)}{\partial \u} - \frac{\partial \T(\u^k(\ome),\ome)}{\partial \u} \right\|_\H 
				\sup_{\ome \in \Omega} \left\| \frac{\partial \u^*(\ome)}{\partial \ome} \right\|_\H \\
				& + (1-\mu) \sup_{\ome \in \Omega}\left\| \frac{\partial \T(\u^*(\ome),\ome)}{\partial \ome} - \frac{\partial \T(\u^k(\ome),\ome)}{\partial \ome} \right\|_\H \\
				& + \mu s_{k+1} \sup_{\ome \in \Omega}\left\| \H^{-1} \nabla_{\u\u} \ell(\u^k(\ome),\ome) \right\|_\H 
				\sup_{\ome \in \Omega} \left\| \frac{\partial \u^*(\ome)}{\partial \ome} \right\|_\H.
			\end{aligned}		
		\end{equation}
		Next, before using Lemma~\ref{lem akbk}, we first show that the last three terms on the right hand side of the above inequality converge to 0 as $k\rightarrow \infty$.
		From the Lipschitz continuity of $\frac{\partial \T}{\partial \u}$ and $\frac{\partial \T}{\partial \ome}$, 
		since $\u^{k+1}(\ome)$ uniformly converges to $\u^*(\ome)$ with respect to $\| \cdot \|_\H$ as $k \rightarrow \infty$ as proved in Eq.~\eqref{eq:sup uk+1 - u*},
		we have $\sup_{\ome \in \Omega}\left\| \frac{\partial \T(\u^*(\ome),\ome)}{\partial \u} - \frac{\partial \T(\u^k(\ome),\ome)}{\partial \u} \right\|_\H \rightarrow 0$
		and $\sup_{\ome \in \Omega}\left\| \frac{\partial \T(\u^*(\ome),\ome)}{\partial \ome} - \frac{\partial \T(\u^k(\ome),\ome)}{\partial \ome} \right\|_\H \rightarrow 0$
		as $k \rightarrow \infty$.
		From Eq.~\eqref{eq:partial u*}, $\left( \I - \frac{\partial \T(\u^*(\ome),\ome)}{\partial \u} \right) \frac{\partial \u^*(\ome)}{\partial \ome} = \frac{\partial \T(\u^*(\ome),\ome)}{\partial \ome}$,
		i.e., $\frac{\partial \u^*(\ome)}{\partial \ome} = \left( \I - \frac{\partial \T(\u^*(\ome),\ome)}{\partial \u} \right)^{-1} \frac{\partial \T(\u^*(\ome),\ome)}{\partial \ome}$.
		Along with the Lipschitz continuity of $\frac{\partial \T}{\partial \u}$ and $\frac{\partial \T}{\partial \ome}$, 
		we have  $\frac{\partial \u^*(\ome)}{\partial \ome}$ is continuous on the compact set $\Omega$,
		and thus 
		\begin{equation} \label{eq:sup partial u*}
			\sup_{\ome \in \Omega} \left\| \frac{\partial \u^*(\ome)}{\partial \ome} \right\|_\H < + \infty.
		\end{equation}
		From the twice continuous differentiability of $\ell$, 
		it holds that
		$\sup_{\ome \in \Omega} \left\| \H^{-1} \nabla_{\u\u} \ell(\u^k(\ome),\ome) \right\|_\H 
		\le \sup_{\ome \in \Omega} \frac{1}{\sqrt{\lambda_{\min}(\H)}} \left\| \nabla_{\u\u} \ell(\u^k(\ome),\ome) \right\| < + \infty$.
		Then from $s_{k+1} \rightarrow 0$ as $k \rightarrow \infty$,
		we have $\mu s_{k+1} \sup_{\ome \in \Omega}\left\| \H^{-1} \nabla_{\u\u} \ell(\u^k(\ome),\ome) \right\|_\H 
		\sup_{\ome \in \Omega} \left\| \frac{\partial \u^*(\ome)}{\partial \ome} \right\|_\H \rightarrow 0$ as $k \rightarrow \infty$.
		Thus, the last three terms in Eq.\eqref{eq:partial uk+1 - partial u*} converge to 0 as $k \rightarrow \infty$.
		
		As for the coefficient $\mu 
		+ (1-\mu) \sup_{\ome \in \Omega}\left\| \frac{\partial \T(\u^k(\ome),\ome)}{\partial \u} \right\|_\H 
		+ \mu s_{k+1} \sup_{\ome \in \Omega}\left\| \H^{-1} \nabla_{\u\u} \ell(\u^k(\ome),\ome) \right\|_\H$ in Eq.\eqref{eq:partial uk+1 - partial u*},
		from the contraction of $\T$ and the Lipschitz continuity of $\frac{\partial \T}{\partial \u}$, we have $\sup_{\ome \in \Omega} \left\| \frac{\partial \T(\u^k(\ome),\ome)}{\partial \u} \right\|_\H \le \rho <1$.
		Since $\sup_{\ome \in \Omega}\left\| \H^{-1} \nabla_{\u\u} \ell(\u^k(\ome),\ome) \right\|_\H < + \infty$ and $s_{k+1} \rightarrow 0$ as $k \rightarrow \infty$,
		i.e., $\mu s_{k+1} \sup_{\ome \in \Omega}\left\| \H^{-1} \nabla_{\u\u} \ell(\u^k(\ome),\ome) \right\|_\H \rightarrow 0$ as $k \rightarrow \infty$,
		there exists $n_0 \in \mathbb{N}$, such that 
		$$\mu 
		+ (1-\mu) \sup_{\ome \in \Omega}\left\| \frac{\partial \T(\u^k(\ome),\ome)}{\partial \u} \right\|_\H 
		+ \mu s_{k+1} \sup_{\ome \in \Omega}\left\| \H^{-1} \nabla_{\u\u} \ell(\u^k(\ome),\ome) \right\|_\H \in (0,1), \forall k>n_0.$$
		Therefore, by applying Lemma~\ref{lem akbk} on Eq.~\eqref{eq:partial uk+1 - partial u*}, we obtain 
		\begin{equation} \label{eq:sup partial uk -partial u*}
			\sup_{\ome \in \Omega} \left\| \frac{\partial \u^k(\ome)}{\partial \ome} - \frac{\partial \u^*(\ome)}{\partial \ome} \right\|_\H \rightarrow 0,\ \text{as}\ k \rightarrow \infty.
		\end{equation}

		Finally, we will prove 
		\begin{equation*}
			\sup_{\ome \in \Omega} \left\| \nabla \varphi_k(\ome) - \nabla \varphi(\ome) \right\|_\H \rightarrow 0,\ \text{as}\ k \rightarrow \infty.
		\end{equation*}
		From the definition of $\varphi_k(\ome) = \ell(\u^k(\ome),\ome)$ and $\varphi(\ome)= \ell(\u^*(\ome),\ome)$, we have
		$$
		\nabla \varphi_k(\ome) = \frac{\partial \ell(\u^k(\ome),\ome) }{\partial \u} \frac{\partial \u^k(\ome)}{\partial \ome}
		+ \frac{\partial \ell(\u^k(\ome),\ome)}{\partial \ome},
		$$
		$$
		\nabla \varphi(\ome) = \frac{\partial \ell(\u^*(\ome),\ome) }{\partial \u} \frac{\partial \u^*(\ome)}{\partial \ome}
		+ \frac{\partial \ell(\u^*(\ome),\ome)}{\partial \ome}.
		$$
		Thus, 
		\begin{equation*}
			\begin{aligned}
				\nabla \varphi_k(\ome) - \nabla \varphi(\ome) 
				= &
				\frac{\partial \ell(\u^k(\ome),\ome) }{\partial \u} 
				\left( \frac{\partial \u^k(\ome)}{\partial \ome} - \frac{\partial \u^*(\ome)}{\partial \ome} \right) \\
				& + \left( \frac{\partial \ell(\u^k(\ome),\ome) }{\partial \u} - \frac{\partial \ell(\u^*(\ome),\ome) }{\partial \u} \right) \frac{\partial \u^*(\ome)}{\partial \ome}
				+ \left( \frac{\partial \ell(\u^k(\ome),\ome)}{\partial \ome} - \frac{\partial \ell(\u^*(\ome),\ome)}{\partial \ome} \right).
			\end{aligned}		
		\end{equation*}
		Then we have the following estimation
		\begin{equation} \label{eq:sup nabla varphi}
			\begin{aligned}
				\sup_{\ome \in \Omega} \| \nabla \varphi_k(\ome) - \nabla \varphi(\ome) \|_\H
				\le &
				\sup_{\ome \in \Omega} \left\| \frac{\partial \ell(\u^k(\ome),\ome) }{\partial \u} \right\|_\H
				\sup_{\ome \in \Omega} \left\| \frac{\partial \u^k(\ome)}{\partial \ome} - \frac{\partial \u^*(\ome)}{\partial \ome} \right\|_\H \\
				& + \sup_{\ome \in \Omega} \left\| \frac{\partial \ell(\u^k(\ome),\ome) }{\partial \u} - \frac{\partial \ell(\u^*(\ome),\ome) }{\partial \u} \right\|_\H 
				\sup_{\ome \in \Omega} \left\| \frac{\partial \u^*(\ome)}{\partial \ome} \right\|_\H \\
				& + \sup_{\ome \in \Omega} \left\| \frac{\partial \ell(\u^k(\ome),\ome)}{\partial \ome} - \frac{\partial \ell(\u^*(\ome),\ome)}{\partial \ome} \right\|_\H.
			\end{aligned}		
		\end{equation}
		We have obtained $\sup_{\ome \in \Omega} \left\| \frac{\partial \u^k(\ome)}{\partial \ome} - \frac{\partial \u^*(\ome)}{\partial \ome} \right\|_\H \rightarrow 0,\ \text{as}\ k \rightarrow \infty$ in Eq.~\eqref{eq:sup partial uk -partial u*},
		and $\sup_{\ome \in \Omega} \left\| \frac{\partial \u^*(\ome)}{\partial \ome} \right\|_\H < + \infty$ in Eq.~\eqref{eq:sup partial u*}.
		Then from the $L_\ell$-smoothness of $\ell(\cdot,\ome)$ and 
		the twice continuous differentiability of $\ell$ on $\R^n \times \Omega$,
		where $\Omega$ is compact,
		we have $\sup_{\ome \in \Omega} \left\| \frac{\partial \ell(\u^k(\ome),\ome) }{\partial \u} - \frac{\partial \ell(\u^*(\ome),\ome) }{\partial \u} \right\|_\H$
		and $\sup_{\ome \in \Omega} \left\| \frac{\partial \ell(\u^k(\ome),\ome) }{\partial \ome} - \frac{\partial \ell(\u^*(\ome),\ome) }{\partial \ome} \right\|_\H$
		converge to 0 as $k \rightarrow \infty$.
		Also, it holds that $\sup_{\ome \in \Omega} \left\| \frac{\partial \ell(\u^k(\ome),\ome) }{\partial \u} \right\|_\H < + \infty$ 
		from Eq.~\eqref{eq:H}.
		Therefore, the three terms on the right hand side of Eq.~\eqref{eq:sup nabla varphi} all converge to 0 as $k \rightarrow \infty$,
		which derives
		\begin{equation*}
			\sup_{\ome \in \Omega} \| \nabla \varphi_k(\ome) - \nabla \varphi(\ome) \|_\H \rightarrow 0,\ \text{as}\ k \rightarrow \infty.
		\end{equation*}
	\end{proof}

	\begin{thm}
		Suppose Assumption~\ref{assum_stationary ell} and Assumption~\ref{assum_stationary D} are satisfied, 
		$\frac{\partial}{\partial \u} \T(\u,\ome)$ and $\frac{\partial}{\partial \ome} \T(\u,\ome)$ are Lipschitz continuous with respect to $\u$,
		and $\hat{\S}(\ome)$ is nonempty for all $\ome \in \Omega$. 
		Let $\{\u^k(\ome)\}$ be the sequence generated by Eq.~\eqref{simple_bilevel_alg} with $\mu \in (0,1)$ and $s_k = \frac{s}{k+1}$, $s \in (0, \frac{\lambda_{\min}(\H_{lb})}{L_{\ell}} )$.
		Let $\ome^K$ be an $\varepsilon_K$-stationary point of $\varphi_{K}(\ome)$, i.e., 
		\begin{equation*}
			\| \nabla \varphi_K(\ome^K) \| = \varepsilon_K.
		\end{equation*}
		Then if $\varepsilon_K \rightarrow 0$, we have that any limit point $\bar{\ome}$ of the sequence $\{\ome^K\}$ is a stationary point of $\varphi$, i.e., 
		\begin{equation*}
			\nabla \varphi(\bar{\ome}) = 0.
		\end{equation*}
	\end{thm}
	
	\begin{proof}
		For any limit point $\bar{\ome}$ of the sequence $\{\ome^K\}$, let $\{\ome^{l}\}$ be a subsequence of $\{\ome^K\}$ such that $\ome^{l} \rightarrow \bar{\ome} \in \Omega$. For any $\epsilon > 0$, as shown in Proposition~\ref{prop1}, there exists $k_1$ such that 
		\begin{equation*}
			\sup_{\ome \in \Omega} \| \nabla \varphi_k(\ome) - \nabla \varphi(\ome) \| \le \epsilon/2, \quad \forall k \ge k_1.
		\end{equation*}
		Since $\varepsilon_k \rightarrow 0$, there exists $k_2 > 0$ such that $\varepsilon_k \le  \epsilon/2$ for any $k \ge k_2$. Then, for any $l \ge \max(k_1,k_2)$, we have
		\begin{equation*}
			\|\nabla \varphi(\ome^l) \| \le \|\nabla \varphi(\ome^l) - \nabla \varphi_l(\ome^l) \| + \| \nabla \varphi_l(\ome^l)  \| \le  \epsilon.
		\end{equation*}
		Taking $l \rightarrow \infty$ in the above inequality, and by the continuity of $\nabla \varphi$, we get
		\begin{equation*}
			\|\nabla \varphi(\bar{\ome}) \| \le \epsilon.
		\end{equation*}
		Since $\epsilon$ is arbitrarily chosen, we obtain $\nabla \varphi(\bar{\ome}) = 0$.
	\end{proof}

	\newpage
	\section{Detailed descriptions for $\D$ in Section \ref{sec:application}}
	\label{sec:appendix B about D}
	
	\subsection{Proximal Gradient Method ($\D_{\mathtt{PG}}$)}
	Consider the following convex minimization problem
	\begin{equation}\label{PGM_prob}
		\min_{\u \in \R^{n}} f(\u) + g(\u) 
	\end{equation}
	where $f, g :\R^{n} \rightarrow \R $ are proper, closed, and convex functions, 
	and $f$ is a continuously differentiable function with a Lipschitz continuous gradient.
	The proximal gradient method for solving problem Eq.~\eqref{PGM_prob} reads as
	\begin{equation}\label{PGM_scheme}
		\u^{k+1} = \underset{\u}{\mathrm{argmin}}~\left\{ f(\u^k) + \langle \nabla f(\u^k) ,\u - \u^k \rangle + g(\u) +  \frac{1}{2\gamma}\|\u- \u^k\|^2_\G \right\}, \\
	\end{equation}
	where $\G \succeq 0$ and $\|\u\|_\G^2:= \langle \u, \G\u\rangle$. By parameterizing functions $f$, $g$ and matrix $\G$ by hyper-parameter $\ome$ in Eq.~\eqref{PGM_scheme} to make them learnable, we can obtain $\D_{\mathtt{PG}}$ in the following form,
	\begin{equation}\label{PGM_scheme_p}
		\D_{\mathtt{PG}}(\u^{k},\ome) = \underset{\u}{\mathrm{argmin}}~\left\{ f(\u^k,\ome) + \langle \nabla_\u f(\u^k,\ome) ,\u - \u^k \rangle + g(\u,\ome) +  \frac{1}{2\gamma}\|\u- \u^k\|^2_{\G(\ome)} \right\}.
	\end{equation}
	
	Next we show that $\D_{\mathtt{PG}}$ satisfies Assumption~\ref{assum_T} under the following standing assumption.
	\begin{assumption}\label{assum_pg}
		For any $\ome \in \Omega$, $f(\cdot,\ome)$ and $g(\cdot,\ome)$  are proper closed convex functions and $f(\cdot,\ome)$ are $L_{f}$-smooth. And there exist $\H_{ub} \succeq \H_{lb} \succ 0$ such that  $\H_{ub} \succeq \G(\ome) \succeq \H_{lb}$ for each $\ome \in \Omega$.
	\end{assumption}
	
	\begin{proposition}
		Suppose Assumption~\ref{assum_pg} holds and $\gamma\in (0,2\lambda_{\min}(\H_{lb})/L_f)$. Then $\D_{\mathtt{PG}}$ satisfies Assumption~\ref{assum_T}.
	\end{proposition} 
	\begin{proof}
		Since $\D_{\mathtt{PG}}( \cdot, \ome) = \left( \I + \gamma \G(\ome)^{-1}\partial_\u g(\cdot,\ome) \right)^{-1}\left( \I - \gamma \G(\ome)^{-1} \nabla_\u f( \cdot ,\ome)  \right)$, and $\nabla\u f( \cdot ,\ome)$ is $L_f$-Lipschitz continuous, by \cite{Cui2019}[Lemma 3.2], \cite{Heinz-MonotoneOperator-2011}[Proposition 4.25] and Assumption~\ref{assum_pg}, we have $\D_{\mathtt{PG}}( \cdot, \ome)$ satisfies Assumption~\ref{assum_T} (1) when setting $\G(\ome)$ as $\H_{\ome}$ for any $\ome \in \Omega$, with $\gamma\in (0,2\lambda_{\min}(\G_{lb})/L_f)$. The closedness of $\D_{\mathtt{PG}}( \cdot, \ome)$ follows from the outer-semicontinuity of $\partial_\u g(\cdot,\ome)$ and continuity of $\nabla_\u f( \cdot ,\ome)$, and then the proof is completed.
	\end{proof}

	\subsection{Proximal Augmented Lagrangian Method ($\D_{\mathtt{ALM}}$)} \label{sec:appendix B1 DPG}
	
	Consider the following convex minimization problem with linear constraints 
	\begin{equation}\label{ALM_prb1}
		\begin{aligned}
			\min_{\u \in \R^{n}} & ~~f(\u) \\
			s.t.~~ & \A\u =  \b,
		\end{aligned}
	\end{equation}
	where $f :\R^{n} \rightarrow \R $ is a proper closed convex function. Proximal ALM for solving problem Eq.~\eqref{ALM_prb1} is given by
	\begin{equation}\label{ALM_scheme}
		\begin{aligned}
			\u^{k+1} &= \underset{\u}{\mathrm{argmin}}~\left\{ f(\u) + \langle \lamm^k,\A\u - \b\rangle + \frac{\beta}{2}\|\A\u - \b\|^2 + \frac{1}{2}\|\u- \u^k\|_\G^2\right\}, \\
			\lamm^{k+1} & = \lamm^k + \beta( \A\u^{k+1} -\b),
		\end{aligned}
	\end{equation}
	where $\beta>0$.
	We can parameterize functions $f$, matrices $\A$ and $\G$, and vector $\b$ by hyper-parameter $\theta$ in Eq.~\eqref{ALM_scheme} to make them learnable as the following
	\begin{equation}\label{ALM_scheme_p}
		\begin{aligned}
			\u^{k+1} &= \underset{\u}{\mathrm{argmin}}~\left\{ f(\u,\theta) + \langle \lamm^k, \A(\theta)\u - \b\rangle + \frac{\beta}{2}\|\A(\theta)\u - \b(\theta)\|^2 + \frac{1}{2}\|\u- \u^k\|_{\G(\theta)}^2\right\}, \\
			\lamm^{k+1} & = \lamm^k + \beta( \A(\theta)\u^{k+1} - \b(\theta)).
		\end{aligned}
	\end{equation}
	By setting $\beta$ as the hyper-parameter and letting ${\ome}:= (\theta,\beta)$ as hyper-parameters for scheme Eq.~\eqref{ALM_scheme_p}, we can define $\D_{\mathtt{ALM}}$ by scheme Eq.~\eqref{ALM_scheme_p} as the following
	\begin{equation}
		\D_{\mathtt{ALM}}(\u^k, \ome) = \u^{k+1}.
	\end{equation}

	We can show that $\D_{\mathtt{ALM}}$ satisfies Assumption~\ref{assum_T} under the following standing assumption.
	\begin{assumption}\label{assum_alm}
		For any ${\ome} \in \Omega$, $f(\cdot,\theta)$ is proper closed convex functions. And there exist $\beta_{ub} \ge \beta_{lb} > 0$ and $\G_{ub} \succeq \G_{lb} \succ 0$ such that $\beta \in [\beta_{lb}, \beta_{ub}]$ and $\G_{ub} \succeq \G(\ome) \succeq \G_{lb}$ for each $\ome \in \Omega$.
	\end{assumption}

	\begin{proposition}
		Suppose Assumption~\ref{assum_alm} holds. Then $\D_{\mathtt{ALM}}$ satisfies Assumption~\ref{assum_T}.
	\end{proposition} 
	\begin{proof}
		As shown in \cite{ALM}, $\D_{\mathtt{ALM}}(\cdot, {\ome}) = (\Phi_{\ome} + \H_{\ome})^{-1}\H_{\ome}$ with 
		$$\Phi_{\ome}(\u,\ome) = 
		\begin{pmatrix}
			\partial_\u f(\u,\theta) + \A(\theta)^\top\lamm \\  - \A(\theta)\u +  \b(\theta)
		\end{pmatrix} \quad \text{and} \quad \H_\omega = 
		\begin{pmatrix}
			\G(\theta) & 0 \\  0 & \frac{1}{\beta}\textbf{I}
		\end{pmatrix}. 
		$$
		Since $\Phi_{\omega}$ is maximal monotone mapping (see, e.g, \cite{rockafellar}[Corollary 37.5.2]), and $\H_{\ome} \succ 0$ from Assumption~\ref{assum_alm}, it can be easily verified that $\D_{\mathtt{ALM}}(\cdot, \ome)$ is firmly non-expansive with respect to $\|\cdot\|_{\H_{\ome}}$, and thus $\D_{\mathtt{ALM}}( \cdot, \omega)$ satisfies Assumption~\ref{assum_T} (1) for any ${\ome} \in \Omega$. The closedness of $\D_{\mathtt{ALM}}( \cdot, \ome)$ follows from the outer-semicontinuity of $\partial_\u f(\cdot,{\ome})$. Next, Assumption~\ref{assum_alm} implies the existences of $\H_{ub} \succeq \H_{lb} \succ 0$ such that $\H_{ub} \succeq \H_{\ome} \succeq \H_{lb}$ for each ${\ome} \in \Omega$ and then the conclusion follows immediately.
	\end{proof}

	\subsection{Composition of $\D_{\num}$ and $\D_{\net}$  ($\D_{\num} \circ \D_{\net}$)}
	
	\begin{proposition}
		Suppose $\D_{\num}$ and $\D_{\net}$ satisfy Assumption~\ref{assum_T} with the same $\H_{\ome}$. Then $\D_{\num} \circ \D_{\net}$ satisfies Assumption~\ref{assum_T}.
	\end{proposition} 
	\begin{proof}
		Since $\D_{\num}$ and $\D_{\net}$ satisfy Assumption~\ref{assum_T} with the same $\H_{\ome}$, it can be easily verified from the definition that $\D_{\num} \circ \D_{\net}$ satisfies Assumption~\ref{assum_T} (1) with $\H_{\ome}$. 
		For the closedness of $\D_{\num}( \cdot, \ome) \circ \D_{\net}( \cdot, \ome)$, for any fixed $\ome \in \Omega$, consider any sequence $\{(\u^k,\v^k)\} \in \mathrm{gph} (\D_{\num}( \cdot, \ome) \circ \D_{\net}( \cdot, \ome))$ satisfying $(\u^k,\v^k) \rightarrow (\bar{\u}, \bar{\v})$. Since $\D_{\net}( \cdot, \ome)$ satisfies Assumption~\ref{assum_T} (1) with $\H_{\ome} \succ 0$, and $\{\u^k\}$ is bounded, we have $\D_{\net}( \u^k, \ome)$ is bounded, and thus there exists a subsequence $\{(\u^i,\v^i)\} \subseteq \{(\u^k,\v^k)\}$ such that $\D_{\net}( \u^i, \ome) \rightarrow \bar{w}$. 
		Then it follows from the closedness of $\D_{\net}( \cdot, \ome)$ and $\D_{\num}( \cdot, \ome)$, 
		that $(\bar{\u}, \bar{\ome}) \in \mathrm{gph} \D_{\net}( \cdot, \ome)$ and $(\bar{\ome}, \bar{\v}) \in \mathrm{gph} \D_{\num}( \cdot, \ome) $, and thus $(\bar{\u}, \bar{\v}) \in \mathrm{gph} (\D_{\num}( \cdot, \ome) \circ \D_{\net}( \cdot, \ome))$. 
		Therefore, the closedness of $\D_{\num}( \cdot, \ome) \circ \D_{\net}( \cdot, \ome)$ follows and the proof is completed.
	\end{proof}
	\begin{remark}
		Given any non-expansive $\D_{\net}$  (which can be achieved by spectral normalization) and any positive-definite matrix $\H_{\ome}$, by setting  $\D_{\mathtt{net^*}}=\H_{\ome}^{-1/2}\D_{\net}\H_{\ome}^{1/2}$, we can obtain that $\D_{\mathtt{net^*}}$  satisfies Assumption~\ref{assum_T} with $\H_{\ome}$.
		%
		
	\end{remark}

	\subsection{Summary of Operator $\D$}

	\begin{table}[!htbp]\small \label{tab:appendix summary of D}
		\centering
		\caption{Summary of operator $\D$ and $\H_{\ome}$.}
		\renewcommand\arraystretch{1.2}
		\setlength{\tabcolsep}{1mm}{
			\renewcommand\arraystretch{1.1}
			\resizebox{\textwidth}{!}{
				\begin{tabular}{ccc}
					\hline
					$\D$ & Operator&$\H_{\ome}$ \\
					\hline
					\multirow{3}{*}{$\D_\num$}&${\mathtt{PG}}:\u^{k+1} = \underset{\u}{\mathrm{argmin}}~\left\{ f(\u^k,\ome) + \langle \nabla_\u f(\u^k,\ome) ,\u - \u^k \rangle + g(\u,\ome) +  \frac{1}{2\gamma}\|\u- \u^k\|^2_{\G(\ome)}\right\}$ &$\G({\ome})$\\
					&${\mathtt{ALM}}:\left\{\	
					\begin{aligned}
						\u^{k+1} &= \underset{\u}{\mathrm{argmin}}~\left\{ f(\u,\theta) + \langle \lamm^k, \A(\theta)\u - \b\rangle + \frac{\beta}{2}\|\A(\theta)\u - \b(\theta)\|^2 + \frac{1}{2}\|\u- \u^k\|_{\G(\theta)}^2\right\} \\
						\lamm^{k+1} & = \lamm^k + \beta( \A(\theta)\u^{k+1} - \b(\theta))
					\end{aligned} \right.
					$&
					$\begin{pmatrix}
						\G(\theta) & 0 \\  0 & \frac{1}{\beta}I
					\end{pmatrix}$ \\
					\hline
					
					$\D_\net$ & Various networks with non-expansive property($1$-Lipschitz continuous)
					&$\textbf{I}$\\		
					\hline
					
					$\D_{\num} \circ \D_{\net}$& $\D_{\num} \circ\left( \H_{\num,\ome}^{-1/2}\D_{\net}\H_{\num,\ome}^{1/2}\right) $ where $\D_{\net}$ is non-expansive &$\H_{\num,\ome}$\\	
					
					\hline
				\end{tabular}
			}
		}
	\end{table}

	\newpage
	\section{Experimental Details} \label{sec:appendix C experiments}

	Our experiments were mainly conducted on a PC with Intel Core i9-10900KF CPU (3.70GHz), 128GB RAM and two NVIDIA GeForce RTX 3090 24GB GPUs. 
    In all experiments, we use synthetic datasets,
    and adopt the Adam optimizer for updating variable $\ome$. 
    

	\subsection{Sparse Coding} \label{sec:appendix C1 sparse coding}
	For sparse coding, we set batch size=128, random seed=1126, training set size=10000. The testing set size depends on the size of each image. 
    Because we conduct unsupervised single image training, we do not use the MSE loss between the clear picture and the generated picture as the upper loss, but instead use the same unsupervised loss as in~\cite{xie2019differentiable}.
    Here $\D_\net$ is set to be $\D_{\mathtt{DLADMM}}$ and is given as follows
	\begin{equation}
		\begin{aligned}
			\u_1^{k+1} &=\underset{\u_1}{\operatorname{argmin}}\left\{\Vert\u_1\Vert_1+\left\langle\beta  \Q^{T}\left( \Q \u_1^{k}+  \u_2^{k}-\b+\lamm^{k} / \beta\right), \u_1\right\rangle+\frac{\rho_{1}}{2}\left\|\u_1-\u_1^{k}\right\|^{2}\right\}, \\
			\u_2^{k+1} &=\underset{\u_2}{\operatorname{argmin}}\left\{\Vert\u_2\Vert_1+\left\langle\beta \left( \Q \u_1^{k+1}+  \u_2^{k}-\b+\lamm^{k} / \beta\right), \u_2\right\rangle+\frac{\rho_{2}}{2}\left\|\u_2-\u_2^{k}\right\|^{2}\right\}, \\
			\lamm^{k+1} &=\lamm^{k}+\gamma \beta\left( \Q \u_1^{k+1}+  \u_2^{k+1}-\b\right).
		\end{aligned}
	\end{equation}

	For comparing with DLADMM more directly, we set $\D_\num$ to be $\I$ (the identity operator) here.
	We choose $ \left(\rho_{1}, \rho_{2}\right) \in \left\{\rho_{1} \geq \beta L_{\Q}^{2}, \rho_{2} \geq \beta L_{\mathbf{I}}^{2}\right\}$ to ensure the non-expansive hypothesis.
	All methods follow the general setting of hyper-parameters given in Table~\ref{hyper list}.

	\begin{table}[!htbp]
		\centering
		\caption{Values for hyper-parameters of sparse coding.}
		\label{hyper list}
		\renewcommand\arraystretch{1.2}
		\setlength{\tabcolsep}{1mm}{
			\begin{tabular}{cc|cc}
				\hline
				Hyper-parameters & Value&Hyper-parameters & Value \\
				\hline
				Epochs & 100 &Learning rate & $0.0002 * 0.5^{epoch / 30}$ \\
				Stage & 15 &Ratio $\beta$& $0.1$\\
				Batch size & 128&Ratio $\gamma$& $1$ \\
				\hline
			\end{tabular}
		}
	\end{table}

	\begin{figure}[!htbp]
		\centering
		\begin{subfigure}[t]{0.48\textwidth}{
				\includegraphics[height=3cm,width=4cm]{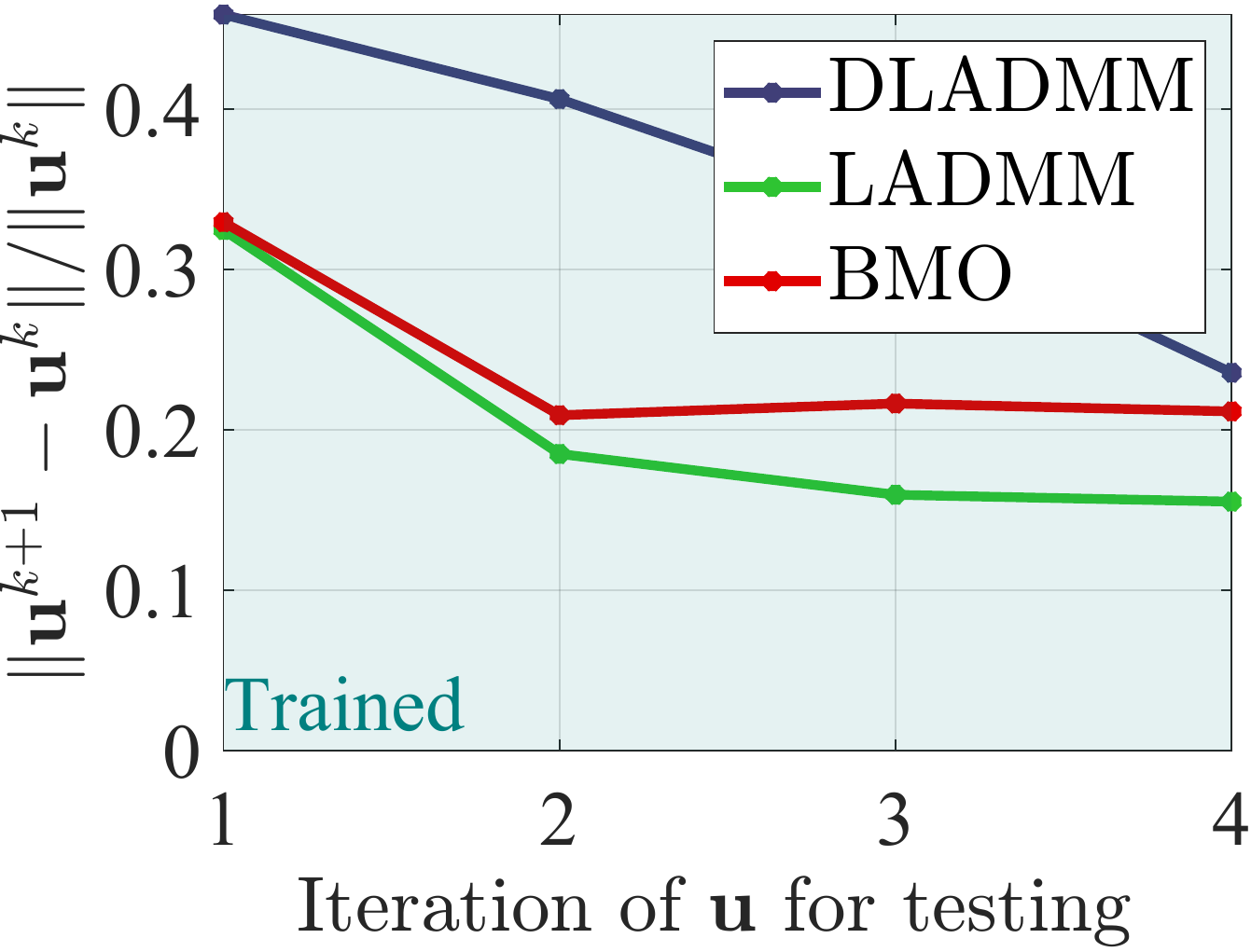}
				\includegraphics[height=3cm,width=4cm]{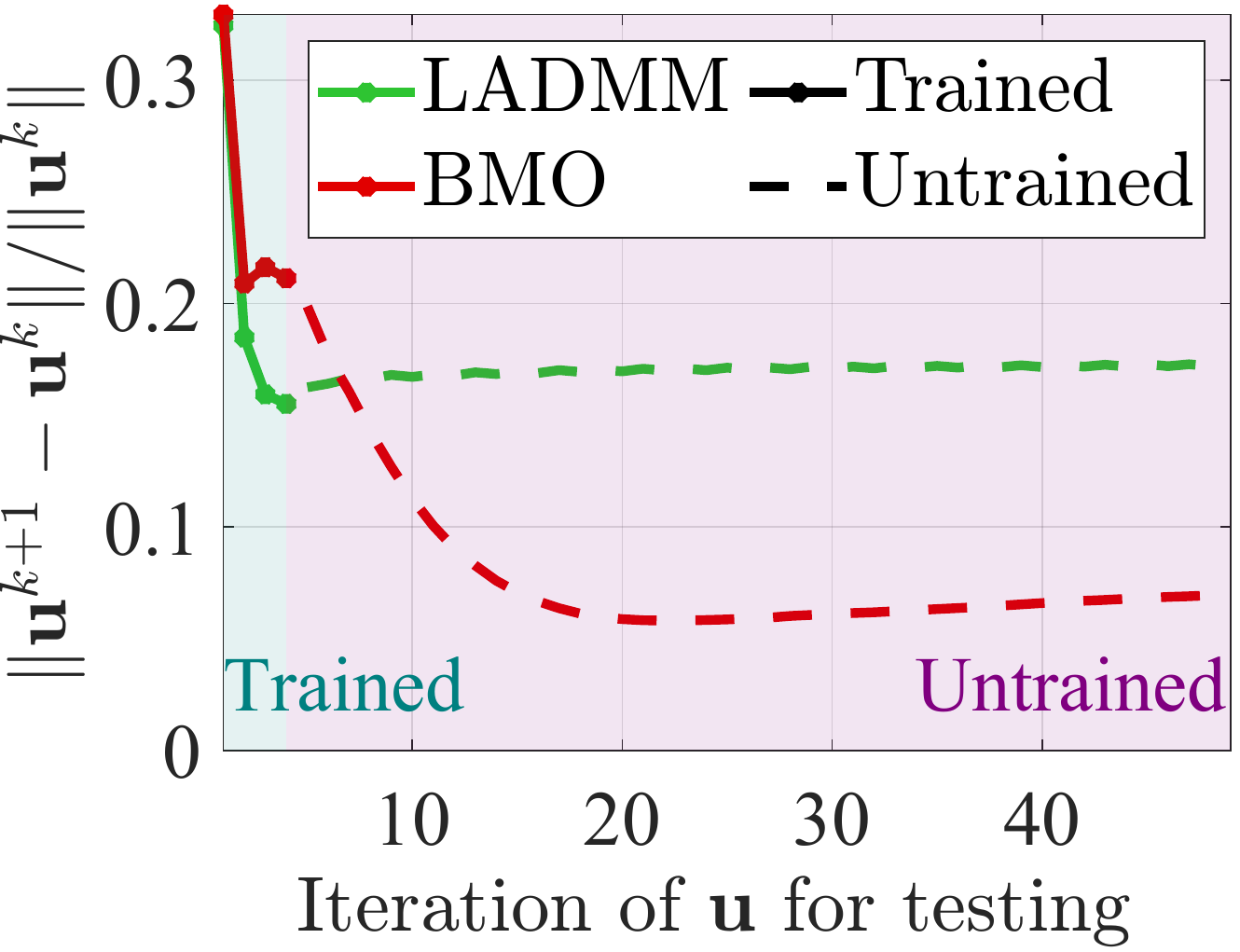}}
			\subcaption{Iterations for training = 5} \label{sub:25}
		\end{subfigure}

		\begin{subfigure}[t]{0.48\textwidth}{
				\includegraphics[height=3cm,width=4cm]{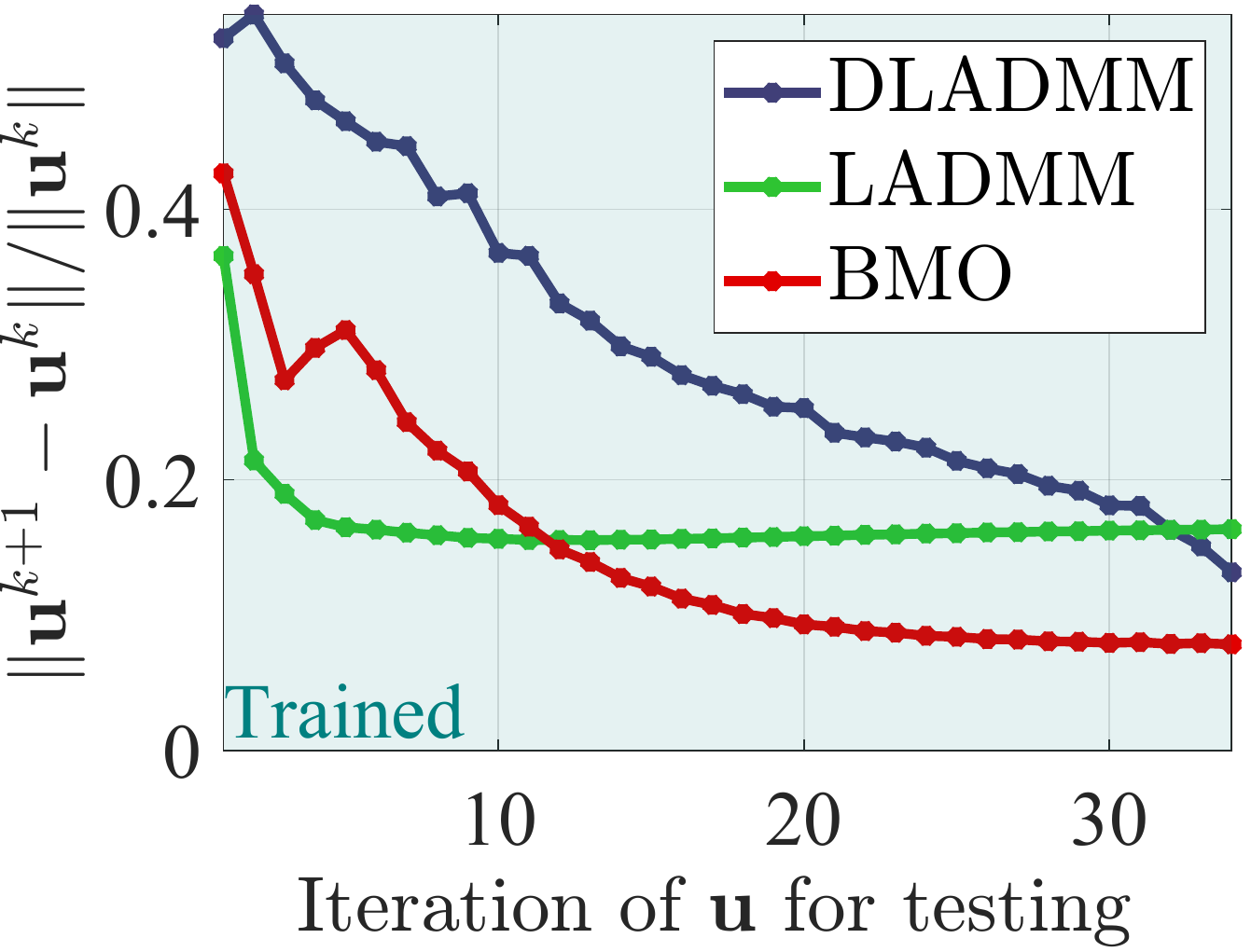}
				\includegraphics[height=3cm,width=4cm]{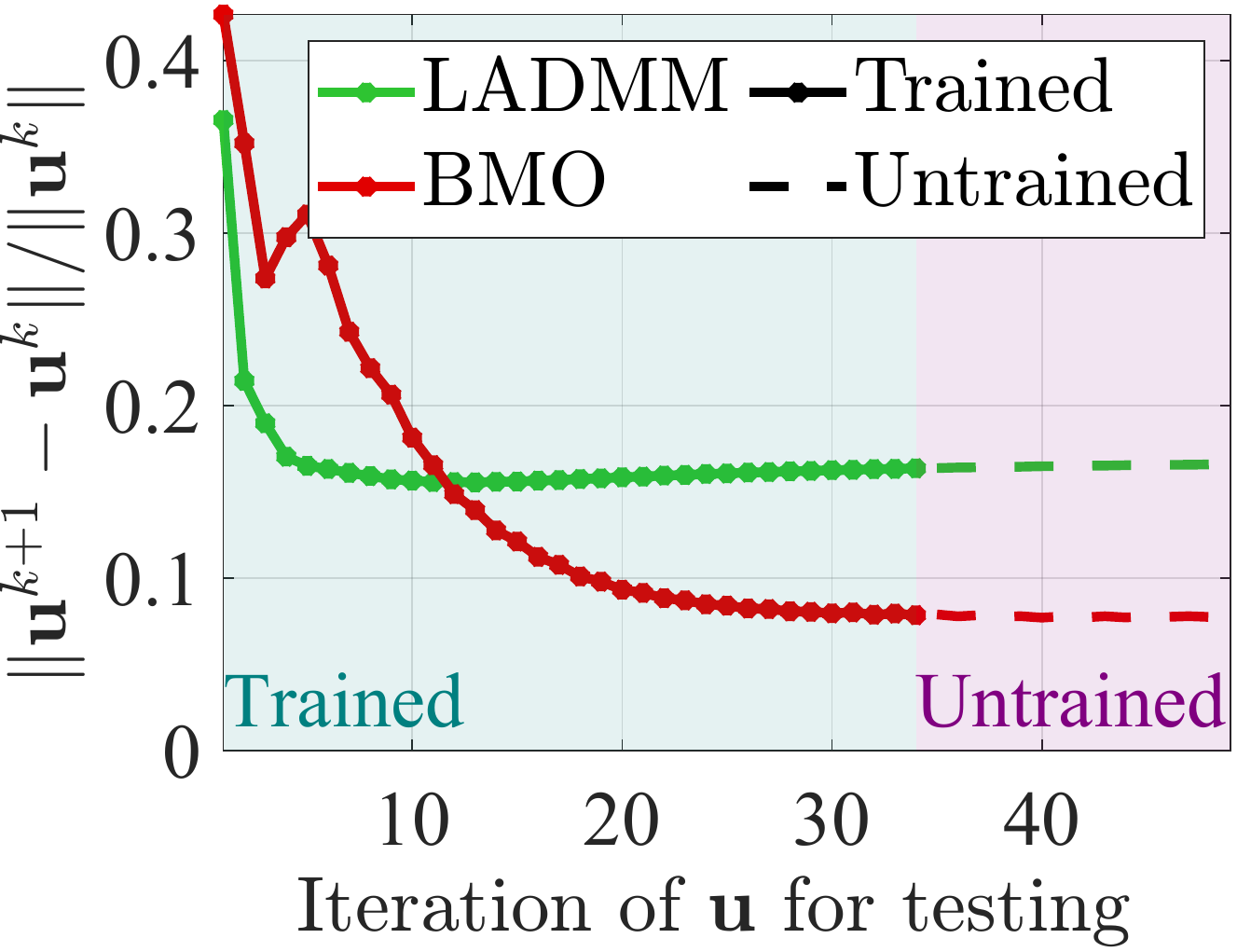}}
			\subcaption{Iterations for training = 35}	\label{sub:35}
		\end{subfigure}

		\caption{Convergence curves of $\Vert\u^{k+1}-\u^{k}\Vert/\Vert\u^{k}\Vert$ with respect to $k$, the number of iterations of $\u$ for testing, after (\subref{sub:25}) 5 and (\subref{sub:35}) 35 iterations for training. 
			Solid lines on the right column represent the iterations for testing are less than those for training (trained iterations), while dotted lines represent the iterations for testing are more than those for training (untrained iterations). 
			It can be seen that our method can successfully learn a non-expansive mapping with different number of iterations for training. 	
		}
		\label{fig:inner loop_appendix}
	\end{figure}

	In addition to Figure~\ref{fig:inner loop}, we show more results to demonstrate the impact on the number of iterations for training in Figure~\ref{fig:inner loop_appendix},
	and we can see that our method remain stable when the number of iterations for training changes.
	Note that for DLADMM, the number of iterations for training have to be more than those for testing, so in the right column we only show the curves of LADMM and BMO.

	\subsection{Image Deconvolution} \label{sec:appendix C2 deconvolution}
	 
	As for network architectures $\D_\net$, we use DRUNet which consists of four scales. 
	Each scale has an identity skip connection between 2 $\times$ 2 strided convolution (SConv) downscaling and 2 $\times$ 2 transposed convolution (TConv) upscaling operators. 
	The number of channels in each layer from the first scale to the fourth scale are 64, 128, 256 and 512, respectively. 
	Four successive residual blocks are adopted in the downscaling and upscaling of each scale. 
	For numerical update $\D_\num$, by using the auxiliary variable $\z=\W\u$, 
	we transform the problem as $\Vert\Q\W^{-1}\z-\b\Vert^2_2$ with a regularization term $\Vert\z\Vert_1$. 
	For this application, the numerical operator $\D_{\mathtt{PG}}$ is given by
	\begin{equation}
		\z^{k+1}=\underset{\u}{\operatorname{argmin}}\left\{\Vert\z^k\Vert_1+\left\langle\W^{-\top}\Q^\top(\Q\W^{-1}\z^k-\b), \z-\z^{k}\right\rangle+\frac{1}{2 }\left\|\z-\z^{k}\right\|_{\G(\ome)}^{2}\right\},\label{pglearn_solve}
	\end{equation}
	where $\G(\ome)$ is a parameterized diagonal matrix defined in Appendix~\ref{sec:appendix B1 DPG}.
	Here we use MSE as the upper loss function.
	We follow the general setting of hyper-parameters given in Table~\ref{hyp}.
	\begin{table}[!htbp]
		\centering
		\caption{Values for hyper-parameters of image deconvolution.} 
		\label{hyp}
		\renewcommand\arraystretch{1.2}
		\setlength{\tabcolsep}{1mm}{
			\begin{tabular}{cc|cc}
				\hline
				Hyper-parameters & Value &Hyper-parameters & Value \\
				\hline
				Epochs & 10000 &Learning rate & 0.0001\\
				Stage & 8 &Ratio $\mu$& $0.3$\\
				Batch size & 1&Ratio $\alpha$& $0.9$ \\
				\hline
			\end{tabular}
		}
	\end{table}

	\subsection{Rain Streak Removal}\label{sec:appendix C3 rain}
	In the rain streak removal task, for dataset, we use Rain100L and Rain100H~\cite{derain}. 
	For network architecture $\D_\net$, we use a 3-layer convolutional network with $\u_b,\u_r$ and $\b$ as the network input to estimate $\u_r$, 
	and a 2-layer convolutional network with $\u_r$ and $\b$ as the input to estimate $\u_b$. 
	In the network for estimating $\u_r$, we use some prior information of $\u_r$ as input just like~\cite{wang2020model}. 
	Such a problem falls into the form of problem in Eq.~\eqref{ALM_prb1} as the following,
	$$
	f(\u)=\frac{1}{2}\Vert\u_b+\u_r-\b\Vert_2^2+\kappa_b\Vert\v_b\Vert_1+\kappa_r\Vert\v_r\Vert_1 \ \text{s.t.}\ \A\u=0,
	$$
	where
	$$
	\u=\left[ \u_b;\u_r;\v_b;\v_r\right], 
	\A=	\left(\begin{array}{cccc}  
		\mathbf{I}&0&-\mathbf{I}&0\\
		0&\nabla&0&-\mathbf{I}\\
	\end{array}\right).
	$$
	Then the numerical operator $\D_{\mathtt{ALM}}$ with hyper-parameters $\beta,\rho_{\u_b}$, $\rho_{\u_r}$, $\rho_{\v_b}$ and $\rho_{\v_r}$ reads as 
	\begin{equation}
		\begin{aligned}
			\u_b^{k+1}&=\arg\min\limits_{\u_b}\{\left\langle (\u_b^k+\u_r^k-\b)+\beta \nabla^\top(\nabla \u_b^k-\v_b^k)+\lamm_b^k,\u_b\right\rangle +\frac{\rho_{\u_b}}{2}\Vert (\u_b-\u_b^k)\Vert_2^2\},\\
			\u_r^{k+1}&=\arg\min\limits_{\u_r}\{\left\langle (\u_b^k+\u_r^k-\b)+\beta (\u_r^k-\v_r^k)+\lamm_r^k,\u_r\right\rangle +\frac{\rho_{\u_r}}{2}\Vert (\u_r-\u_r^k)\Vert_2^2\},\\
			\v_b^{k+1}&=\arg\min\limits_{\v_b}\{\kappa_b \Vert \v_b\Vert_1+\left\langle -\beta (\nabla \u_b^{k+1}-\v_b^k)+\lamm_b^k,\v_b\right\rangle +\frac{\rho_{\v_b}}{2}\Vert  (\v_b-\v_b^k)\Vert_2^2\},\\
			\v_r^{k+1}&=\arg\min\limits_{\v_r}\{\kappa_m \Vert \v_r\Vert_1+\left\langle -\beta ( \u_r^{k+1}-\v_r^k)+\lamm_r^k,\v_r\right\rangle +\frac{\rho_{\v_r}}{2}\Vert  (\v_r-\v_r^k)\Vert_2^2\},\\
			\lamm_b^{k+1}&=\lamm_b^k+\beta (\nabla \u_b^{k+1}-\v_b^{k+1}),\\
			\lamm_r^{k+1}&=\lamm_r^k+\beta (\u_r^{k+1}-\v_r^{k+1}).
		\end{aligned}
	\end{equation}

	By setting $\ome = (\beta,\rho_{\u_b}$, $\rho_{\u_r}$, $\rho_{\v_b})$, we have
	\begin{equation}
		\H_{\ome}=\left(\begin{array}{ccccc}  
			\rho_{\u_b}-\beta\mathbf{I}&0&0&0&0\\
			0&\rho_{\u_r}-\beta\nabla^\top\nabla&0&0&0\\
			0&0&\rho_{\v_b}-\beta\mathbf{I}&0&0\\
			0&0&0&\rho_{\u_b}-\beta\mathbf{I}&0\\
			0&0&0&0&\frac{1}{\beta}\mathbf{I}\\
		\end{array}\right).
	\end{equation}
	In practice, we choose $\Omega$ such that $ \H_{\ome} \succ 0$ for each $\ome \in \Omega$,
	and $\H_{\ome}$ can be inverted quickly by Fourier transform.
	Here we use MSE as the upper loss function.
	We follow the general setting of hyper-parameters given in Table~\ref{tab:derain list}.
	\begin{table}[!htbp]
		\centering
		\caption{Values for hyper-parameters of rain streak removal.}
		\label{tab:derain list}
		\renewcommand\arraystretch{1.1}
		\setlength{\tabcolsep}{1mm}{
			\begin{tabular}{cc|cc}
				\hline
				Hyper-parameter & Value&Hyper-parameter & Value\\
				\hline
				Epochs & 100&Learning rate & 0.001 \\
				Stage & 17 &Ratio $\mu$& $0.1$\\
				Batch size & 16 &Ratio $\alpha$& $0.9$ \\
				\hline
			\end{tabular}
		}
	\end{table}

\end{document}